\documentclass[preprint,12pt,authoryear]{elsarticle}
\usepackage[top=.7in,bottom=.7in,left=.5in,right=.5in]{geometry}
\usepackage{rotating}
\usepackage{threeparttable}

\usepackage{amssymb}
\usepackage{mylatexstyle}
\usepackage{mathtools}
\usepackage{amsmath}
\allowdisplaybreaks
\usepackage{bm}
\usepackage{natbib}
\usepackage{multirow}      
\usepackage{subcaption}    
\usepackage{array}         

\usepackage{tikz}

\newcommand{\circled}[2][.7]{\tikz[baseline=(char.base)]{
    \node[shape=circle, draw, inner sep=0.8pt, minimum size=#1em] (char) {#2};}}
\usepackage{algorithm}      
\usepackage{algorithmic}    

\usepackage{threeparttable}
\usepackage{booktabs}
\usepackage{multirow}
\usepackage{siunitx}
\usepackage{makecell}
\usepackage{xcolor}

\def\T{\top}

\journal{Journal of Multivariate Analysis}

\begin{document}

\begin{frontmatter}



\title{Sufficient Dimension Reduction via Generalized Stein's Lemma} 


\author[label1]{Ye Tian} 
\affiliation[label1]{organization={Key Laboratory of Applied Statistics of MOE, Key Laboratory of Big
Data Analysis of Jilin Province, School of Mathematics and Statistics,
Northeast Normal University},
            addressline={Renmin Street}, 
            city={Changchun},
            postcode={130024}, 
            state={Jilin},
            country={China}}


\begin{abstract}
Sufficient dimension reduction (SDR) seeks the minimal subspace of the predictors that captures the full conditional distribution of the response, which is known as the central subspace (CS). When the response is multivariate, the problem becomes considerably more challenging, particularly when the sample size is limited. Existing methods face different limitations: inverse regression approaches rely on strong distributional assumptions and matrix inversion, and their multi-response extensions suffer from severe slice sparsity; forward regression methods depend on computationally intensive iterative smoothing whose cost grows with the response dimension; and deep learning-based approaches demand large amounts of labeled data. To circumvent these shortcomings, we propose an SDR framework based on the generalized Stein's lemma. Our method constructs a cross-moment matrix between the multivariate response and the marginal score function of the predictors, and recovers the CS via its singular value decomposition. The proposed method does not rely on the linearity condition, avoids matrix inversion as well as iterative smoothing, and can leverage unlabeled data when available. We establish convergence guarantees for the proposed estimator under standard regularity conditions. Moreover, we propose a practical rank-selection algorithm to estimate the dimension of the CS. Extensive simulation studies and a real data application demonstrate that the proposed methods consistently outperform existing approaches across a variety of settings, particularly in moderate-dimensional, label-scarce scenarios with high noise levels.
\end{abstract}


\begin{highlights}
\item robust sufficient dimension reduction for challenging scenarios with low signal-to-noise ratio and small sample sizes
\item nonlinear multi-response regression approach to CS estimation 
\item semi-supervised estimators leveraging unlabeled data for improved stability
\end{highlights}


\end{frontmatter}




\section{Introduction}\label{sec:introduction}

Sufficient dimension reduction (SDR) seeks an matrix $\bB$ such that
\begin{align*}
y \perp\!\!\!\perp \bx \mid \bB^\top \bx,
\end{align*}
i.e., the response $y$ is conditionally independent of the predictor $\bx$ given the projected predictors $\bB^\top\bx$. The subspace $\mathcal{S}(\bB)$ spanned by the columns of $\bB$ is called a dimension reduction subspace. The intersection of all such subspaces is referred to as the central subspace (CS) \citep{Cook2009}. The CS is the most parsimonious subspace that captures the full conditional relationship between $y$ and $\bx$.

In many practical settings, however, the primary object of interest is the conditional mean function. In such cases, the dimension reduction objective becomes finding $\bB$ such that
\begin{align*}
\EE(y\mid \bx)=\EE(y\mid \bB^\top \bx).
\end{align*}
Analogously, the intersection of subspaces satisfying the above relation is called the central mean subspace (CMS) \citep{li2002}. The CMS is the minimal subspace required to capture the conditional mean structure.

The estimation of CS and CMS has been extensively studied in the univariate response setting. With the rapid advancement of modern data collection technologies, however, response variables in many applications, for example genomics (joint analysis of multiple gene expressions), neuroimaging (simultaneous recording of multi-region brain activities), and finance (joint forecasting of multiple asset returns), are naturally vector-valued, i.e., $\by \in \mathbb{R}^q$ with $q>1$. In these multi-response problems, one must not only characterize the dependence between the predictors and each individual response component, but also account for the joint dependence among the response components.

In the univariate case, the CMS is generally a proper subspace of the CS. This gap even widens considerably when $\by$ is vector-valued: the CMS captures only the conditional mean vector $\EE(\by \mid \bx)$ and provides no information on any joint dependence among the response components, such as conditional covariances or higher-order interactions. Estimating the CMS is far simpler, yet due to this gap, it typically falls short of recovering the CS. However, when the model ensures $\text{CS} = \text{CMS}$, the full conditional  distribution can be recovered through conditional mean-based approaches. Therefore, we adopt the following multi-response multi-index model as a unified estimation framework:
\begin{equation}\label{model:1}
    \by = \cF(\bB^\top \bx) + \bepsilon, \quad \bepsilon \perp\!\!\!\perp \bx,
\end{equation}
where $\bx\in\mathbb{R}^p$, $\by\in\mathbb{R}^q$, $\bB\in\mathbb{R}^{p\times r}$ has full column rank with $r\ll p$, $\cF:\mathbb{R}^r\to\mathbb{R}^q$ is an unknown smooth vector-valued link function, and $\bepsilon$ is a zero-mean additive error independent of $\bx$. This model naturally accommodates vector-valued responses and, importantly, guarantees that the CS coincides with the CMS. Therefore, the CS can be recovered through the simpler task of CMS estimation.

As in general multi-index models, $\bB$ itself is not identifiable: if the model holds for $(\cF,\bB)$, then $(\cF\circ \bO^{-1}, \bB\bO)$ is also a solution for any invertible matrix $\bO\in\mathbb{R}^{r\times r}$. However, $\cS(\bB)$ remains invariant under such transformations and is precisely the CS (equivalently, the CMS). Hence, our objective is to estimate $\cS(\bB)$, which is a well-defined estimand. Once this subspace is obtained, the link function $\cF$ can be subsequently learned using, for example, neural networks, though this is not the focus of the present work.

Classical SDR approaches can be broadly grouped into two categories, each suffering from distinct limitations.The first category is the inverse regression family, which slices the response space and analyzes the conditional moments of the predictors within each slice. Its most prominent representative is sliced inverse regression (SIR) \citep{Li01061991}. SIR estimates the CS via an eigendecomposition of the covariance matrix of the slice means of $\bx$, and is computationally attractive owing to its closed-form solution. However, its consistency relies on the linearity condition, which requires that $\EE(\bx\mid\bB^\top\bx)$ be an affine function of $\bB^\top\bx$. This condition holds for elliptically symmetric distributions such as the normal or $t$-distribution, but is frequently violated for skewed or multi-modal predictors, leading to severe bias. Moreover, SIR involves estimating and inverting the $p\times p$ sample covariance matrix, which becomes unstable or infeasible when $n$ is not substantially larger than $p$. Its multi-response extension, MSIR \citep{barreda2007extensions}, directly partitions the response space into slices. This approach not only inherits the same linearity requirement but also suffers from extreme slice sparsity when $n$ is not large relative to $q$.

The second category is the forward regression family, which directly models the conditional mean function through nonparametric smoothing. The leading method in this class is minimum average variance estimation (MAVE) \citep{Xia2002}; a computationally less demanding alternative is the outer product gradient (OPG) method \citep{Xia2002}. Its extension, csOPG \citep{Xia2006}, is specifically designed for estimating the CS. Unlike inverse regression approaches, forward regression methods do not rely on the linearity condition and accommodate a broader class of predictor distributions. However, they typically require nonparametric smoothing to estimate gradients; in the multi-response setting, this estimation must be performed for each response component, causing the computational cost to grow with $q$. Moreover, while some software implementations of csOPG (e.g., the \texttt{MAVE} package) allow a matrix-valued response as input, the original theoretical framework of \citet{Xia2006} assumes a univariate response; the nonparametric kernel estimation of gradients and densities is not justified theoretically for multivariate responses, nor does it account for the covariance structure among response components. This gap between interface support and theoretical validation makes applicability of csOPG to multi-response problems statistically less grounded.

A separate line of work integrates DNNs with SDR. Representative works include mutual information-based sufficient representation learning~\citep{Zheng2022} and the belted and ensembled neural network~\citep{Tang25032026}. These methods leverage the representation learning power of DNNs to capture complex nonlinear dependencies between the predictors and the response. However, they require large amounts of labeled data to achieve stable estimation, severely limiting their applicability when labeled samples are scarce.

In summary, existing classical SDR methods and DNN approaches face significant limitations in the multi-response SDR setting. This work targets the practically common scenario where the predictor-response relationship is mildly nonlinear and labeled samples are scarce. In this setting, classical methods are often constrained by distributional assumptions or incur prohibitive computational cost as the response dimension grows, while DNN approaches require more labeled data than is typically available.

Motivated by these limitations, we propose a novel SDR method based on the generalized Stein's lemma. Our framework constructs a cross-moment matrix between the multivariate response and the marginal score function $\bs(\bx) = -\nabla_{\bx} \ln p(\bx)$:
\begin{equation}
    \bM = \EE\{\bs(\bx) \by^\T\},
\end{equation}
and the CS is then recovered via the singular value decomposition of $\bM$, with its dimension determined by a rank-selection procedure. This construction avoids matrix inversion and iterative smoothing, while relying on substantially weaker assumptions than existing SDR approaches. We demonstrate its effectiveness through extensive simulations and a real data application, where it consistently outperforms existing competitors.

The remainder of this paper is organized as follows. Section~\ref{sec:model} reviews necessary background on SDR and the generalized Stein's lemma, and introduces the proposed methodology. Section~\ref{sec:theory} presents the theoretical properties of the proposed estimators. Simulation results are reported in Section~\ref{sec:ss}, followed by a real data application in Section~\ref{sec:rda}. Concluding remarks are given in Section~\ref{sec:fw}. Additional technical proofs and supplementary experimental details are provided in the Appendix.

\section{Methodology}\label{sec:model}

\subsection{Notation}
Before proceeding, we establish the notation used throughoutthe rest of the paper. We use bold lowercase letters (e.g., $\ba, \bb$) for vectors and bold uppercase letters (e.g., $\bA, \bB$) for matrices. For a vector $\ba$, $\ba_i$ denotes its $i$-th element; in the context of observed data, $\ba_i$ also denotes the $i$-th observation, and $\ba_{i,j}$ denotes the $j$-th component of that observation. For a matrix $\bA$, $\|\bA\|_F$ denotes its Frobenius norm, $\sigma_i(\bA)$ its $i$-th largest singular value, and $\text{SVD}_{\ell,i}(\bA)$ the $p \times i$ matrix whose columns are the left singular vectors corresponding to the $i$ largest singular values of $\bA$. Let $\cV_r$ denote the set of $r \times r$ orthogonal matrices. For $n \in \mathbb{N}_+$, write $[n] = \{1, \dots, n\}$. For two sequences $\{x_n\}$ and $\{y_n\}$, we write $x_n \asymp y_n$ if $x_n/y_n$ is bounded away from zero and infinity uniformly in $n$.

\subsection{the Model}

Recall the multi-response multi-index model introduced in Section~\ref{sec:introduction}:
\begin{align*}
\by = \cF(\bB^\top \bx) + \bepsilon, \quad \bepsilon \perp\!\!\!\perp \bx, 
\end{align*}
where $\bx \in \mathbb{R}^p$ is the predictor vector, $\by \in \mathbb{R}^q$ is the multivariate response, $\bB \in \mathbb{R}^{p \times r}$ is a full-rank dimension-reduction matrix with $r \le p$, $\cF: \mathbb{R}^r \to \mathbb{R}^q$ is an unknown link function, and $\bepsilon$ is a zero-mean random error independent of $\bx$. Here $\bB$ is understood to be of minimal rank among all matrices satisfying $\by \perp\!\!\!\perp \bx \mid \bB^\top\bx$; equivalently, its column space $\cS(\bB)$ is the CS \citep{Cook2009}. Since $\bepsilon$ is independent of $\bx$, the conditional mean $\EE(\by \mid \bx) = \cF(\bB^\top\bx)$, and the minimality of $\cS(\bB)$ further implies that it is also the CMS \citep{li2002}. Thus the CS and CMS coincide, both equal to $\cS(\bB)$. Our goal is therefore to estimate this common subspace.

To estimate $\cS(\bB)$, we invoke the following generalized Stein's lemma. For notational convenience, we henceforth take $\bB$ to be an orthogonal matrix whose columns form a basis for $\cS(\bB)$; this orthogonal parameterization is without loss of generality, since any basis can be orthonormalized without changing the column space.

\begin{lemma}(Generalized Stein's Lemma)\label{lem:fs}
Suppose model~\eqref{model:1} holds. Assume that the expectations $\mathbb{\EE}\{\by_{j}\bs(\bx)\}$ as well as $\mathbb{\EE}\{\nabla_{\bz}f_{j}(\bB^{\top}\bx)\}$ both exist and well-defined for $j \in [p]$. Further assume that $\lim_{\| \bx\| \rightarrow \infty }f_{j}(\bB^{\top}\bx) p(\bx) \rightarrow 0$.  Then, we obtain 
\begin{align}\label{eq:moment}
\mathbb{\EE}\{ \by_{j} \bs(\bx)\} = \bB \mathbb{\EE}\{ \nabla_{\bz}f_{j} (\bB^{\top}\bx) \}.   
\end{align}
Let $\bM_1 = \{\nabla_{\bz} f_1(\bB^\top\bx), \ldots, \nabla_{\bz} f_q(\bB^\top\bx)\} \in \mathbb{R}^{r \times q}$. We further obtain
\begin{align*}
    \bM = \bB \bM_1.
\end{align*}
\end{lemma}

Lemma~\ref{lem:fs} shows that $\bB$ is fully separated from the expectations of derivatives of the link functions, enabling its estimation without knowing the explicit form of $\cF$. A key observation is that the multi-response structure allows us to use a first-order Stein's lemma to recover the CS, circumventing the need for its second-order counterpart commonly used in single-response multi-index settings. Second-order approaches converge substantially slower and require reliable estimates of higher-order score functions, which are difficult to obtain in small samples. The success of our first-order strategy relies on the diversity of the responses. When the responses are sufficiently diverse, $\bM_1$ achieves full row rank $r$, which is precisely the condition needed for $\bB$ to be fully recovered from $\bM$. This subtle link between multi-response diversity and first-order Stein's lemma is the key insight of our method, turning a seemingly higher-order problem into a tractable one for first-order methods.

\subsection{Estimation of the CS}

\subsubsection{Plug-in Estimator}

Existing score-based methods in the broader statistical literature commonly assume the true score function $\bs(\cdot)$ is known \citep{yangs,balasubramanian2018}. This assumption, while convenient for theoretical analysis, is unrealistic in most applications, as the true density of $\bx$ is rarely known. We therefore focus on the practical setting where the score must be estimated from data. Given a generic score estimator $\widehat{\bs}(\cdot)$, a natural estimator of $\bM$ is 
\begin{align}\label{eq:def-fe}
\widehat{\bM}_n = \frac{1}{n}\sum_{i=1}^{n} \widehat{\bs}(\bx_i)\by_i^\top.
\end{align}
Then, a plug-in estimator of $\bB$ can be obtained by taking the top-$r$ left singular vectors of $\widehat{\bM}_n$:
\begin{equation*}
    \widehat{\bB}_n = \operatorname{SVD}_{l,r}\left( \widehat{\bM}_n \right).
\end{equation*}

\subsubsection{Score Estimators}\label{sec:score-est}

We now turn to the choice of the score estimator $\widehat{\bs}(\cdot)$. Two paradigms are considered.

Kernel-based score estimators are appealing for their theoretical clarity and practical reliability in small-sample settings. This clarity stems from the linear structure of the reproducing kernel Hilbert space (RKHS) and the convexity of the optimization landscape, enabling rigorous analysis and, under regularity conditions, even pointwise error bounds \citep{tian2026}. In practice, they perform well with limited data, where more complex models tend to overfit. The principal drawback is computational, as matrix operations scale poorly with sample size, rendering kernel-based methods impractical for large-scale applications.

Deep neural networks (DNNs), in contrast, represent a more recent paradigm for score estimation. Their major drawback is a heavy dependence on large sample sizes; without sufficient data, they are prone to unstable estimation. Yet this very dependence is also their strength: the highly parameterized structure allows them to effectively leverage abundant data to capture complex nonlinearities beyond the reach of kernel methods, and they can be trained efficiently via stochastic gradient-based algorithms. The downside is the notorious difficulty of theoretical analysis. The highly nonlinear and non-convex nature precludes comparable theoretical guarantees; existing work on neural score estimators has largely been restricted to mean squared error \citep{chen2023score,shen2024}, approximation error \citep{yakovlev2025a}, or distributional divergences \citep{oko2023diffusion,yakovlev2025b}; pointwise error bounds remain absent.

Given the complementary strengths and limitations of both paradigms, this work employs two specific score estimators: the Tikhonov regularized estimator with the curl-free IMQ kernel \citep{zhou2020}, denoted TRE for brevity, and a DNN estimator whose architectural details are deferred to the Appendix. Their empirical performance is compared in Sections~\ref{sec:ss} and~\ref{sec:rda}.

\subsubsection{Score Estimation with Unlabeled Data}\label{sec:se}

We consider the scenario where, in addition to a limited number of labeled data, a large amount of unlabeled samples are also available. This is common in many applications: labeled samples are expensive to obtain, while unlabeled ones are cheap and abundant. For instance, in autonomous driving, collecting raw sensor data is inexpensive, but annotating each frame with object bounding boxes is costly; in medical imaging, unlabeled scans are plentiful, yet expert labeling is time-consuming and expensive.

Specifically, suppose we have a labeled set $\mathcal{T}_{\ell} = \{(\bx_i, \by_i)\}_{i=1}^{n_{\ell}}$ and an unlabeled set $\mathcal{T}_{u} = \{\bx_i\}_{i=n_{\ell}+1}^{n_{\ell}+n_{u}}$. Let $\mathcal{T}_{\ell,u} = \mathcal{T}_{\ell} \cup \mathcal{T}_u$ denote the combined set and $n = n_{\ell} + n_{u}$ the total sample size.

Because score estimation is unsupervised, the estimator $\widehat{\bs}(\cdot)$ can be fitted using either only the labeled covariates $\{\bx_i\}_{i=1}^{n_{\ell}}$ or the full set $\{\bx_i\}_{i=1}^{n}$. Recall the cross-moment estimator from \eqref{eq:def-fe}, which is now specialized to the labeled sample,
\begin{align*}
\widehat{\bM}_{\ell} = \frac{1}{n_{\ell}}\sum_{i=1}^{n_{\ell}} \widehat{\bs}(\bx_i)\by_i^\top.
\end{align*}
At first glance, this quantity depends on $\widehat{\bs}$ only through its evaluations at the labeled covariates $\{\bx_i\}_{i=1}^{n_\ell}$. However, The situation changes, when $\widehat{\bs}$ is trained on the same labeled data. The estimator and the points on which it is evaluated are no longer independent, introducing a statistical coupling that complicates the analysis.

A standard technique to circumvent this difficulty is data splitting: train $\widehat{\bs}$ on $\mathcal{T}_u$ only and reserve $\mathcal{T}_{\ell}$ for moment estimation, which ensures independence between the estimator and its evaluation points. While splitting simplifies analysis, it discards the labeled data during score estimation; for estimators with tractable coupling, however, this information can be retained. We consider three configurations for each estimator, distinguished by the training data:
\begin{enumerate}
    \item $\mathrm{S}^{\mathrm{split}}$: train on $\mathcal{T}_u$ only; $\mathcal{T}_{\ell}$ reserved for moment estimation (no coupling);
    \item $\mathrm{S}^{\mathrm{sup}}$: train on $\mathcal{T}_{\ell}$ only (coupled);
    \item $\mathrm{S}^{\mathrm{full}}$: train on $\mathcal{T}_{\ell,u} = \mathcal{T}_{\ell} \cup \mathcal{T}_u$ (coupled).
\end{enumerate}
For the TRE, all three configurations are theoretically tractable. Our main proposals are $\mathrm{S}_{\mathrm{TRE}}^{\mathrm{sup}}$ and $\mathrm{S}_{\mathrm{TRE}}^{\mathrm{full}}$; $\mathrm{S}_{\mathrm{TRE}}^{\mathrm{split}}$ is included only for comparison in Sections~\ref{sec:ss} and~\ref{sec:rda}. For the DNN, only $\mathrm{S}_{\mathrm{NN}}^{\mathrm{split}}$ admits theoretical guarantees; the coupled configurations are also included for empirical comparison only, as their theoretical analysis is left for future work.

\subsubsection{Rank Selection}\label{sec:rank-s}

When the dimension of the CS, $r$, is known a priori, the estimation procedures described above apply directly. In practice, however, $r$ is typically unknown and must be estimated from data. The problem reduces to estimating the rank of a low-rank population matrix subject to random perturbation. As will be shown in Section~\ref{sec:theory}, under regularity conditions, the true rank can be recovered with high probability by hard thresholding the singular values of $\widehat{\bM}_{\ell}$, provided the threshold is chosen appropriately. This oracle procedure is summarized in Algorithm~\ref{alg:rs} for reference.

\begin{algorithm}
\caption{Rank Selection via Hard Thresholding (Oracle)}\label{alg:rs}
\begin{algorithmic}[1]
\REQUIRE Matrix $\widehat{\bM}_{\ell} \in \mathbb{R}^{p \times q}$, threshold $\tau > 0$.
\ENSURE Selected rank $\widehat{r}_{\tau}$.
\STATE Compute the compact SVD: $\bU, \bSigma, \bV^\top = \text{SVD}(\widehat{\bM}_{\ell})$.
\STATE $\widehat{r}_{\tau} \gets \sum_{i=1}^{\min(p,q)} \mathbf{1}[\sigma_i(\widehat{\bM}_{\ell}) > \tau]$.
\RETURN $\widehat{r}_{\tau}$.
\end{algorithmic}
\end{algorithm}

In principle, the philosophy underlying our rank determination is aligned with SIR-type methods, where the rank of $\mathrm{Cov}\{\EE(\bx \mid \by)\}$ corresponds to the dimension of the CS and is recovered via a generalized eigenvalue problem. MAVE-type methods, in contrast, typically determine the dimension by minimizing a local weighted least squares loss.

However, in practice, the oracle threshold is unknown, and the distribution of the noise in $\widehat{\bM}_{\ell}$ is non-standard, so conventional rank selection criteria relying on standard asymptotic thresholds are not directly applicable. We therefore adopt an adaptive rank selection strategy that distinguishes between light-tailed and heavy-tailed covariate distributions.

The procedure proceeds as follows. First, we assess the tail behavior of the covariates by computing the sample skewness and excess kurtosis for each feature. If any feature exceeds the specified thresholds, the input is classified as heavy-tailed; otherwise, it is treated as light-tailed.

For heavy-tailed data, we employ an energy criterion that computes the cumulative proportion of total variance explained by the leading singular values of $\widehat{\bM}_{\ell}$ and selects the smallest rank $k$ that reaches a pre-set threshold $\eta$.

For light-tailed data, we use a permutation test to determine the significance of each singular value. Given the estimated score matrix $\widehat{\bS} \in \mathbb{R}^{n_\ell \times p}$ whose $i$-th row is $\widehat{\bs}(\bx_i)^\top$, we repeatedly permute the rows of the response matrix $\bY \in \mathbb{R}^{n_\ell \times q}$ to obtain $\bY_{\text{perm}}$, compute the permuted cross-moment matrix $\bM_{\text{perm}} = (1/n_\ell)\widehat{\bS}^\top \bY_{\text{perm}} \in \mathbb{R}^{p \times q}$, and record its leading singular values to build a null distribution. For each candidate rank $k$, we obtain a threshold $\theta_k$ as the $\alpha$-quantile of the $k$-th singular values from the permutations. We then sequentially compare the observed singular values $\lambda_k$ of $\widehat{\bM}_{\ell}$ with $\theta_k$; the rank is chosen as the largest $k$ such that $\lambda_k > \theta_k$ holds for all preceding directions, with a minimum of one direction always retained.

The complete procedure is summarized in Algorithm~\ref{alg:ars} in the Appendix. Its effectiveness is demonstrated in Sections~\ref{sec:ss} and~\ref{sec:rda}, while the rationale behind the design choices is discussed in Remark~\ref{rm:rs}.

\section{Theoretical Analysis}\label{sec:theory}

This section presents the main theoretical results of this work. The overall structure is as follows: we first establish error bounds for the cross-moment estimator under two different score estimators, then translate these bounds into estimation error for $\cS(\bB)$, and finally establish that, with an oracle threshold level, hard-thresholding the singular values of the cross-moment estimator consistently recovers the true rank.

For the data splitting configuration, the score estimator is trained independently of the labeled evaluation points. This independence enables a general analysis that applies to arbitrary score estimators and yields a finite-sample error bound in Theorem~\ref{thm:subgcomp}. This bound holds for any sample size and only requires mild assumptions on the score estimator's convergence behavior in Assumption~\ref{ass:score} and sub-Gaussianity of the score estimates and responses in Assumption~\ref{assum:1}. Moreover, the bound is valid in both low- and high-dimensional settings, meaning $p$ and $q$ can be either fixed or grow with $n_u$ and $n_\ell$.

For the coupled configurations of the TRE, the score estimator is trained on both labeled and unlabeled data. The resulting dependence structure is more intricate and precludes a straightforward finite-sample analysis at the same level of generality. Instead, we leverage the explicit RKHS representation of the TRE to derive an asymptotic convergence rate in Theorem~\ref{thm:kernel}, which establishes consistency of the coupled estimator in the large-sample regime and complements the finite-sample bound obtained for the data splitting configuration. The finite-sample advantages of the coupled TRE, which are not captured by the asymptotic rate, are discussed separately in Remark~\ref{rem:ctre} and verified empirically. This analysis is conducted in the fixed-dimensional setting, where $p$ and $q$ are constants independent of the sample size, aligning with the classical framework of SDR and nonparametric kernel estimation. For DNNs, such coupled analysis remains open.

We then analyze the theoretical properties of the oracle rank selection procedure based on hard thresholding in Algorithm~\ref{alg:rs}, which serves as the conceptual foundation for the adaptive procedure in Algorithm~\ref{alg:ars}. We show that, under regularity conditions and with a suitably chosen threshold, the oracle procedure recovers the true rank $r$ with high probability. The adaptive version in Algorithm~\ref{alg:ars} is provided for practical implementation and is evaluated empirically in Sections~\ref{sec:ss} and~\ref{sec:rda}.

To measure the discrepancy between the estimated and true subspaces, we use the orthogonal Procrustes distance. For two matrices $\bTheta_1, \bTheta_2 \in \cV_r$, we define
\begin{align*}
\operatorname{dist}(\bTheta_1, \bTheta_2) \coloneqq \inf_{\bV \in \mathcal{V}_r} \|\bTheta_1 - \bTheta_2 \bV\|_\mathrm{F}.
\end{align*}
Let $\widehat{\bB}$ be an estimated orthogonal basis of $\cS(\bB)$. The subspace distance between $\widehat{\bB}$ and $\cS(\bB)$ is then given by $\operatorname{dist}(\widehat{\bB}, \bB)$. This metric is equivalent to the sine of the principal angles between the two subspaces, up to a constant factor.

We now present the first main result for the data splitting configuration under the following assumptions.

\begin{assumption}[Component-wise sub-Gaussianity]
\label{assum:1}
For any $i \in [n]$, $k \in [p]$, $j \in [n_{\ell}]$, and $m \in [q]$ assume that $\widehat{\bs}_k(\bx_i)$ and $\by_{j,m}$,  are sub-Gaussian random variables. Moreover, there exists constant $K > 0$, such that  
\begin{align*}
\|\widehat{\bs}_k(\bx_i)\|_{\psi_2} \le K \text{, and } \|\by_{j,m}\|_{\psi_2} \le K,   
\end{align*}
where $\|\cdot\|_{\psi_2}$ denotes the sub-Gaussian Orlicz norm.
\end{assumption}

\begin{assumption}[Score convergence]\label{ass:score}
For any $\delta \in (0,1)$, there exists a non-increasing sequence $\epsilon(n_u) \to 0$ as $n_u \to \infty$, such that for all $n_u$, the score estimator $\widehat{\bs}^u$ trained on $\mathcal{T}_u$ satisfies
\begin{align*}
\PP_{\mathcal{T}_u}\left\{ \sqrt{\EE_{\bx}\|\widehat{\bs}^u(\bx) - \bs(\bx)\|^2_2} \le \epsilon(n_u) \right\} \ge 1 - \frac{\delta}{2}.
\end{align*}
\end{assumption}

Assumption~\ref{assum:1} requires the score estimates $\widehat{\bs}^u(\cdot)$ and the responses $\by_i$ to be component-wise sub-Gaussian. This is a standard condition in the score estimation literature and is satisfied, for example, when the true score function $\bs(\cdot)$ is bounded. For kernel-based score estimators with uniform convergence guarantees, sub-Gaussianity is inherited from the true score function. For DNN estimators, ensuring sub-Gaussianity requires additional control over the architecture and training procedure. In cases where such control is not exercised, finite moment conditions may still suffice for Stein's lemma; \citet{2017zhuoran} developed a truncation-based approach under finite fourth moments, though these considerations are not the focus of the present work.

Assumption~\ref{ass:score} posits a finite-sample bound $\epsilon(n_u)$ on the score estimation error that tends to zero as $n_u \to \infty$. This is a minimal consistency requirement. For kernel-based estimators, such bounds have been established under various regularity conditions on the RKHS and the target function \citep{zhou2020}; for DNN estimators, non-asymptotic guarantees are available under different sets of assumptions depending on the architecture and the distributional setting \citep{fu2025}. The specific conditions under which Assumption~\ref{ass:score} holds are not the focus of this work; we refer the interested reader to the respective references.

The following theorem provides a finite-sample bound for the data splitting configuration under Assumptions~\ref{assum:1} and~\ref{ass:score}.

\begin{theorem}[Finite-sample bound for data splitting configuration]\label{thm:subgcomp}
Suppose that Assumptions~\ref{assum:1} and~\ref{ass:score} hold. Define
\begin{align*}
\widehat{\bM}^{s}_{\ell,u} \coloneqq \frac{1}{n_{\ell}}\sum_{i=1}^{n_{\ell}} \widehat{\bs}^u(\bx_i)\by_i^\top.
\end{align*}
Then there exists an absolute constant $C > 0$ such that for any $\delta \in (0,1)$, with probability at least $1-\delta$,
\begin{align*}
\left\| \widehat{\bM}^{s}_{\ell,u} - \bM \right\|_\mathrm{F}
\le
C K^2 \left\{ \sqrt{q} \, \epsilon(n_u)
+\sqrt{\frac{pq\ln\left(\frac{pq}{\delta}\right)}{n_{\ell}}} + \frac{\sqrt{pq}\ln\left(\frac{pq}{\delta}\right)}{n_{\ell}} \right\}.
\end{align*}
\end{theorem}

\begin{remark}
The two empirical error terms in the bound in Theorem~\ref{thm:subgcomp} arise from the sub-exponential Bernstein's inequality to the entries of the sample cross-moment matrix. The bound is adaptive to the convergence rate of the score estimator: the first term propagates the score estimation error $\epsilon(n_u)$ to the cross-moment estimator, while the second and third terms capture the empirical approximation error due to the finite labeled sample size. If the score estimator achieves a parametric rate, i.e., $\epsilon(n_u) = O(n_u^{-1/2})$, and if $n_\ell$ is of the same order as $n_u$, then the cross-moment estimator attains the parametric rate up to logarithmic factors. Conversely, if the score estimator is nonparametric and converges at a slower rate, the cross-moment estimator inherits that slower rate, up to the additional $n_\ell^{-1/2}$ term. This highlights the flexibility of the data splitting configuration: it propagates the convergence behavior of the score estimator to the cross-moment estimation task without imposing additional structural assumptions on the estimator itself.
\end{remark}

We now analyze the coupled configurations of the TRE. The closed-form structure of the TRE allows us to characterize the coupling explicitly without imposing a general consistency assumption on $\widehat{\bs}$; instead, we require only the sub-Gaussianity of the covariates and responses, and the convergence of $\widehat{\bs}$ follows from the analysis.

\begin{assumption}[Sub-Gaussianity of covariates and responses]\label{assum:subG}
For any $i \in [n]$, $k \in [p]$, $j \in [n_{\ell}]$, and $m \in [q]$, the random variables $\bx_{i,k}$ and $\by_{j,m}$ are sub-Gaussian.
\end{assumption}

To establish the theoretical guarantees for the TRE, we work within the vector-valued RKHS framework. Let $\mathcal{H}_{\mathcal{K}}$ denote the RKHS associated with a matrix-valued kernel $\mathcal{K}:\mathcal{X}\times\mathcal{X}\to\mathbb{R}^{p\times p}$, and let $L_{\mathcal{K}}$ be the corresponding integral operator. We impose the following source condition on the true score function.

\begin{assumption}[Source condition]\label{assum:source}
The true score function satisfies
\begin{align*}
\bs = L_{\mathcal{K}}^{\omega} \boldsymbol{f}_0,
\end{align*}
for some $\boldsymbol{f}_0 \in \mathcal{H}_{\mathcal{K}}$ and $\omega \in (0,1]$. The parameter $\omega$ quantifies the regularity of $\bs$ relative to the RKHS; larger $\omega$ corresponds to smoother functions and leads to faster convergence rates.
\end{assumption}

Let $\widehat{\bs}^{\mathcal{K}_{\mathrm{CI}},\lambda}_{\cT_{\ell,u}}(\cdot)$ denote the TRE trained on the full set $\cT_{\ell,u}$, where the subscript indicates that the training set consists of $n_\ell$ labeled and $n_u$ unlabeled samples. The corresponding cross-moment estimator is defined as
\begin{align*}
\widehat{\bM}^{\mathcal{K}_{\mathrm{CI}},\lambda}_{\ell,u}(\cT_{\ell,u}) = \frac{1}{n_{\ell}}\sum_{i=1}^{n_{\ell}} \widehat{\bs}^{\mathcal{K}_{\mathrm{CI}},\lambda}_{\cT_{\ell,u}}(\bx_i) \by_i^\top.
\end{align*}

When the context is clear, we abbreviate $\widehat{\bM}^{\mathcal{K}_{\mathrm{CI}},\lambda}_{\ell,u}(\cT_{\ell,u})$ as $\widehat{\bM}^{\mathcal{K}_{\mathrm{CI}},\lambda}_{\ell,u}$. For the labeled-only setting, $\cT_{\ell,u} = \cT_\ell$ (equivalently $n_u = 0$); for the full-data setting, $\cT_{\ell,u} = \cT_\ell \cup \cT_u$ ($n_u > 0$).

To characterize the interplay between labeled and unlabeled sample sizes, we parameterize $n_\ell \asymp n^\beta$ for some $\beta \in (0,1]$. When $\beta = 1$, the labeled sample size $n_\ell$ is of the same order as $n$; when $\beta < 1$, we have $n_\ell = o(n)$, so the unlabeled samples dominate. The following theorem establishes the convergence rate of the coupled TRE, confirming that coupling does not compromise consistency.

\begin{theorem}\label{thm:kernel}
Suppose Assumptions~\ref{assum:subG} and~\ref{assum:source} hold. Let $\widehat{\bM}^{\cK_{\mathrm{CI}},\lambda}_{\ell,u}$ be the estimated cross-moment matrix. Then, with probability at least $1-\delta$:

\begin{enumerate}
\item If $1/2 < \beta \le 1$, choose the regularization parameter as
$
\lambda \asymp n^{-(2\beta - 1)/\{2(\omega +1)\}}$. Then, there exists a constant $C_\delta > 0$ such that for all sufficiently large $n$, with probability at least $1-\delta$,
\begin{align*}
\left\| \widehat{\bM}^{\cK_{\mathrm{CI}},\lambda^*}_{\ell,u} - \bM \right\|_\mathrm{F}
\le
C_\delta \cdot n^{-\frac{\omega(2\beta - 1)}{2(\omega + 1)}}.
\end{align*}
The constant $C_\delta$ absorbs all constants depending on $\delta$ arising from the thresholds of the concentration inequalities.

\item If $\beta \le 1/2$, the error does not converge to zero in probability.
\end{enumerate}
\end{theorem}

Theorem~\ref{thm:kernel} shows that the coupled TRE attains the faster rate only when $n_\ell$ is of the same order as $n$ (i.e., $\beta = 1$). When $n_\ell$ grows more slowly ($\beta < 1$), the additional variance terms decay more slowly and become dominant, so the estimator does not achieve that rate.

We also note that the convergence rates in Theorems~\ref{thm:subgcomp} and~\ref{thm:kernel} are not directly comparable, as they are derived under different assumptions. Nevertheless, both bounds can be translated into that on the estimate of $\bB$.

A key condition for this translation is the identifiability of the column space, which requires a lower bound on the $r$-th singular value of $\bM$. To see why such a lower bound is natural, recall that $\bM = \bB\bM_1$. In typical SDR settings, the columns of $\bB$ are the effective directions and are normalized to have constant norm, so $\sigma_r(\bB) \ge c_1$ for some absolute constant $c_1 > 0$. Moreover, when $\bM_1$ is full row rank with entries of constant order and bounded condition number, its singular values are all of order $\sqrt{q/r}$, so $\sigma_r(\bM_1) \ge c_2\sqrt{q/r}$ for some absolute constant $c_2 > 0$. Therefore,
\begin{align*}
\sigma_r(\bM) = \sigma_r(\bB\bM_1) \ge \sigma_r(\bB)\sigma_r(\bM_1) \ge c_1 c_2 \sqrt{q/r} = c\sqrt{q/r}.
\end{align*}
We assume this condition holds in the following corollary.

\begin{corollary}\label{cor:1}
Assume the identifiability condition $\sigma_r(\bM) \ge C\sqrt{q/r}$ holds. Let $\widehat{\bB}$ be the top-$r$ left singular vectors of a generic estimator $\widehat{\bM}$ of $\bM$. Then,
\begin{align}\label{eq:bb}
\inf_{\bO\in\mathcal{V}_r}\|\widehat{\bB}\bO - \bB \|_\mathrm{F}
\le
C_0 \cdot \sqrt{\frac{r}{q}} \, \|\widehat{\bM} - \bM\|_\mathrm{F},    
\end{align}
for some absolute constant $C_0>0$.

Applying this general bound to the data splitting configuration with the bound from Theorem~\ref{thm:subgcomp}, we obtain, for $\forall \delta > 0$, with probability at least $1 - \delta$,
\begin{align*}
\inf_{\bO\in\mathcal{V}_r}\|\widehat{\bB}_{\ell,u}^{s}\bO - \bB\|_\mathrm{F}
\le
C_{\mathrm{s}} \left\{\sqrt{r} \, \epsilon(n_u)
+
\sqrt{ \frac{pr \ln\left(\frac{pq}{\delta}\right)}{n_{\ell}}} +\frac{\sqrt{pr}\ln\left(\frac{pq}{\delta}\right)}{n_{\ell}}  \right\},
\end{align*}
where $\widehat{\bB}_{\ell,u}^{s}$ denotes the top-$r$ left singular vectors of $\widehat{\bM}^{s}_{\ell,u}$, and $C_{\mathrm{s}}>0$ is an absolute constant.

Similarly, for the coupled TRE configuration, for $n$ large enough, with probability at least $1 - \delta$,
\begin{align*}
\inf_{\bO\in\mathcal{V}_r}\| \widehat{\bB}^{\mathcal{K}_{\mathrm{CI}},\lambda}_{\ell,u}\bO - \bB\|_\mathrm{F}
\le
C_{\mathrm{c},\delta} \, n^{-\frac{\omega (2\beta - 1)}{2(\omega +1)}},
\end{align*}
for the regime $1/2 < \beta \le 1$, where $\widehat{\bB}^{\mathcal{K}_{\mathrm{CI}},\lambda}_{\ell,u}$ denotes the top-$r$ left singular vectors of $\widehat{\bM}^{\mathcal{K}_{\mathrm{CI}},\lambda}_{\ell,u}$, and $C_{\mathrm{c},\delta} >0$ is a constant depending on $\delta$.
\end{corollary}

\begin{remark}\label{rem:ctre}
For the coupled TRE, the finite-sample behavior deserves a closer look. Although the asymptotic rate in Theorem~\ref{thm:kernel} is slower than the parametric $n^{-1/2}$ rate due to the nonparametric nature of the score estimator, the coupled TRE can still be competitive in small samples. Several features contribute to this. The Tikhonov regularizer $\lambda \|f\|_{\mathcal{H}}^2$ provides intrinsic complexity control, particularly beneficial when $n_\ell$ is small. The RKHS framework adapts to the unknown smoothness of $\bs(\cdot)$ via the source condition, whereas parametric approaches are confined to a fixed model class. Moreover, the curl-free kernel encodes the gradient-field structure of the score function, supplying a strong geometric prior that reduces the effective dimensionality of the estimation problem.

These structural advantages propagate directly to the cross-moment estimator, which is a weighted average of score evaluations at the labeled points. Consequently, the benefits of score estimation, including the curl-free prior, RKHS regularization, and coupling, are inherited by $\widehat{\bM}^{\cK_{\mathrm{CI}},\lambda}_{\ell,u}$, especially when $n_\ell$ is small. For methods that directly estimate the conditional mean or distribution, such strong structural priors are not generally available.
\end{remark}

The preceding theoretical results assume that the dimension $r$ of the CS is known. In practice, $r$ must be estimated from the data. The following theorem establishes that, under suitable conditions, the hard-thresholding procedure based on the singular values of the cross-moment estimator recovers the true rank with high probability.

\begin{theorem}[Dimension Selection]\label{the:rank-s}
Suppose $\sigma_r(\bM)$ is bounded below by an absolute constant $C > 0$, and let $\widehat{\bM}$ be a generic estimator of $\bM$ such that for any $\varepsilon > 0$ and any $\delta \in (0,1)$, there exist $N_\ell, N_u$ such that for all $n_\ell \ge N_\ell$ and $n_u \ge N_u$,
\begin{align}\label{eq:dr}
\PP\left(\|\widehat{\bM} - \bM\|_\mathrm{F} \le \varepsilon\right) \ge 1 - \delta.
\end{align}
Then, fix the threshold level $\tau = C/2$, for any $\delta \in (0,1)$, for large enough $n_{\ell}$ and $n_u$, we obtain
\begin{align*}
\PP(\widehat{r}_\tau = r) \ge 1 - \delta,
\end{align*}
where $\widehat{r}_\tau$ is the number of singular values of $\widehat{\bM}$ exceeding $\tau$ defined in Algorithm~\ref{alg:rs}.
\end{theorem}

The convergence condition in Theorem~\ref{the:rank-s} holds for both the data splitting configuration in Theorem~\ref{thm:subgcomp} and the coupled TRE configuration in Theorem~\ref{thm:kernel} in the regime $1/2 < \beta \le 1$. Therefore, Theorem~\ref{the:rank-s} applies to both configurations.

\begin{remark}\label{rm:rs}
Theorem~\ref{the:rank-s} shows that hard thresholding recovers the true rank with high probability when the oracle threshold is used. In practice, however, such a threshold is unavailable. Moreover, for the proposed TRE and DNN estimators, the estimation errors relative to $\bM$ are correlated across entries, unlike the i.i.d. settings in which standard threshold selection methods are derived. This correlation invalidates naive entry-wise thresholding approaches that assume independence.

The practical procedure in Algorithm~\ref{alg:ars} circumvents this difficulty by shifting the problem from selecting a noise-level threshold to specifying parameters that are more amenable to domain knowledge: the variance proportion in the energy method and the significance level in the permutation test. These parameters are considerably easier to interpret and specify than a cut-off value on the singular value spectrum. The choice between the two strategies is made automatically based on the sample skewness and excess kurtosis of the covariates, and the overall procedure is evaluated empirically in Sections~\ref{sec:ss-rs} and~\ref{sec:rda}.
\end{remark}

\section{Simulation Study}\label{sec:ss}

The simulation study consists of two parts. In the first part, we compare the performance of the proposed methods in recovering the CS against several competing approaches. In the second part, we demonstrate the effectiveness of our proposed rank selection procedure.

We compare the proposed method with the multi-response extensions of SIR, csOPG, and KSIR, as well as a neural network baseline. We implement our own version of SIR following the standard procedure in the R package dr, as the original package often yields invalid estimates, such as NaN, for small sample sizes. For the neural network baseline, we use a single-hidden-layer multilayer perceptron (MLP). After training the network, we perform SVD on its first-layer weight matrix $\bW_1$ and take its first $r$ right singular vectors as the estimated basis. The consistency of this approach for single-response SDR was established in \citet{xu2025neural}. For csOPG and KSIR, we use the implementations provided in the MAVE R package.

For our proposed methods, we consider the TRE and DNN score estimators introduced in Section~\ref{sec:score-est}.
For each estimator, we consider the three configurations $\mathrm{S}^{\mathrm{split}}$, $\mathrm{S}^{\mathrm{sup}}$, and $\mathrm{S}^{\mathrm{full}}$ defined in Section~\ref{sec:se}.

In both parts, we fix the  dimension of the covariate at $p = 100$, that of the response at $q = 10$. The number of unlabeled observations is fixed at $n_u = 10{,}000$ for the NN estimator and $n_u = 3{,}500$ for the TRE We consider four distributions of the covariate: 
\begin{itemize}
    \item[] Case (a): standard isotropic Gaussian,
    \item[] Case (b): zero-mean multivariate normal with AR(1) covariance $\bSigma_{i,j} = \rho^{|i-j|}$ where $\rho = 0.5$, 
    \item[] Case (c): independent log-normal with $\ln \bx_j \sim \mathcal{N}(0, 1.5^2)$,
    \item[] Case (d): componentwise exponential transformation of the aforementioned AR(1) normal,
\end{itemize}
denoted by Isotropic Gaussian, AR(1) Gaussian, Log-normal, and Exp-trans AR(1) in the following, respectively. The error $\bepsilon$ is drawn from the standard isotropic Gaussian. Responses are generated according to model~\eqref{model:1}, with the specific link functions detailed in each part below. We consider a low signal-to-noise ratio setting, where the noise level is comparable to or larger than the signal, as is common in practice. All results are reported as the sample mean, with standard deviations in parentheses, over 100 independent repetitions. Throughout the simulation tables, the minimum mean in each column is boldfaced.
\subsection{Performance of the Proposed Method in Recovering the CS}\label{sec:ss-space}
In this part, the true dimension of the CS  is fixed at $r = 3$. The true basis matrix $\bB$ is obtained as the orthogonal factor from the QR decomposition of a $p \times r$ matrix with i.i.d. standard normal entries. To assess the impact of the number of labeled observations, we vary $n_{\ell} \in \{50, 100, 300, 500\}$ while keeping the unlabeled data sizes fixed at the values stated above. The responses are generated from model~\eqref{model:1} according to the following link functions:
\begin{subequations}
\begin{align*}
f_1(\bx) &= \bx_1 + 0.2 \cdot \text{expit}(\bx_2),
\end{align*}
\begin{align*}
f_2(\bx) &= \frac{\bx_2}{0.5 + (\bx_3 + 1.5)^2}, 
\end{align*}
\begin{align*}
f_3(\bx) &= \tanh(0.8 \cdot \bx_3) + \text{erf}(1.2 \cdot \bx_1),
\end{align*}
\begin{align*}
f_4(\bx) &= \tanh(0.8 \cdot \bx_1) + \text{expit}(\bx_3) + 0.8 \cdot \text{erf}(1.2 \cdot \bx_2),
\end{align*}
and for $k \geq 5$,
\begin{equation*}
f_k(\bx) = \sin(\bx_i), 
\end{equation*}
where $i \equiv (k\mod r)$.
\end{subequations}

The orthogonal Procrustes distance is used to evaluate the estimation accuracy of the CS. The results are summarized in Table~\ref{tab:sim}. 

The baseline methods SIR, csOPG, KSIR, and MLP generally perform poorly, with SIR being the best among them in most settings. In contrast, our proposed methods achieve substantially lower distances across all configurations.

Among the TRE configurations, the core methods $\mathrm{S}_{\mathrm{TRE}}^{\mathrm{sup}}$ and $\mathrm{S}_{\mathrm{TRE}}^{\mathrm{full}}$ consistently achieve the lowest average distances across all settings, confirming their robustness. Under simple distributions such as Isotropic Gaussian and AR(1) Gaussian, $\mathrm{S}_{\mathrm{TRE}}^{\mathrm{sup}}$ generally outperforms $\mathrm{S}_{\mathrm{TRE}}^{\mathrm{full}}$. Under more challenging heavy-tailed distributions, including Log-normal and Exp-trans AR(1), $\mathrm{S}_{\mathrm{TRE}}^{\mathrm{full}}$ offers a modest improvement over $\mathrm{S}_{\mathrm{TRE}}^{\mathrm{sup}}$ when labeled samples are scarce, but this advantage diminishes as the labeled sample size grows; with a relatively large labeled sample, $\mathrm{S}_{\mathrm{TRE}}^{\mathrm{sup}}$ performs comparably to or slightly better than $\mathrm{S}_{\mathrm{TRE}}^{\mathrm{full}}$. The contrast configuration $\mathrm{S}_{\mathrm{TRE}}^{\mathrm{split}}$ behaves similarly to $\mathrm{S}_{\mathrm{TRE}}^{\mathrm{full}}$ but with a consistent degradation that does not narrow with more labeled data, reflecting the structural loss induced by data splitting.

For the NN configurations, $\mathrm{S}_{\mathrm{NN}}^{\mathrm{sup}}$ performs poorly, often worse than the baseline methods, while $\mathrm{S}_{\mathrm{NN}}^{\mathrm{split}}$ and $\mathrm{S}_{\mathrm{NN}}^{\mathrm{full}}$ yield substantially smaller distances. $\mathrm{S}_{\mathrm{NN}}^{\mathrm{split}}$ is generally comparable to $\mathrm{S}_{\mathrm{NN}}^{\mathrm{full}}$, with only a slight degradation for small sample sizes that diminishes as labeled samples increase. Given the complexity of the coupled DNN configuration and the lack of theoretical guarantees, our focus on the split configuration as the core DNN method is justified.

The observed differences between the $\mathrm{S}^{\mathrm{split}}$, $\mathrm{S}^{\mathrm{full}}$, and $\mathrm{S}^{\mathrm{sup}}$ configurations across the two score estimators reflect the distinct inductive biases of the TRE and DNN.

For the TRE, the curl-free kernel encodes a strong structural prior: its RKHS consists of gradient fields of scalar functions. This prior enables accurate score estimation at labeled points even with limited labeled samples, explaining the strong performance of $\mathrm{S}_{\mathrm{TRE}}^{\mathrm{sup}}$ in small-sample regimes under simple covariate distributions such as Gaussian. Adding unlabeled data introduces a trade-off: it helps capture the global structure of the covariate distribution but may dilute the local accuracy at labeled points. For simple distributions where local information is already sufficient, $\mathrm{S}_{\mathrm{TRE}}^{\mathrm{sup}}$ therefore outperforms both $\mathrm{S}_{\mathrm{TRE}}^{\mathrm{split}}$ and $\mathrm{S}_{\mathrm{TRE}}^{\mathrm{full}}$. For more complex heavy-tailed distributions, local information alone is insufficient, and the global information from unlabeled data becomes beneficial, as seen in cases (c) and (d). However, $\mathrm{S}_{\mathrm{TRE}}^{\mathrm{split}}$ consistently underperforms $\mathrm{S}_{\mathrm{TRE}}^{\mathrm{full}}$ because it discards labeled data during score estimation; this gap persists regardless of the labeled sample size, since data splitting discards the information contained in labeled samples for score estimation.

For the DNN estimator, the situation is different. Lacking an explicit gradient prior, the DNN requires substantially more samples to characterize the score function. With limited labeled data alone, $\mathrm{S}_{\mathrm{NN}}^{\mathrm{sup}}$ performs poorly, often worse than baseline methods. Unlabeled data are therefore essential: both $\mathrm{S}_{\mathrm{NN}}^{\mathrm{split}}$ and $\mathrm{S}_{\mathrm{NN}}^{\mathrm{full}}$ yield substantial improvements.

The key distinction between score-based methods and the baseline approaches lies in the source of nonlinearity they exploit. Score-based methods extract nonlinear information from the covariate distribution itself through $\bs(\cdot)$; this information can be obtained from unlabeled data or from structural priors, such as the curl-free gradient-field prior encoded in the TRE kernel, and does not require labeled responses. Baseline methods such as SIR and MAVE, in contrast, extract nonlinear information from the relationship between $\bx$ and $\by$, specifically through the inverse regression curve or the gradient of the link function. This information can only be learned from labeled data and thus requires a sufficiently large labeled sample to be accurate. This fundamental difference explains why score-based methods can perform well even when labeled samples are limited. When the covariate distribution is simple, modeling the score is relatively easy and can be done accurately with limited data, as demonstrated by the TRE. When the covariate distribution is complex, unlabeled data become valuable for capturing its structure, which is reflected in the improvements of $\mathrm{S}_{\mathrm{TRE}}^{\mathrm{full}}$ and $\mathrm{S}_{\mathrm{NN}}^{\mathrm{full}}$ over their labeled-only counterparts. In all cases, the proposed score-based methods outperform the baseline approaches.

Based on these observations, we offer the following practical recommendations for problems with limited labeled data. For TRE-based estimation, $\mathrm{S}_{\mathrm{TRE}}^{\mathrm{sup}}$ is the default choice in most cases due to its robust performance. The only exception is when the covariate distribution is complex and labeled samples are extremely scarce; in such cases, $\mathrm{S}_{\mathrm{TRE}}^{\mathrm{full}}$ offers a modest improvement and is preferred. For DNN-based estimation, $\mathrm{S}_{\mathrm{NN}}^{\mathrm{split}}$ is the recommended default, as it performs comparably to $\mathrm{S}_{\mathrm{NN}}^{\mathrm{full}}$ while enjoying theoretical guarantees. In scenarios where theoretical tractability is less of a concern and the sample size is very small, $\mathrm{S}_{\mathrm{NN}}^{\mathrm{full}}$ may be used as an alternative.

\begin{table}[htbp]
\caption{Estimation accuracy of the central subspace under different covariate distributions and labeled sample sizes.}
\label{tab:sim}
\centering
\scriptsize
\setlength{\tabcolsep}{4.5pt}
\renewcommand{\arraystretch}{0.8}

\subfloat[Isotropic Gaussian]{
\begin{tabular}{lcccc}
\toprule
Method & $n=50$ & $n=100$ & $n=300$ & $n=500$ \\
\midrule
\multicolumn{5}{c}{\textit{External baselines}} \\
SIR & $2.1563 (0.0781)$ & $2.2236 (0.0577)$ & $1.2035 (0.0816)$ & $0.8299 (0.0472)$ \\
csOPG & $2.2048 (0.0672)$ & $2.0762 (0.0763)$ & $1.4928 (0.0857)$ & $1.0306 (0.0571)$ \\
KSIR & $2.1916 (0.0502)$ & $2.0896 (0.0761)$ & $1.6363 (0.1390)$ & $1.2062 (0.1846)$ \\
MLP & $2.0356 (0.0859)$ & $1.8711 (0.0953)$ & $1.8423 (0.1008)$ & $1.9203 (0.0933)$ \\
\midrule
\multicolumn{5}{c}{\textit{Framework: contrast configurations}} \\
$\mathrm{S}_{\mathrm{NN}}^{\mathrm{sup}}$  & $2.2618 (0.0449)$ & $2.2317 (0.0557)$ & $2.1115 (0.0593)$ & $1.6972 (0.0637)$ \\
$\mathrm{S}_{\mathrm{NN}}^{\mathrm{full}}$ & $1.8956 (0.0698)$ & $1.6520 (0.0679)$ & $1.2246 (0.0515)$ & $1.0444 (0.0426)$ \\
$\mathrm{S}_{\mathrm{TRE}}^{\mathrm{split}}$ & $1.8295 (0.0719)$ & $1.5762 (0.0606)$ & $1.1345 (0.0513)$ & $0.9469 (0.0408)$ \\
\midrule
\multicolumn{5}{c}{\textit{Framework: core configurations}} \\
$\mathrm{S}_{\mathrm{NN}}^{\mathrm{split}}$ & $1.8960 (0.0774)$ & $1.6581 (0.0666)$ & $1.2454 (0.0539)$ & $1.0441 (0.0416)$ \\
$\mathrm{S}_{\mathrm{TRE}}^{\mathrm{sup}}$  & $\textbf{1.8056 (0.0730)}$ & $\textbf{1.5169 (0.0576)}$ & $\textbf{1.0463 (0.0484)}$ & $\textbf{0.7312 (0.0317)}$ \\
$\mathrm{S}_{\mathrm{TRE}}^{\mathrm{full}}$ & $1.8281 (0.0719)$ & $1.5731 (0.0605)$ & $1.1255 (0.0510)$ & $0.9073 (0.0392)$ \\
\bottomrule
\end{tabular}
}
\hfill
\subfloat[AR(1) Gaussian]{
\begin{tabular}{lcccc}
\toprule
Method & $n=50$ & $n=100$ & $n=300$ & $n=500$ \\
\midrule
\multicolumn{5}{c}{\textit{External baselines}} \\
SIR & $2.1824 (0.0682)$ & $2.2387 (0.0603)$ & $1.4093 (0.1039)$ & $1.0230 (0.0645)$ \\
csOPG & $2.2040 (0.0617)$ & $2.1143 (0.0676)$ & $1.6574 (0.0956)$ & $1.2624 (0.0690)$ \\
KSIR & $2.2025 (0.0579)$ & $2.0968 (0.0751)$ & $1.7812 (0.1308)$ & $1.4248 (0.1764)$ \\
MLP & $2.0558 (0.0699)$ & $1.9247 (0.0862)$ & $1.8693 (0.1108)$ & $1.9118 (0.0972)$ \\
\midrule
\multicolumn{5}{c}{\textit{Framework: contrast configurations}} \\
$\mathrm{S}_{\mathrm{NN}}^{\mathrm{sup}}$  & $2.2619 (0.0489)$ & $2.2422 (0.0611)$ & $2.1418 (0.0616)$ & $1.7870 (0.0694)$ \\
$\mathrm{S}_{\mathrm{NN}}^{\mathrm{full}}$ & $2.0143 (0.0790)$ & $1.8253 (0.0818)$ & $1.4329 (0.0655)$ & $1.2504 (0.0650)$ \\
$\mathrm{S}_{\mathrm{TRE}}^{\mathrm{split}}$ & $1.8645 (0.0775)$ & $1.6427 (0.0560)$ & $1.3138 (0.0452)$ & $1.0907 (0.0494)$ \\
\midrule
\multicolumn{5}{c}{\textit{Framework: core configurations}} \\
$\mathrm{S}_{\mathrm{NN}}^{\mathrm{split}}$ & $2.0200 (0.0727)$ & $1.8416 (0.0812)$ & $1.4564 (0.0638)$ & $1.2608 (0.0627)$ \\
$\mathrm{S}_{\mathrm{TRE}}^{\mathrm{sup}}$  & $\textbf{1.8339 (0.0787)}$ & $\textbf{1.5856 (0.0614)}$ & $\textbf{1.2509 (0.0425)}$ & $\textbf{0.8964 (0.0421)}$ \\
$\mathrm{S}_{\mathrm{TRE}}^{\mathrm{full}}$ & $1.8628 (0.0775)$ & $1.6394 (0.0560)$ & $1.3061 (0.0449)$ & $1.0510 (0.0485)$ \\
\bottomrule
\end{tabular}
}

\vspace{0.2cm}

\subfloat[Log-normal]{
\begin{tabular}{lcccc}
\toprule
Method & $n=50$ & $n=100$ & $n=300$ & $n=500$ \\
\midrule
\multicolumn{5}{c}{\textit{External baselines}} \\
SIR & $2.2148 (0.0516)$ & $2.2607 (0.0564)$ & $1.8753 (0.0694)$ & $1.7199 (0.1055)$ \\
csOPG & $2.2302 (0.0483)$ & $2.2130 (0.0468)$ & $2.0672 (0.0675)$ & $1.9047 (0.0796)$ \\
KSIR & $2.2328 (0.0477)$ & $2.2276 (0.0505)$ & $2.1610 (0.0685)$ & $1.9765 (0.0549)$ \\
MLP & $2.1031 (0.0580)$ & $1.9755 (0.0691)$ & $1.8717 (0.0640)$ & $1.8597 (0.0720)$ \\
\midrule
\multicolumn{5}{c}{\textit{Framework: contrast configurations}} \\
$\mathrm{S}_{\mathrm{NN}}^{\mathrm{sup}}$  & $2.2635 (0.0480)$ & $2.2688 (0.0465)$ & $2.2660 (0.0515)$ & $2.2666 (0.0469)$ \\
$\mathrm{S}_{\mathrm{NN}}^{\mathrm{full}}$ & $2.1647 (0.0609)$ & $2.0849 (0.0749)$ & $1.9306 (0.0801)$ & $2.0633 (0.0686)$ \\
$\mathrm{S}_{\mathrm{TRE}}^{\mathrm{split}}$ & $1.9893 (0.0698)$ & $1.8167 (0.0963)$ & $1.4472 (0.0974)$ & $1.2986 (0.1116)$ \\
\midrule
\multicolumn{5}{c}{\textit{Framework: core configurations}} \\
$\mathrm{S}_{\mathrm{NN}}^{\mathrm{split}}$ & $2.1740 (0.0629)$ & $2.1236 (0.0689)$ & $1.9870 (0.0827)$ & $1.9214 (0.0823)$ \\
$\mathrm{S}_{\mathrm{TRE}}^{\mathrm{sup}}$  & $2.0930 (0.0671)$ & $1.8867 (0.0835)$ & $\textbf{1.4271 (0.1089)}$ & $\textbf{1.1639 (0.1325)}$ \\
$\mathrm{S}_{\mathrm{TRE}}^{\mathrm{full}}$ & $\textbf{1.9877 (0.0693)}$ & $\textbf{1.8126 (0.0973)}$ & $1.4323 (0.0988)$ & $1.2642 (0.1154)$ \\
\bottomrule
\end{tabular}
}
\hfill
\subfloat[Exp-trans AR(1)]{
\begin{tabular}{lcccc}
\toprule
Method & $n=50$ & $n=100$ & $n=300$ & $n=500$ \\
\midrule
\multicolumn{5}{c}{\textit{External baselines}} \\
SIR & $2.2214 (0.0556)$ & $2.2560 (0.0445)$ & $1.9023 (0.0729)$ & $1.7718 (0.1091)$ \\
csOPG & $2.2362 (0.0527)$ & $2.2297 (0.0564)$ & $2.0905 (0.0799)$ & $1.9316 (0.0876)$ \\
KSIR & $2.2245 (0.0598)$ & $2.2318 (0.0573)$ & $2.1837 (0.0646)$ & $2.0321 (0.0694)$ \\
MLP & $2.1101 (0.0624)$ & $1.9873 (0.0723)$ & $1.8635 (0.0700)$ & $1.8458 (0.0713)$ \\
\midrule
\multicolumn{5}{c}{\textit{Framework: contrast configurations}} \\
$\mathrm{S}_{\mathrm{NN}}^{\mathrm{sup}}$  & $2.2675 (0.0510)$ & $2.2648 (0.0474)$ & $2.2606 (0.0467)$ & $2.2621 (0.0473)$ \\
$\mathrm{S}_{\mathrm{NN}}^{\mathrm{full}}$ & $2.1466 (0.0604)$ & $2.0610 (0.0864)$ & $1.8862 (0.0759)$ & $2.0189 (0.0765)$ \\
$\mathrm{S}_{\mathrm{TRE}}^{\mathrm{split}}$ & $2.0000 (0.0678)$ & $1.8329 (0.0912)$ & $1.4793 (0.0905)$ & $1.3389 (0.1061)$ \\
\midrule
\multicolumn{5}{c}{\textit{Framework: core configurations}} \\
$\mathrm{S}_{\mathrm{NN}}^{\mathrm{split}}$ & $2.1727 (0.0669)$ & $2.0946 (0.0713)$ & $1.9547 (0.0741)$ & $1.8892 (0.0788)$ \\
$\mathrm{S}_{\mathrm{TRE}}^{\mathrm{sup}}$  & $2.0850 (0.0726)$ & $1.8959 (0.0933)$ & $1.4803 (0.1021)$ & $\textbf{1.2114 (0.1367)}$ \\
$\mathrm{S}_{\mathrm{TRE}}^{\mathrm{full}}$ & $\textbf{1.9983 (0.0679)}$ & $\textbf{1.8288 (0.0912)}$ & $\textbf{1.4667 (0.0925)}$ & $1.3040 (0.1090)$ \\
\bottomrule
\end{tabular}
}
\end{table}

\subsection{Performance of the Rank-Selection Procedure}\label{sec:ss-rs}

We use the same four covariate distributions as in the first part. Rank selection is evaluated for $r \in \{1,3,5\}$. For $r=1$, we construct the link functions separately. For $r=3$, we use the same link functions as in the first part. For $r=5$, we add two additional directions and modify the link functions accordingly to make them detectable. The explicit forms for all three cases are given below. 
\begin{subequations}
\begin{align*}
f_1(\bx) =& \bx_1 + 0.2\,\text{expit}\left(\bx_{i_1}\right), \label{eq:f1}\\
f_2(\bx) =& 
\begin{cases}
\dfrac{\bx_2}{0.5 + (\bx_1 + 1.5)^2}, & r \ge 2,\\
\bx_1^2, & r = 1;
\end{cases} 
\end{align*}
if $r = 3$,
\begin{align*}
f_3(\bx) = &\tanh(0.8\bx_3) + 0.8\,\text{erf}(1.2\bx_1),\\
f_4(\bx) = & \tanh(0.8\bx_1) + \text{expit}(\bx_3) + 0.8\,\text{erf}(1.2\bx_2);\\
\end{align*}
or if $r = 5$, 
\begin{align*}
 f_3(\bx) = & 1.5\tanh(\bx_3) + 0.8\,\text{erf}(1.2\bx_1),\\
 f_4(\bx) = & 2.0\tanh(\bx_4) + 0.8\,\text{expit}(\bx_1),\\
 f_5(\bx) = & 2.0\sin(\bx_5) + 1.5\,\text{erf}(1.5\bx_5) + 0.5\sin(\bx_2);
\end{align*}
and for all remaining $k$:
\begin{equation*}
f_k(\bx) = \sin\left(\bx_{i_k}\right),
\end{equation*}
where $i_k \equiv \{(k-1) \bmod r\} + 1$.
\end{subequations}

For csOPG and KSIR, rank selection is performed via cross-validation using the \texttt{mave.dim()} function from the MAVE R package. For SIR, we apply a rank selection procedure analogous to Algorithm~\ref{alg:ars}, adapted to the slice covariance matrix; see Algorithm~\ref{alg:sir-rank} for details. Using a common rank selection criterion across methods allows us to isolate the performance of the matrix estimators themselves. For MLP, the permutation-based approach is computationally prohibitive; we instead use a conservative energy-based criterion, with full details deferred to~\ref{sec:app-ss&rda}. The maximum candidate rank is set to 7 for all methods, as the true rank belongs to $\{1,3,5\}$. The results are summarized in Table~\ref{tb:rs}.

The proposed methods consistently outperform traditional baselines in most settings.

Under light-tailed settings, including Isotropic Gaussian and AR(1) Gaussian, the proposed methods remain stable across all ranks. Traditional methods such as csOPG and KSIR are competitive at $r=1$ and $r=3$, but degrade markedly at $r=5$.

For heavy-tailed settings, including Log-normal and Exp-trans AR(1), the patterns are less systematic than in the light-tailed case. Nevertheless, traditional methods suffer severely: SIR consistently selects the maximum allowable rank, MLP nearly always fails, csOPG and KSIR also perform worse than in the light-tailed case. In contrast, the proposed methods remain robust overall.

Within the proposed methods, TRE generally outperforms DNN in most settings, achieving lower errors and smaller standard deviations due to the stability provided by the strong prior constraints from the RKHS framework and the curl-free kernel, except under heavy-tailed distributions with $r=3$, where DNN configurations achieve lower errors. The benefit of unlabeled data is also evident: the split and full-data configurations achieve the lowest errors and greatest stability, while $\mathrm{S}_{\mathrm{NN}}^{\mathrm{sup}}$ suffers from failure rates exceeding 80\% in certain Gaussian settings.

\begin{table}[htpb]
\centering
\tiny
\setlength{\tabcolsep}{2pt}
\caption{Rank selection results: absolute error $|\widehat{r} - r|$ for each method across four data distributions and three true ranks.}
\label{tb:rs}
\begin{tabular}{lcccccccccccc}
\toprule
\multirow{2}{*}{Method} & \multicolumn{12}{c}{Data generating distribution} \\
\cmidrule(lr){2-13}
 & \multicolumn{3}{c}{Isotropic Gaussian} & \multicolumn{3}{c}{AR(1) Gaussian} & \multicolumn{3}{c}{Log-normal} & \multicolumn{3}{c}{Exp-trans AR(1)} \\
\cmidrule(lr){2-4} \cmidrule(lr){5-7} \cmidrule(lr){8-10} \cmidrule(lr){11-13}
 & $r=1$ & $r=3$ & $r=5$ & $r=1$ & $r=3$ & $r=5$ & $r=1$ & $r=3$ & $r=5$ & $r=1$ & $r=3$ & $r=5$ \\
\midrule
\multicolumn{13}{c}{\textit{External baselines}} \\
SIR & $0.13 (0.44)$ & $2.91 (2.45)$ & $4.43 (1.08)$ & $0.22 (0.67)$ & $2.96 (2.56)$ & $4.27 (1.34)$ & $9.00 (0.00)$ & $7.00 (0.00)$ & $5.00 (0.00)$ & $9.00 (0.00)$ & $7.00 (0.00)$ & $5.00 (0.00)$ \\
csOPG & $0.28 (0.55)$ & $\textbf{0.26 (0.44)}$ & $0.59 (0.65)$ & $0.16 (0.42)$ & $\textbf{0.29 (0.48)}$ & $0.85 (0.61)$ & $3.13 (1.76)$ & $1.88 (1.28)$ & $1.44 (1.17)$ & $3.19 (1.58)$ & $1.94 (1.47)$ & $1.28 (1.21)$ \\
KSIR & $0.09 (0.29)$ & $0.66 (0.65)$ & $1.09 (0.63)$ & $0.07 (0.26)$ & $0.68 (0.68)$ & $1.23 (0.68)$ & $2.26 (1.05)$ & $3.72 (1.88)$ & $2.53 (1.53)$ & $2.20 (1.19)$ & $4.11 (2.06)$ & $3.22 (1.52)$ \\
MLP & $9.00 (0.00)$ & $7.00 (0.00)$ & $5.00 (0.00)$ & $9.00 (0.00)$ & $7.00 (0.00)$ & $5.00 (0.00)$ & $7.52 (0.77)$ & $7.00 (0.00)$ & $4.98 (0.14)$ & $7.41 (0.83)$ & $7.00 (0.00)$ & $4.99 (0.10)$ \\
\midrule
\multicolumn{13}{c}{\textit{Framework: contrast configurations}} \\
$\mathrm{S}_{\mathrm{NN}}^{\mathrm{sup}}$ & $1.23 (0.42)^{*}$ & $1.43 (0.86)$ & $1.16 (1.14)$ & $1.17 (0.50)^{*}$ & $1.65 (1.00)$ & $1.08 (1.02)$ & $\textbf{0.00 (0.00)}$ & $1.60 (0.53)$ & $3.66 (0.49)$ & $\textbf{0.00 (0.00)}$ & $1.65 (0.52)$ & $3.61 (0.51)$ \\
$\mathrm{S}_{\mathrm{NN}}^{\mathrm{full}}$ & $0.11 (0.34)$ & $0.34 (0.85)$ & $0.41 (0.81)$ & $0.12 (0.32)$ & $0.45 (0.93)$ & $0.45 (0.85)$ & $\textbf{0.00 (0.00)}$ & $\textbf{0.70 (0.67)}$ & $0.86 (0.82)$ & $\textbf{0.00 (0.00)}$ & $0.84 (0.72)$ & $0.51 (0.67)$ \\
$\mathrm{S}_{\mathrm{TRE}}^{\mathrm{split}}$ & $0.07 (0.26)$ & $0.38 (0.86)$ & $\textbf{0.32 (0.73)}$ & $0.12 (0.50)$ & $0.48 (0.88)$ & $0.29 (0.73)$ & $\textbf{0.00 (0.00)}$ & $1.21 (0.70)$ & $0.14 (0.35)$ & $\textbf{0.00 (0.00)}$ & $1.35 (0.61)$ & $\textbf{0.14 (0.35)}$ \\
\midrule
\multicolumn{13}{c}{\textit{Framework: core configurations}} \\
$\mathrm{S}_{\mathrm{NN}}^{\mathrm{split}}$ & $0.11 (0.31)$ & $0.34 (0.92)$ & $0.40 (0.82)$ & $0.19 (0.46)$ & $0.58 (1.20)$ & $0.44 (0.78)$ & $\textbf{0.00 (0.00)}$ & $0.73 (0.66)$ & $1.18 (0.93)$ & $\textbf{0.00 (0.00)}$ & $\textbf{0.74 (0.63)}$ & $0.84 (0.82)$ \\
$\mathrm{S}_{\mathrm{TRE}}^{\mathrm{sup}}$ & $0.09 (0.38)$ & $0.31 (0.74)$ & $0.37 (0.87)$ & $\textbf{0.03 (0.17)}$ & $0.31 (0.72)$ & $\textbf{0.26 (0.64)}$ & $\textbf{0.00 (0.00)}$ & $2.02 (0.42)$ & $0.36 (0.48)$ & $\textbf{0.00 (0.00)}$ & $1.99 (0.50)$ & $0.45 (0.50)$ \\
$\mathrm{S}_{\mathrm{TRE}}^{\mathrm{full}}$ & $\textbf{0.06 (0.24)}$ & $0.35 (0.82)$ & $0.32 (0.77)$ & $0.08 (0.34)$ & $0.46 (0.87)$ & $0.32 (0.80)$ & $\textbf{0.00 (0.00)}$ & $1.33 (0.68)$ & $\textbf{0.11 (0.31)}$ & $\textbf{0.00 (0.00)}$ & $1.45 (0.59)$ & $\textbf{0.14 (0.35)}$ \\
\bottomrule
\end{tabular}
\begin{flushleft}
\tiny $^{*}$ The failure rates (NaN) of this method are 87\% (Isotropic Gaussian, $r=1$) and 82\% (AR(1) Gaussian, $r=1$); the reported mean and standard deviation are computed only over successful runs.
\end{flushleft}
\end{table}

\section{Real Date Analysis}\label{sec:rda}

We evaluate the proposed method on the M1 Patch-seq dataset \citep{scala2021phenotypic}, one of the largest publicly available Patch-seq datasets to date. The dataset contains $n = 1213$ neurons from the primary motor cortex of adult mice, covering all neuronal types, both excitatory and inhibitory. For each neuron, $q = 16$ electrophysiological properties were recorded as the response variables $\by$, and $p = 1000$ highly variable genes were selected as the predictors $\bx$, following the feature selection procedure described in the original publication.

We then applied the pre-processing pipeline of \citet{kobak2021sparse}. In line with their approach, we performed cross-validation to select 100 genes as the final input variables.\footnote{We chose this number because the original study reported selecting roughly 100 genes but did not provide the exact cross-validation parameters.} The gene expression counts were converted to counts per million for sequencing depth normalization and then $\ln_2(x+1)$-transformed. Both the gene expression values and the electrophysiological properties were standardized to zero mean and unit variance.

We assume the electrophysiological responses $\by$ and the gene expressions $\bx$ satisfy the multi-response regression model \eqref{model:1}:
\begin{equation*}
\by = \cF(\bB^{\top}\bx) + \bepsilon.
\end{equation*}

To assess the performance of different methods, we randomly partition the full dataset into four disjoint subsets: a labeled training set $\mathcal{T}_{\ell}$, an unlabeled set $\mathcal{T}_{u}$, an independent labeled set $\mathcal{T}_{f}$ for fitting the link function, and a hold-out test set $\mathcal{V}$ for final evaluation. The three configurations use these subsets as follows: $\mathrm{S}^{\mathrm{sup}}$ uses $\mathcal{T}_{\ell}$ for both score estimation and cross-moment computation; $\mathrm{S}^{\mathrm{split}}$ trains the score estimator on $\mathcal{T}_u$ and computes the cross-moment on $\mathcal{T}_{\ell}$; $\mathrm{S}^{\mathrm{full}}$ uses the combined set $\mathcal{T}_{\ell,u} = \mathcal{T}_{\ell} \cup \mathcal{T}_u$ for both purposes.

To investigate the impact of the training sample size and the labeled-to-unlabeled ratio, we fix the proportions of $\mathcal{T}_{f}$ and $\mathcal{V}$ at $0.15$ and $0.05$, respectively, and vary the proportion of $\mathcal{T}_{\ell,u}$ between $0.6$ and $0.8$. Within this combined set, we further vary the fraction of labeled samples, $\gamma = |\mathcal{T}_{\ell}| / (|\mathcal{T}_{\ell}| + |\mathcal{T}_{u}|)$, from $0.1$ to $0.5$.

The competing methods considered in this real data analysis are the same as those described in the simulation study. For each method, the rank is first determined using the approach described in Section~\ref{sec:ss-rs}, and the CS is then estimated accordingly. For csOPG and KSIR, which rely on cross-validation for rank selection, we set the maximum candidate dimension to $r_{\max} = 12$, following the empirical finding that the CS dimension in this dataset is approximately $10$ under the reduced rank regression model \citep{kobak2021sparse}. For SIR, we use $h = 12$ slices. For each configuration, after obtaining the corresponding estimate $\widehat{\bB}$ via the procedure described in Section~\ref{sec:model}, we project the covariates in $\mathcal{T}_{f}$ and $\mathcal{V}$ to obtain the low-dimensional representations $\widehat{\bB}^{\top}\bx$ for both sets. On $\mathcal{T}_{f}$, we fit the link function $\cF$ using a single-hidden-layer neural network, with the projected features $\widehat{\bB}^{\top} \bx$ as inputs and the electrophysiological responses $\by$ as the target. The network parameters are estimated by minimizing the mean squared error loss. The test MSE on $\mathcal{V}$ then serves as the criterion for assessing the quality of the estimated CS.

The results are reported as the mean over 100 independent repetitions, with standard deviations in parentheses. For each configuration, we present the average estimated rank $\widehat{r}$, the test MSE obtained using the adaptively selected rank, and the test MSE obtained using a fixed reference rank $r=10$.\footnote{For csOPG, rank selection remains feasible at $\gamma=0.1$, but subspace estimates occasionally yield NaN values; we report results only over successful trials. The success rates for csOPG and KSIR are 0.74 at training ratio 0.6 and 0.84 at 0.8.} The minimum test MSE based on the estimated rank $\widehat{r}$ in each column is boldfaced.

We first examine the rank selection performance of all methods.

In this nonlinear setting, the true rank is expected to be smaller than the linear benchmark of 10, since nonlinear dependencies can often be captured with fewer latent directions. 

We begin with the external baselines. SIR and MLP select ranks that are consistently large, often 11 or 12, well above the reference rank of 10. Despite their seemingly reasonable MSE, which is competitive in some cases, these large ranks do not reflect successful recovery of the true underlying structure; rather, they suggest overfitting or an inability to identify a parsimonious representation. In contrast, csOPG and KSIR select ranks that are too small, and their MSE based on the selected rank is noticeably higher than that using the reference rank, indicating a clear loss of predictive information due to rank underestimation. Overall, none of the baselines provides a reliable rank estimate.

Turning to our framework, we first consider the contrast configurations. $\mathrm{S}_{\mathrm{TRE}}^{\mathrm{split}}$ selects ranks around 8 to 9 and achieves MSE close to the reference, indicating that the selected rank is appropriate. $\mathrm{S}_{\mathrm{NN}}^{\mathrm{full}}$ also selects ranks around 9 with comparable MSE. In contrast, $\mathrm{S}_{\mathrm{NN}}^{\mathrm{sup}}$ selects only 3 to 7 directions, and the corresponding MSE is markedly higher than that using the reference rank, confirming that the rank is underestimated.

For the TRE-based core configurations, $\mathrm{S}_{\mathrm{TRE}}^{\mathrm{sup}}$ and $\mathrm{S}_{\mathrm{TRE}}^{\mathrm{full}}$ perform similarly to $\mathrm{S}_{\mathrm{TRE}}^{\mathrm{split}}$, all selecting ranks around 8 to 9 with MSE close to the reference. For the DNN core configuration, $\mathrm{S}_{\mathrm{NN}}^{\mathrm{split}}$ performs substantially better than $\mathrm{S}_{\mathrm{NN}}^{\mathrm{sup}}$, and is comparable to $\mathrm{S}_{\mathrm{NN}}^{\mathrm{full}}$ with no notable degradation.

In terms of predictive performance, the TRE methods achieve the lowest test MSE among all methods. Among the TRE configurations, $\mathrm{S}_{\mathrm{TRE}}^{\mathrm{sup}}$ achieves the lowest MSE, followed closely by $\mathrm{S}_{\mathrm{TRE}}^{\mathrm{full}}$ and $\mathrm{S}_{\mathrm{TRE}}^{\mathrm{split}}$, with only minor differences among them. Although the simulation results suggested that unlabeled data could be beneficial under heavy-tailed distributions, this advantage did not carry over to the real data application. One possible explanation is that the specific structure of the Patch-seq data or the limited sample size diminishes the additional value of unlabeled samples in this setting.

For the DNN configurations, $\mathrm{S}_{\mathrm{NN}}^{\mathrm{split}}$ performs similarly to $\mathrm{S}_{\mathrm{NN}}^{\mathrm{full}}$, confirming that the split configuration does not incur substantial loss; given its theoretical preference, it is the recommended choice among the DNN variants. $\mathrm{S}_{\mathrm{NN}}^{\mathrm{sup}}$ underperforms due to insufficient labeled data. The DNN methods generally underperform relative to TRE, consistent with the DNN's known requirement for large training sets. The difference in test MSE between training proportions 0.6 and 0.8 is small across all methods, suggesting that both training proportions remain in a small-sample regime where the modest increase in total sample size does not yield substantial improvement. In contrast, increasing $\gamma$ from 0.1 to 0.5, which directly increases the labeled sample size, leads to meaningful gains for most methods.

\begin{table}[t]
\centering
\tiny
\setlength{\tabcolsep}{2pt} 
\caption{Estimated rank $\widehat{r}$, and test MSE based on $\widehat{r}$ and on the reference rank $r=10$, for training proportions 0.6 and 0.8.}
\label{tab:real_combined}
\begin{tabular}{c l *{3}{c c c}}
\toprule
\multirow{2}{*}{Train ratio} & \multirow{2}{*}{Method} & \multicolumn{9}{c}{$\gamma$} \\
\cmidrule(lr){3-11}
 & & \multicolumn{3}{c}{0.1} & \multicolumn{3}{c}{0.3} & \multicolumn{3}{c}{0.5} \\
\cmidrule(lr){3-5} \cmidrule(lr){6-8} \cmidrule(lr){9-11}
 & &$\widehat{r}$ & MSE ($\widehat{r}$) & MSE ($r=10$) & $\widehat{r}$ & MSE ($\widehat{r}$) & MSE ($r=10$) & $\widehat{r}$ & MSE ($\widehat{r}$) & MSE ($r=10$) \\
\midrule
\multirow{11}{*}{0.6} & \multicolumn{10}{c}{\textit{External baselines}} \\
                     & SIR & 12.00(0.00) & 1.0197(0.1007) & 1.0328(0.0893) & 11.49(0.50) & 0.6795(0.0562) & 0.6813(0.0571) & 11.00(0.00) & 0.6264(0.0515) & 0.6257(0.0510) \\
                     & csOPG & 4.53(1.47) & 0.7854(0.0871) & 0.7475(0.0732) & 5.07(1.42) & 0.7571(0.0774) & 0.6909(0.0635) & 5.59(1.49) & 0.7262(0.0726) & 0.6785(0.0670) \\
                     & KSIR & 3.28(1.13) & 0.7505(0.0752) & 0.7080(0.0708) & 3.38(0.88) & 0.6992(0.0729) & 0.6501(0.0589) & 3.36(0.87) & 0.6727(0.0666) & 0.6186(0.0571) \\
                     & MLP & 12.00(0.00) & \textbf{0.6632(0.0601)} & 0.6655(0.0581) & 12.00(0.00) & 0.6318(0.0555) & 0.6406(0.0544) & 12.00(0.00) & 0.6529(0.0594) & 0.6723(0.0559) \\
                     & \multicolumn{10}{c}{\textit{Framework: contrast configurations}} \\
                     & $\mathrm{S}_{\mathrm{TRE}}^{\mathrm{split}}$ & 8.90(0.30) & 0.6826(0.0602) & 0.6862(0.0634) & 8.34(0.47) & 0.6093(0.0555) & 0.6113(0.0545) & 7.94(0.31) & 0.5909(0.0540) & 0.5934(0.0531) \\
                     & $\mathrm{S}_{\mathrm{NN}}^{\mathrm{sup}}$ & 5.17(0.60) & 0.8249(0.0838) & 0.7595(0.0680) & 3.30(0.46) & 0.8145(0.0890) & 0.7138(0.0664) & 7.25(0.46) & 0.6487(0.0548) & 0.6377(0.0547) \\
                     & $\mathrm{S}_{\mathrm{NN}}^{\mathrm{full}}$ & 9.01(0.39) & 0.7562(0.0724) & 0.7610(0.0675) & 8.99(0.22) & 0.6561(0.0570) & 0.6575(0.0582) & 8.83(0.40) & 0.6329(0.0547) & 0.6250(0.0559) \\
                     & \multicolumn{10}{c}{\textit{Framework: core configurations}} \\
                     & $\mathrm{S}_{\mathrm{TRE}}^{\mathrm{sup}}$ & 9.27(1.12) & 0.6807(0.1140) & 0.6684(0.0813) & 9.04(0.20) & \textbf{0.6019(0.0536)} & 0.6038(0.0507) & 8.93(0.26) & \textbf{0.5844(0.0525)} & 0.5831(0.0503) \\
                     & $\mathrm{S}_{\mathrm{TRE}}^{\mathrm{full}}$ & 8.97(0.26) & 0.6788(0.0620) & 0.6783(0.0633) & 8.91(0.29) & 0.6067(0.0544) & 0.6078(0.0529) & 8.68(0.47) & 0.5854(0.0526) & 0.5880(0.0536) \\
                     & $\mathrm{S}_{\mathrm{NN}}^{\mathrm{split}}$ & 8.67(0.49) & 0.7592(0.0677) & 0.7583(0.0658) & 8.09(0.40) & 0.6689(0.0549) & 0.6641(0.0595) & 6.73(0.51) & 0.6497(0.0623) & 0.6446(0.0525) \\
\midrule
\multirow{11}{*}{0.8} & \multicolumn{10}{c}{\textit{External baselines}} \\
                     & SIR & 12.00(0.00) & 1.1073(0.1040) & 1.1373(0.1111) & 11.17(0.38) & 0.6776(0.0619) & 0.6746(0.0601) & 11.00(0.00) & 0.6189(0.0555) & 0.6188(0.0561) \\  
                     & csOPG & 4.40(1.45) & 0.7978(0.0783) & 0.7545(0.0639) & 4.94(1.38) & 0.7540(0.0757) & 0.6856(0.0600) & 6.02(1.52) & 0.7264(0.0823) & 0.6769(0.0750) \\
                     & KSIR & 2.96(1.10) & 0.7505(0.0698) & 0.7141(0.0644) & 3.36(0.91) & 0.6928(0.0657) & 0.6475(0.0583) & 3.16(0.92) & 0.6729(0.0665) & 0.6171(0.0554) \\
                     & MLP & 12.00(0.00) & 0.6647(0.0592) & 0.6630(0.0575) & 12.00(0.00) & 0.6278(0.0582) & 0.6370(0.0558) & 12.00(0.00) & 0.6474(0.0604) & 0.6700(0.0586) \\
                     & \multicolumn{10}{c}{\textit{Framework: contrast configurations}} \\
                     & $\mathrm{S}_{\mathrm{TRE}}^{\mathrm{split}}$ & 8.85(0.43) & 0.6797(0.0651) & 0.6787(0.0661) & 8.44(0.50) & 0.6090(0.0532) & 0.6118(0.0533) & 7.96(0.20) & 0.5890(0.0546) & 0.5917(0.0525) \\
                     & $\mathrm{S}_{\mathrm{NN}}^{\mathrm{sup}}$ & 4.95(0.61) & 0.8164(0.0822) & 0.7616(0.0768) & 3.88(0.52) & 0.7608(0.0826) & 0.7030(0.0622) & 7.22(0.44) & 0.6396(0.0610) & 0.6361(0.0585) \\
                     & $\mathrm{S}_{\mathrm{NN}}^{\mathrm{full}}$ & 8.95(0.33) & 0.7553(0.0727) & 0.7501(0.0719) & 8.97(0.17) & 0.6563(0.0627) & 0.6583(0.0614) & 8.82(0.38) & 0.6266(0.0574) & 0.6259(0.0576) \\
                     & \multicolumn{10}{c}{\textit{Framework: core configurations}} \\
                     & $\mathrm{S}_{\mathrm{TRE}}^{\mathrm{sup}}$ & 9.67(0.58) & \textbf{0.6531(0.0610)} & 0.6520(0.0607) & 9.04(0.24) & \textbf{0.6018(0.0574)} & 0.6047(0.0570) & 8.85(0.36) & \textbf{0.5836(0.0527)} & 0.5825(0.0529) \\
                     & $\mathrm{S}_{\mathrm{TRE}}^{\mathrm{full}}$ & 8.97(0.36) & 0.6756(0.0640) & 0.6758(0.0615) & 8.93(0.29) & 0.6043(0.0545) & 0.6044(0.0542) & 8.55(0.50) & 0.5842(0.0544) & 0.5832(0.0552) \\
                     & $\mathrm{S}_{\mathrm{NN}}^{\mathrm{split}}$ & 8.74(0.44) & 0.7592(0.0740) & 0.7568(0.0646) & 8.15(0.36) & 0.6657(0.0592) & 0.6633(0.0615) & 6.74(0.46) & 0.6467(0.0648) & 0.6383(0.0588) \\
\bottomrule
\end{tabular}
\end{table}

\section{Limitations and Future Work}\label{sec:fw}

The proposed method has two main limitations that point to promising future directions.

First, the SVD-based recovery of the CS relies on the population matrix $\bM_1 \in \mathbb{R}^{r \times q}$ having full row rank $r$. This condition requires that at least $r$ of the $q$ response components provide linearly independent gradient directions. It can fail if too few response components contribute meaningful information. For instance, if several response components are symmetric about zero and the distribution of $\bB^\top\bx$ is symmetric, the corresponding columns of $\bM_1$ vanish. Similarly, strong correlations among response functions can reduce the effective rank. In either case, the rank of $\bM_1$ can be smaller than $r$, leading to an underestimated dimension or biased subspace estimation. Developing more robust procedures that remain valid under rank-deficient $\bM_1$ is a direction for future work.

Second, the current framework assumes continuous covariates \(\mathbf{x}\), as Stein's identity relies on the differentiability of \(\ln p(\mathbf{x})\). This excludes discrete or categorical covariates, which are common in many applications. Recent developments in discrete Stein's methods, such as Concrete Score Matching \citep{meng2022}, could possibly extend the proposed SDR pipeline to mixed continuous-discrete predictors, broadening its applicability to modern datasets with heterogeneous data types.

\bibliographystyle{elsarticle-harv} 
\bibliography{main.bib}

\newpage
\appendix

\section{Proofs of Main Results} 

\subsection{Technical Lemmas}

\begin{lemma}\label{lem:sub-exp} Suppose that $x$ and $y$ are both sub-Gaussian random variables, define $z \coloneqq xy - \EE(xy)$. Then, $z$ is mean zero sub-exponential. Moreover, $\| z \|_{\psi_1} \leq C\| x \|_{\psi_2} \| y \|_{\psi_2}$, for some absolute constant $C$.

\end{lemma}

\begin{lemma}\label{lem:app-ge}
Let $\{\bx_i\}_{i=1}^n \subset \mathbb{R}^p$ and $\{\by_j\}_{j=1}^{n_\ell} \subset \mathbb{R}^{q}$ be i.i.d. random vectors, respectively. Define
\begin{align*}
\mathcal{Z} \coloneqq \left\{ \bx_{i,k} : i \in [n],\ k \in [p] \right\} \cup \left\{ \by_{j,m} : j \in [n_\ell],\ m \in [q] \right\}.
\end{align*}
Assume there exist constants $K>0$ and $c>0$ such that
\begin{align*}
\sup_{z \in \mathcal{Z}} \|z\|_{\psi_2} \le K,
\end{align*}
and for every $z \in \mathcal{Z}$,
\begin{align*}
\PP(|z| > t) \le 2\exp\left( -\frac{c t^2}{K^2} \right), \qquad \forall t > 0. 
\end{align*}

For any $0<\delta<1$ and any $0<\rho<1$, define the good event
\begin{align*}
\mathcal{G}
\coloneqq
\left\{
\max_{1\le i\le n} \|\bx_i\|_2 \le M_{\bx},
\quad
\max_{1\le j\le n_\ell} \|\by_j\|_2 \le M_{\by}
\right\},
\end{align*}
with
\begin{align*}
M_{\bx} = K \sqrt{\frac{p}{c} \ln\left(\frac{4 n^4 p}{\rho\delta}\right)},\qquad
M_{\by} = K \sqrt{\frac{q}{c} \ln\left(\frac{4 n_\ell^4 q}{\rho\delta}\right)}.
\end{align*}
Then, $\PP(\mathcal{G}^c) \le \rho\delta$.
\end{lemma}

\subsection{Proof of Lemma~\ref{lem:fs}}
\begin{proof}
For $\forall i \in [p], j \in [q]$, 
\begin{align*}
   & \EE\{\by_{j} \bs_{i}(\bx)\} = \EE[ \{f_{j}(\bB^{\T}\bx) +\epsilon_{j}\}\bs_{i}(\bx)] = \EE\{f_{j}(\bB^{\T}\bx)\bs_{i}(\bx)\}  + \EE\{\epsilon_{j}\bs_{i}(\bx) \} =  \EE\{ f_{j}(\bB^{\T}\bx)\bs_{i}(\bx)\}\\
    = & \int_{\RR^{p}}  f_{j}(\bB^{\T}\bx) \bs_{i}(\bx)p(\bx)\dx = -\int_{\RR^{p}} f_{j}(\bB^{\T}\bx) [\nabla_{\bx}p(\bx)]_{i}\dx\\
    = & \int_{\RR^{p}} [\nabla_{\bx} f_{j}(\bB^{\T}\bx)]_{i} p(\bx)\dx - \Delta\\
    = & \bb^{\T}_{i} \EE [\nabla_{\bz} \{ f_{j}(\bB^{\T}\bx)\}],
\end{align*}   
where $\bb_{i}$ is the $i$-th column of $\bB^{\T}$ and
\begin{align*}
   \Delta = &\int_{\RR^{p-1}}  \{\lim_{a \rightarrow \infty} f_{j}(\bB^{\T}\bx) p(\bx)|_{\bx = (x_{1}, \ldots,x_{i-1},a,  x_{i+1}, \ldots, x_{p}) } - \lim_{b \rightarrow -\infty} f_{j}(\bB^{\T}\bx) p(\bx)|_{\bx = (x_{1}, \ldots,x_{i-1},b,  x_{i+1}, \ldots, x_{p}) }\}\\
  & \text{d}(x_{1}, \ldots,x_{i-1}, x_{i+1}, \ldots, x_{p}) =  0.
\end{align*}

Stacking $\EE\{  \by_{j} \bs_{i}(\bx)\}, i \in [p]$  horizontally together, we have $\EE\{ \by_{j}\bs(\bx) \} =  \bB \EE [\nabla_{\bz} \{ f_{j}(\bB^{\T}\bx)\}]$, i.e., $\bM = \bB \bM_1$.  
\end{proof}

\subsection{Proof of Theorem~\ref{thm:subgcomp}}
\begin{proof}
Condition on $\cT_u$, $\widehat{\bs}^u$ is fixed.
By triangle inequality,
\begin{align*}
\left\| \widehat{\bM}^{s}_{\ell,u} - \bM \right\|_\mathrm{F} 
\le
\underbrace{\left\| \widehat{\bM}^{s}_{\ell,u} - \EE(\widehat{\bM}^{s}_{\ell,u}) \right\|_\mathrm{F}}_{\text{Term I}}
+
\underbrace{\left\| \EE(\widehat{\bM}^s_{\ell,u}) - \bM \right\|_\mathrm{F}}_{\text{Term II}}.
\end{align*}

For Term I, let $z_{i}^{j,k}$ denote $\widehat{\bs}_k^u(\bx_i) \by_{i,j} - \EE\{\widehat{\bs}_k^u(\bx) \by_j\}$.
Since $\widehat{\bs}^u_k$ and $\by_j$ are sub-Gaussian, by Lemma~\ref{lem:sub-exp}, 
$z_{i}^{j,k}$ is sub-exponential. Moreover, there exists some constant $C > 0$ such that $\|z_{i}^{j,k}\|_{\psi_1} \le CK^2$. Let $K'$ denote $CK^2$, by the sub-exponential Bernstein's inequality, there exists an absolute constant $C>0$ such that for any $\delta \in (0,1)$, with probability at most $\delta$,
\begin{equation*}
    \left| \frac{1}{n_{\ell}}\sum_{i=1}^{n_{\ell}} z_i^{j,k} \right|
    \ge C K' \left\{ \sqrt{\frac{\ln\left(\frac{1}{\delta}\right)}{n}} + \frac{\ln\left(\frac{1}{\delta}\right)}{n} \right\}.
\end{equation*}

Applying the union bound over all $p \times q$ entries gives
\begin{align*}
\text{Term I}
=\left\| \widehat{\bM}^{s}_{\ell,u} - \EE(\widehat{\bM}^{s}_{\ell,u}) \right\|_\mathrm{F}  \le C K' \left\{ \sqrt{\frac{pq\ln\left(\frac{pq}{\delta}\right)}{n}} + \frac{\sqrt{pq}\ln\left(\frac{pq}{\delta}\right)}{n} \right\} = C K^2 \left\{ \sqrt{\frac{pq\ln\left(\frac{pq}{\delta}\right)}{n}} + \frac{\sqrt{pq}\ln\left(\frac{pq}{\delta}\right)}{n} \right\},
\end{align*}
with probability at least $1 - \delta/2$, for some absolute constant $C$.

By Jensen's inequality and Cauchy-Schwarz inequality, we obtain, with probability at least $1-\delta/2$,  
\begin{align*}
&\text{Term II} \\
= & \left\| \EE \left\{ \frac{1}{n_{\ell}}\sum_{i=1}^{n_{\ell}} \widehat{\bs}(\bx_i)\by_i^\top \right\} -   \EE\{s(\bx)\by^\top\}
 \right \|_\mathrm{F} \\
 = & \left\| \EE \left\{  \widehat{\bs}(\bx)\by^\top \right\} -   \EE\{s(\bx)\by^\top\}
 \right \|_\mathrm{F} \\
 \le & \EE \left\{ \left\|   \widehat{\bs}(\bx)\by^\top  -   s(\bx)\by^\top
 \right \|_\mathrm{F}\right\} \\
 = & \EE \left\{ \| \widehat{\bs}(\bx)-s(\bx)\|_2 \|\by\|_2\right\}\\
\le & \sqrt{\EE\{\| \widehat{\bs}(\bx)-s(\bx)\|^2_2\}\EE(\| \by \|_2^2)}\\
\le & C K^2 \sqrt{q}  \epsilon(n_u),
\end{align*}
for some absolute constant $C>0$.

Taking the union bound with the event for Term II yields the desired result.
\end{proof}

\subsection{Proof of Theorem~\ref{thm:kernel}}

\subsubsection{Properties of General Hermitian Positive-Definite Reproducing Operator-Valued kernels}

\begin{definition}[Hermitian positive-definite reproducing operator-valued kernel]
\label{def:operator_kernel}
An operator-valued kernel $\cK:\mathcal{X}\times\mathcal{X}\to L(\mathcal{Y})$ is called the Hermitian positive-definite reproducing kernel of the RKHS $\mathcal{H}$ if and only if:

\begin{enumerate}
\item[(i)] $\forall \bx\in\mathcal{X}$, $\forall \by\in\mathcal{Y}$, the mapping $\bx' \mapsto \cK(\bx',\bx)\by$ belongs to $\mathcal{H}_{\cK}$.

\item[(ii)] $\forall \boldsymbol{f}\in\mathcal{H}$, $\forall \bx\in\mathcal{X}$, $\forall \by\in\mathcal{Y}$, $\langle \boldsymbol{f}(\bx), \by\rangle_{\mathcal{Y}} = \langle \boldsymbol{f},\; \cK(\cdot,\bx)\by\rangle_{\mathcal{H}_{\cK}}$.
\item[(iii)] $\forall \bx_1,\bx_2\in\mathcal{X}$,
$\cK(\bx_1,\bx_2) = \cK(\bx_2, \bx_1)^* \in L(\mathcal{Y})$, 
where $*$ denotes the adjoint.

\item[(iv)] $\forall n\ge 1$, $\forall \{\bx_i\}_{i=1}^n, \{\bx'_i\}_{i=1}^n \subset \mathcal{X}$, $\forall \{\by_i\}_{i=1}^n, \{\by'_i\}_{i=1}^n \subset \mathcal{Y}$, $\sum_{k,\ell=1}^{n} \langle \cK(\cdot,\bx_k)\by_k,\; \cK(\cdot,\bx'_\ell)\by'_\ell\rangle_{\mathcal{H}_{\cK}} \ge 0$.
\end{enumerate}
\end{definition}

\begin{lemma}\label{lem:app-rp}
Let $\mathcal{H}_{\mathcal{K}}$ be a vector-valued RKHS associated with a matrix-valued kernel
\begin{align*}
\mathcal{K}: \mathcal{X}\times\mathcal{X} \to \mathbb{R}^{p\times p}.
\end{align*}
Then for any $\boldsymbol{f}\in\mathcal{H}_{\mathcal{K}}$ and any $\bx\in\mathcal{X}$,
\begin{align*}
\|\boldsymbol{f}(\bx)\|_2 \le \sqrt{\|\mathcal{K}(\bx,\bx)\|_{\mathrm{op}}} \, \|\boldsymbol{f}\|_{\mathcal{H}_{\mathcal{K}}}.
\end{align*}
\end{lemma}

\begin{lemma}\label{lem:replacement-stability}
Let $\mathcal{H}_{\cK}$ be a vector-valued RKHS with reproducing kernel $\cK$ satisfying Definition~\ref{def:operator_kernel}. Further assume that
\begin{align*}
\sup_{\bx\in\mathcal{X}} \|\cK(\bx,\bx)\|_{\mathrm{op}} \le \kappa^2 < \infty.
\end{align*}
Given a labeled training set $\cT_{\ell} = \{(\bx_i,\by_i)\}_{i=1}^{n_{\ell}}$, satisfying $\max_{i \in [n]}\| \by_i \|_2 \leq M_{\by}$, the kernel ridge regression estimator (KRRE) is defined as 
\begin{equation*}
\bg^{\cK,\lambda}_{\cT_{\ell}}
\coloneqq \arg\min_{\boldsymbol{f}\in\mathcal{H}_{\cK}}\left\{
\frac{1}{n_{\ell}}\sum_{i=1}^{n_{\ell}} \left \|\boldsymbol{f}(\bx_i)-\by_i\right \|_2^2 + \lambda \|\boldsymbol{f}\|_{\cH_{\cK}}^2
\right\},
\quad \lambda>0.   
\end{equation*}
Let $\cT^{j}_{\ell}$ be the dataset obtained by replacing the $j$-th sample with an independent copy $(\bx_j',\by_j')$, and let $\bg^{\cK,\lambda}_{\cT^{j}_{\ell}}$ be the corresponding KRRE. If $\|\by'_j\|_2 \le M_{\by}$,
 then, for any test point $\bx \in \cX$,
\begin{align*}
\left\| \bg^{\cK,\lambda}_{\cT_{\ell}}(\bx) - \bg^{\cK,\lambda}_{\cT_{\ell}^{j}}(\bx) \right\|_2
\le
\frac{CM_{\by}}{n_{\ell}},
\end{align*}
where $C = 2(1 + \kappa/\sqrt{\lambda})\kappa^2/\lambda$.
\end{lemma}

\subsubsection{Properties of TRE with Curl-Free IMQ Kernels}

In the following proof, given dataset $\bX = \{ \bx_i \}_{i=1}^{n}$ and kernel $\cK$, let $\mathrm{div}_{\bx}  \cK(\bx, \cdot)$ be the divergence of $\cK(\bx, \cdot)$, defined as a vector-valued
function, whose $i$-th component is the divergence of the
$i$-th column of $\cK(\bx, \cdot)$. Then, define 
\begin{align*}
\widehat{\bzeta}^{\cK}_{\bX}(\bx) \coloneqq \frac{1}{n}\sum_{j=1}^n \mathrm{div}_{\bx} \cK(\bx,\bx_j)^\top.    
\end{align*}

By Theorem 3.1 of~\citet{zhou2020}, the TRE can be decomposed as
\begin{align*}
\widehat{\bs}^{\cK,\lambda}_{\bX}(\bx) = \bg^{\cK,\lambda}_{\cT_{\cK,\bX}}(\bx) - \frac{1}{\lambda}\widehat{\bzeta}^{\cK}_{\bX}(\bx),
\end{align*}
where $\bg^{\cK,\lambda}_{\cT_{\cK,\bX}}$ is the KRRE defined in Lemma~\ref{lem:replacement-stability} estimated on the training set $\cT_{\cK,\bX} = $\\$[\{\bx_i,(1/\lambda) \widehat{\bzeta}^{\cK}_{\bX}(\bx_i)\}]_{i=1}^n$.

\begin{lemma}[Replacement Stability]
\label{lem:kef-replacement}

Let $\cH_{\cK}$ be a vector-valued RKHS with kernel $\cK$ satisfying Definition~\ref{def:operator_kernel} and
\begin{align*}
\sup_{\bx\in\mathcal{X}} \|\cK(\bx,\bx)\|_{\mathrm{op}} \le \kappa^2 < \infty,
\qquad
\sup_{\bx,\by \in \cX} \|\mathrm{div}_{\bx} \cK(\bx,\by)^\top\|_2 \le M_{\bzeta} < \infty.
\end{align*}
Fix $\lambda>0$, let $\widehat{\bs}^{\cK,\lambda}_{\bX}$ be the TRE with regularization parameter $\lambda$ and kernel $\cK$ estimated on dataset $\bX = \{ \bx_i \}_{i=1}^{n}$, as defined in Theorem 3.1 of~\citet{zhou2020}.
Let $\bX^{j}$ be the dataset obtained by replacing the $j$-th sample of $\bX$ with an independent copy $\bx_j'$, and let $\widehat{\bs}^{\cK,\lambda}_{\bX^{j}}$ be the corresponding TRE. Then for any test point $\bx \in \cX$,
\begin{align*}
\left \|\widehat{\bs}^{\cK,\lambda}_{\bX}(\bx) - \widehat{\bs}^{\cK,\lambda}_{\bX^{j}}(\bx) \right \|_2
\le \left\{ \frac{2\left(1 + \frac{\kappa}{\sqrt{\lambda}} \right)\kappa^2}{\lambda} + 2\right \}\frac{M_{\bzeta}}{n}.
\end{align*}
\end{lemma}

Let $k_{\mathrm{I}}(\bx,\by) \coloneqq \left(1 + \|\bx-\by\|_2^2/c^2\right)^{-1/2}$, $c>0$, be the scalar IMQ kernel on $\mathbb{R}^p$.

\begin{lemma}\label{lem:app-m3}
All third-order mixed partial derivatives of $k_{\mathrm{I}}$ are globally bounded. Specifically, the following holds. 
\begin{align*}
M_3 \coloneqq \sup_{\bx,\by\in\mathbb{R}^p} \max_{i,j,k} \left| \frac{\partial^3}{\partial x_i \partial x_k \partial y_j} k_1(\bx,\by) \right| < \infty.
\end{align*}
\end{lemma}

For $\phi(r): [0,\infty) \to \mathbb{R}$, a twice continuously differentiable radial function, the curl-free matrix-valued kernel is defined by
\begin{align*}
\mathcal{K}_{\mathrm{cf}}(\bx, \by) \coloneqq -\nabla_{\bx}^2 \phi(\|\bx - \by\|_2), 
\qquad \bx, \by \in \mathbb{R}^p.
\end{align*}

We take
$\phi(r) = \left(1 + r^2/c^2\right)^{-1/2} 
$, the corresponding scalar kernel becomes the IMQ kernel, $k_{\mathrm{I}}(\bx,\by)$.
We name the matrix-valued kernel curl-free IMQ  kernel, denoted by $\cK_{\mathrm{CI}}$.

\begin{lemma}\label{lem:app-bd-general}
Let $M^{\mathrm{CI}}_{\bzeta} = \sup_{\bx, \by\in\mathbb{R}^p} \left\| \mathrm{div}_{\bx} \mathcal{K}_{\mathrm{CI}}(\bx, \by)^\top \right\|_2$, then, for $p \ge 3$, $M^{\mathrm{CI}}_{\bzeta} \asymp p$.
\end{lemma}

\begin{lemma}\label{lem:app-opb}
The operator norm $\|\mathcal{K}_{\mathrm{CI}}(\bx, \bx)\|_{\mathrm{op}}$ is uniformly bounded for all $\bx\in \cX$, and moreover,
\begin{align*}
\sup_{\bx\in\cX} \|\mathcal{K}_{\mathrm{CI}}(\bx, \bx)\|_{\mathrm{op}} = \frac{1}{c^2}.
\end{align*}
\end{lemma}
In the following, let $\kappa_{\mathrm{CI}}$ denote $1/c$.

\begin{lemma}[Pointwise boundedness of the TRE estimator]
\label{lem:kef-prop}
Let $\widehat{\bs}^{\cK_{\mathrm{CI}},\lambda}_{\bX}$ be the TRE with curl-free IMQ  kernel trained on a fixed dataset $\bX = \{\bx_i\}_{i=1}^n$ with regularization parameter $\lambda>0$. Then, for any test point $\bx \in \cX$,
\begin{align*}
\|\widehat{\bs}^{\cK_{\mathrm{CI}},\lambda}_{\bX}(\bx)\|_2 \le \left( 1 + \kappa_{\mathrm{CI}}\right) \frac{M^{\mathrm{CI}}_{\bzeta}}{\lambda}.
\end{align*}
Moreover, for any $\bx,\bx'\in \cX$,
\begin{align*}
\left \|\widehat{\bs}^{\cK_{\mathrm{CI}},\lambda}_{\bX}(\bx) - \widehat{\bs}^{\cK_{\mathrm{CI}},\lambda}_{\bX}(\bx') \right \|_2
\le \left(\frac{M_3 M^{\mathrm{CI}}_{\bzeta}}{\lambda^2} + \frac{M_3}{\lambda}\right) p^{\frac{3}{2}} \|\bx - \bx'\|_2.
\end{align*}
\end{lemma}

\subsubsection{Proof of Theorem~\ref{thm:kernel}}
\begin{proof}

We now establish the main convergence result for the Tikhonov regularized cross-moment estimator with the curl-free IMQ kernel, $\widehat{\bM}^{\cK_{\mathrm{CI}},\lambda}_{\ell,u}$. 
Recall the definition
\begin{align*}
\bM \coloneqq \EE\left\{\bs(\bx)\by^\top\right\}.
\end{align*}
By the triangle inequality,
\begin{align}
\|\widehat{\bM}^{\cK_{\mathrm{CI}},\lambda}_{\ell,u} - \bM^*\|_\mathrm{F}
\le
\underbrace{\|\widehat{\bM}^{\cK_{\mathrm{CI}},\lambda}_{\ell,u} - \EE(\widehat{\bM}^{\cK_{\mathrm{CI}},\lambda}_{\ell,u})\|_\mathrm{F}}_{\circled{1}}
+
\underbrace{\|\EE(\widehat{\bM}^{\cK_{\mathrm{CI}},\lambda}_{\ell,u}) - \bM^*\|_\mathrm{F}}_{\circled{2}}.    
\end{align}

We control the two terms respectively. First, we deal with 
$\circled{1}$. Considering the good event defined in Lemma~\ref{lem:app-ge}
\begin{align*}
\mathcal{G} \coloneqq \left\{
\max_{1\le i\le n} \|\bx_i\|_2 \le M_{\bx} \text{, and }
\max_{1\le j\le n_\ell} \|\by_j\|_2 \le M_{\by}
\right\},
\end{align*}
with
\begin{align*}
M_{\bx} \coloneqq K \sqrt{\frac{p}{c} \ln\left(\frac{4 n^4 p}{\rho\delta}\right)}\text{, and }
M_{\by} \coloneqq K \sqrt{\frac{q}{c} \ln\left(\frac{4 n_\ell^4 q}{\rho\delta}\right)},
\end{align*}
we obtain $\PP(\mathcal{G}^c) \le \rho\delta$. 

Define the projection operators $\mathcal{P}_{\bx}:\mathbb{R}^p\to\mathbb{R}^p$ and $\mathcal{P}_{\by}:\mathbb{R}^{q}\to\mathbb{R}^{q}$:
\begin{align*}
\mathcal{P}_{\bx}(\bx)\coloneqq
\begin{cases}
\bx, & \|\bx\|_2\le M_{\bx},\\[6pt]
M_{\bx}\dfrac{\bx}{\|\bx\|_2}, & \|\bx\|_2>M_{\bx},
\end{cases}
\end{align*}
and 
\begin{align*}
\mathcal{P}_{\by}(\by)\coloneqq
\begin{cases}
\by, & \|\by\|_2\le M_{\by},\\[6pt]
M_{\by}\dfrac{\by}{\|\by\|_2}, & \|\by\|_2>M_{\by}.
\end{cases}
\end{align*}

Both projection operators are 1-Lipschitz.

For the training set $\cT_{\ell,u}$, define the dataset-level projection operator $\Pi$ as
\begin{align*}
 \Pi(\cT_{\ell,u})\coloneqq [\left\{ \mathcal P_{\bx}(\bx_i), \mathcal P_{\by}(\by_i) \right\}]_{i=1}^{n_\ell} \cup \left\{ \mathcal P_{\bx}(\bx_i) \right\}_{i=n_\ell+1}^{n}.
\end{align*}
For notational convenience, let $\widetilde{\bx}_i$ denote $ \mathcal P_{\bx}(\bx_i)$, $\widetilde{\by}_i $ denote $ \mathcal P_{\by}(\by_i)$, and $\widetilde{\cT}_{\ell,u} $ denote $ \Pi(\cT_{\ell,u})$.

Define the projected estimator
\begin{align}
\label{eq:proxy-est}
\widehat{\bM}^{\cK_{\mathrm{CI}},\lambda}_{\ell,u} \circ \Pi(\cT_{\ell,u})
\coloneqq
\frac{1}{n_\ell}\sum_{i=1}^{n_\ell}
\widehat{\bs}^{\cK_{\mathrm{CI}},\lambda}_{\widetilde{\mathcal{T}}_{\ell,u}}(\widetilde{\bx}_i)\,\widetilde{\by}_i^\top.
\end{align}
On the good event $\mathcal G$, we have $\widehat{\bM}^{\cK_{\mathrm{CI}},\lambda}_{\ell,u} \circ \Pi(\cT_{\ell,u}) = \widehat{\bM}^{\cK_{\mathrm{CI}},\lambda}_{\ell,u}(\cT_{\ell,u})$.

Let $\cT_{\ell,u}$ and $\cT_{\ell,u}'$ differ in exactly one sample. We consider two cases.

\noindent Case 1: Replacing an unlabeled sample.\\
Suppose the replaced sample is $\bx_j$ with $j \in [n]\setminus[n_\ell]$. All labeled pairs remain unchanged. By Lemma~\ref{lem:kef-replacement}, for every $i \in [n_\ell]$,
\begin{align}
\left\| \widehat{\bs}^{\cK_{\mathrm{CI}},\lambda}_{\widetilde{\mathcal{T}}_{\ell,u}}(\widetilde{\bx}_i)
-
\widehat{\bs}^{\cK_{\mathrm{CI}},\lambda}_{\widetilde{\mathcal{T}}_{\ell,u}'}(\widetilde{\bx}_i)
\right\|_2
\le \frac{C_{\mathrm{rep}}}{n}, \label{eq:rep-unlabeled}
\end{align}
where
\begin{align*}
C_{\mathrm{rep}} =
\left\{
\frac{2\left(1 + \frac{\kappa_{\mathrm{CI}}}{\sqrt{\lambda}} \right)\kappa^2_{\mathrm{CI}}}{\lambda} + 2
\right\} M^{\mathrm{CI}}_{\bzeta}.
\end{align*}
Combining \eqref{eq:proxy-est} and \eqref{eq:rep-unlabeled}, and using $\|\widetilde{\by}_i\|_2 \le M_{\by}$,
\begin{align}
\left\| \widehat{\bM}^{\cK_{\mathrm{CI}},\lambda}_{\ell,u}\circ \Pi(\cT) - \widehat{\bM}^{\cK_{\mathrm{CI}},\lambda}_{\ell,u}\circ \Pi(\cT') \right\|_\mathrm{F}
\le
\frac{C_{\mathrm{rep}} M_{\by}}{n}. \label{eq:case1bd}
\end{align}

\noindent Case 2: Replacing a labeled sample. \\
Let the $j$-th labeled sample $(\bx_j,\by_j)$ be replaced by $(\bx_j',\by_j')$. Using the triangle inequality,
\small{
\begin{align}
& \left\| \widehat{\bM}^{\cK_{\mathrm{CI}},\lambda}_{\ell,u}\circ \Pi(\cT) - \widehat{\bM}^{\cK_{\mathrm{CI}},\lambda}_{\ell,u}\circ \Pi(\cT') \right\|_\mathrm{F} \notag \\
\le &
\underbrace{
\frac{1}{n_\ell}\sum_{i=1}^{n_\ell}
\left\| 
\left\{
\widehat{\bs}^{\cK_{\mathrm{CI}},\lambda}_{\widetilde{\mathcal{T}}_{\ell,u}}(\widetilde{\bx}_i)
-
\widehat{\bs}^{\cK_{\mathrm{CI}},\lambda}_{\widetilde{\mathcal{T}}_{\ell,u}'}(\widetilde{\bx}_i)
\right\}
\widetilde{\by}_i^\top
\right\|_\mathrm{F}
}_{\text{(a) estimator change}} +
\underbrace{
\frac{1}{n_\ell}
\left\| 
\widehat{\bs}^{\cK_{\mathrm{CI}},\lambda}_{\widetilde{\mathcal{T}}_{\ell,u}'}(\widetilde{\bx}_j)
(\widetilde{\by}_j - \widetilde{\by}_j')^\top
\right\|_\mathrm{F}
}_{\text{(b) label change}} +
\underbrace{
\frac{1}{n_\ell}
\left\| 
\left\{
\widehat{\bs}^{\cK_{\mathrm{CI}},\lambda}_{\widetilde{\mathcal{T}}_{\ell,u}'}(\widetilde{\bx}_j)
-
\widehat{\bs}^{\cK_{\mathrm{CI}},\lambda}_{\widetilde{\mathcal{T}}_{\ell,u}'}(\widetilde{\bx}_j')
\right\}
(\widetilde{\by}_j')^\top
\right\|_\mathrm{F}
}_{\text{(c) input change}}.
\label{eq:three-term-decomp}
\end{align}}

By Lemma~\ref{lem:kef-replacement} and $\|\widetilde{\by}_i\|_2 \le M_{\by}$,
\begin{align}\label{eq:term-a-bound}
\text{(a)} \le \frac{C_{\mathrm{rep}} M_{\by}}{n}. 
\end{align}

Since $\widehat{\bs}^{\cK_{\mathrm{CI}},\lambda}_{\widetilde{\mathcal{T}}_{\ell,u}}$ is bounded (Lemma~\ref{lem:kef-prop}) and $\|\widetilde{\by}_j - \widetilde{\by}_j'\|_2 \le 2M_{\by}$,
\begin{align}
\text{(b)} \le \frac{C_{\mathrm{label}}}{n_\ell}, \label{eq:term-b-bound}
\end{align}
where $C_{\mathrm{label}} = 2 \left( 1 + \kappa_{\mathrm{CI}}\right) M^{\mathrm{CI}}_{\bzeta} M_{\by}/\lambda$.

By the Lipschitz property of the score estimator (Lemma~\ref{lem:kef-prop}),
\begin{align}
\text{(c)} \le \frac{C_{\mathrm{input}}}{n_\ell}, \label{eq:term-c-bound}
\end{align}
where $C_{\mathrm{input}} = 2 \left(M_3 M^{\mathrm{CI}}_{\bzeta}/\lambda^2 + M_3/\lambda\right) p^{3/2} M_{\bx}M_{\by}$.

Combining \eqref{eq:three-term-decomp} with \eqref{eq:term-a-bound}, \eqref{eq:term-b-bound}, and \eqref{eq:term-c-bound}, the total change in Case 2 is bounded by
\begin{align}
\frac{C_{\mathrm{rep}} M_{\by}}{n}
+
\frac{C_{\mathrm{label}} + C_{\mathrm{input}}}{n_\ell}. \label{eq:case2-bound}
\end{align}

Define
\begin{align}\label{eq:delta-def}
\Delta \coloneqq 
\frac{C_{\mathrm{rep}} M_{\by}}{n}
+
\frac{C_{\mathrm{label}} + C_{\mathrm{input}}}{n_\ell}. 
\end{align}
Combining \eqref{eq:case1bd} and \eqref{eq:case2-bound}, for any two datasets differing in one sample,
\begin{align}\label{eq:delta-bound}
\left\| \widehat{\bM}^{\cK_{\mathrm{CI}},\lambda}_{\ell,u}\circ \Pi(\cT_{\ell,u}) - \widehat{\bM}^{\cK_{\mathrm{CI}},\lambda}_{\ell,u}\circ \Pi(\cT'_{\ell,u}) \right\|_\mathrm{F} \le \Delta. 
\end{align}

Since $\Delta$ is deterministic and finite, McDiarmid's inequality applies to
\begin{align*}
\widetilde F(\cT_{\ell,u}) \coloneqq \|\widehat{\bM}^{\cK_{\mathrm{CI}},\lambda}_{\ell,u}\circ \Pi(\cT_{\ell,u}) - \EE\{\widehat{\bM}^{\cK_{\mathrm{CI}},\lambda}_{\ell,u}\circ \Pi(\cT_{\ell,u})\}\|_\mathrm{F}.
\end{align*}
For any $\epsilon > 0$,
\begin{align}
\PP\left\{ \widetilde F(\cT_{\ell,u}) \ge \epsilon \right\}
\le 2\exp\left( -\frac{2\epsilon^2}{n\Delta^2} \right). \label{eq:mcdiarmid-proxy}
\end{align}
Define
\begin{align}\label{eq:eps-def}
\epsilon \coloneqq \Delta \sqrt{\frac{n}{2} \ln\frac{8}{\delta}},     
\end{align}
we obtain, with probability at least $1-\delta/4$,
\begin{align*}
\widetilde F(\cT_{\ell,u}) \le \Delta \sqrt{\frac{n}{2} \ln\frac{4}{\delta}}.
\end{align*}

Define
\begin{align*}
F(\cT_{\ell,u}) \coloneqq \|\widehat{\bM}^{\cK_{\mathrm{CI}},\lambda}_{\ell,u}(\cT_{\ell,u}) - \EE\{\widehat{\bM}^{\cK_{\mathrm{CI}},\lambda}_{\ell,u}(\cT_{\ell,u})\}\|_\mathrm{F}
\end{align*}
as the deviation of the original estimator. We now relate $\EE\widehat{\bM}^{\cK_{\mathrm{CI}},\lambda}_{\ell,u}\circ \Pi$ to $\EE\widehat{\bM}^{\cK_{\mathrm{CI}},\lambda}_{\ell,u}$. The difference only appears on the bad event $\mathcal{G}^c$:
\begin{align}\label{eq:app-ed}
\left\|\EE\widehat{\bM}^{\cK_{\mathrm{CI}},\lambda}_{\ell,u}\circ \Pi - \EE\widehat{\bM}^{\cK_{\mathrm{CI}},\lambda}_{\ell,u}\right\|_\mathrm{F}
\le \EE\left\{
\|\widehat{\bM}^{\cK_{\mathrm{CI}},\lambda}_{\ell,u}(\widetilde{\mathcal{T}}_{\ell,u})-\widehat{\bM}^{\cK_{\mathrm{CI}},\lambda}_{\ell,u}(\mathcal{T}_{\ell,u})\|_\mathrm{F}
\cdot \mathbb{I}_{\mathcal{G}^c}
\right\}.
\end{align}

Applying the three-term decomposition to each summand in $\widehat{\bM}^{\cK_{\mathrm{CI}},\lambda}_{\ell,u}$, we obtain
\small{
\begin{align}\label{eq:three-term-bias}
&\|\widehat{\bM}^{\cK_{\mathrm{CI}},\lambda}_{\ell,u}(\widetilde{\mathcal{T}}_{\ell,u})-\widehat{\bM}^{\cK_{\mathrm{CI}},\lambda}_{\ell,u}(\mathcal{T}_{\ell,u})\|_\mathrm{F} \notag \\
\le &
\underbrace{
\frac{1}{n_\ell}\sum_{i=1}^{n_\ell}
\left\| 
\left\{
\widehat{\bs}^{\cK_{\mathrm{CI}},\lambda}_{\widetilde{\mathcal{T}}_{\ell,u}}(\widetilde{\bx}_i)
-
\widehat{\bs}^{\cK_{\mathrm{CI}},\lambda}_{\mathcal{T}_{\ell,u}}(\bx_i)
\right\}
\by_i^\top
\right\|_\mathrm{F}
}_{\text{(i) estimator change}} 
 + 
\underbrace{
\frac{1}{n_\ell}\sum_{i=1}^{n_\ell}
\left\| 
\widehat{\bs}^{\cK_{\mathrm{CI}},\lambda}_{\mathcal{T}_{\ell,u}}(\bx_i)
(\by_i - \widetilde{\by}_i)^\top
\right\|_\mathrm{F}
}_{\text{(ii) label projection error}} 
+ 
\underbrace{
\frac{1}{n_\ell}\sum_{i=1}^{n_\ell}
\left\| 
\left\{
\widehat{\bs}^{\cK_{\mathrm{CI}},\lambda}_{\mathcal{T}_{\ell,u}}(\bx_i)
-
\widehat{\bs}^{\cK_{\mathrm{CI}},\lambda}_{\mathcal{T}_{\ell,u}}(\widetilde{\bx}_i)
\right\}
\widetilde{\by}_i^\top
\right\|_\mathrm{F}
}_{\text{(iii) input projection error}}.
\end{align}}

We focus on Term (i), defined as the expectation of the first term in \eqref{eq:three-term-bias} on the bad event, i.e.,
\begin{align} \label{eq:app-3a}
\text{Term (i)}
&\coloneqq 
\frac{1}{n_\ell} \sum_{i=1}^{n_\ell}
\EE\left[
\left\| 
\left\{ 
\widehat{\bs}^{\cK_{\mathrm{CI}},\lambda}_{\widetilde{\mathcal{T}}_{\ell,u}}(\widetilde{\bx}_i)
-
\widehat{\bs}^{\cK_{\mathrm{CI}},\lambda}_{\mathcal{T}_{\ell,u}}(\bx_i)
\right\} \by_i^\top
\right\|_\mathrm{F}
\cdot \mathbb{I}_{\mathcal{G}^c}
\right]  = 
\frac{1}{n_\ell} \sum_{i=1}^{n_\ell}
\EE\left\{
\left\|
\widehat{\bs}^{\cK_{\mathrm{CI}},\lambda}_{\widetilde{\mathcal{T}}_{\ell,u}}(\widetilde{\bx}_i)
-
\widehat{\bs}^{\cK_{\mathrm{CI}},\lambda}_{\mathcal{T}_{\ell,u}}(\bx_i)
\right\|_2 \cdot \|\by_i\|_2 \cdot \mathbb{I}_{\mathcal{G}^c}
\right\}.
\end{align}

Since $\widetilde{\mathcal{T}}_{\ell,u}$ is obtained from $\mathcal{T}_{\ell,u}$ by projecting every sample, we construct a sequence of datasets $\mathcal{T}_{\ell,u}^{(0)},\; \mathcal{T}_{\ell,u}^{(1)},\; \dots,\; \mathcal{T}_{\ell,u}^{(n)}$, where $\mathcal{T}_{\ell,u}^{(0)} = \mathcal{T}_{\ell,u}$, $\mathcal{T}_{\ell,u}^{(n)} = \widetilde{\mathcal{T}}_{\ell,u}$, and for each step $m=1,\dots,n$, the dataset $\mathcal{T}_{\ell,u}^{(m)}$ is obtained from $\mathcal{T}_{\ell,u}^{(m-1)}$ by replacing the $m$-th sample with its projection:
\begin{align*}
(\bx_m,\by_m) &\mapsto (\widetilde{\bx}_m, \widetilde{\by}_m), && m=1,\dots,n_\ell,\\
\bx_m &\mapsto \widetilde{\bx}_m, && m=n_\ell+1,\dots,n.
\end{align*}
Telescoping the differences,
\begin{align*}
\widehat{\bs}^{\cK_{\mathrm{CI}},\lambda}_{\mathcal{T}_{\ell,u}}(\bx_i)
-
\widehat{\bs}^{\cK_{\mathrm{CI}},\lambda}_{\widetilde{\mathcal{T}}_{\ell,u}}(\bx_i)
=
\sum_{m=1}^n \left\{
\widehat{\bs}^{\cK_{\mathrm{CI}},\lambda}_{\mathcal{T}_{\ell,u}^{(m-1)}}(\bx_i)
-
\widehat{\bs}^{\cK_{\mathrm{CI}},\lambda}_{\mathcal{T}_{\ell,u}^{(m)}}(\bx_i)
\right\}.
\end{align*}
By Lemma~\ref{lem:kef-replacement}, each single-sample replacement changes the estimator by at most $C_{\rm rep}$. Hence,
\begin{align*}
\left\|
\widehat{\bs}^{\cK_{\mathrm{CI}},\lambda}_{\mathcal{T}_{\ell,u}}(\bx_i)
-
\widehat{\bs}^{\cK_{\mathrm{CI}},\lambda}_{\widetilde{\mathcal{T}}_{\ell,u}}(\bx_i)
\right\|_2
\le
n \cdot C_{\rm rep}.
\end{align*}
Using this bound in \eqref{eq:app-3a}, we obtain
\begin{align}\label{eq:app-ta}
\text{Term (i)}
\le 
\frac{n \cdot C_{\rm rep}}{n_\ell} \sum_{i=1}^{n_\ell}
\EE\left( \|\by_i\|_2 \cdot \mathbb{I}_{\mathcal{G}^c} \right). 
\end{align}

The union bound gives
\begin{align*}
\mathbb{I}_{\mathcal{G}^c} \le \sum_{k=1}^n \mathbb{I}_{\{\|\bx_k\|_2 > M_{\bx}\}} + \sum_{j=1}^{n_\ell} \mathbb{I}_{\{\|\by_j\|_2 > M_{\by}\}}.
\end{align*}
Substituting the bound above into \eqref{eq:app-ta} and defining
\begin{align*}
\mathcal{E}_{\bx} \coloneqq \frac{n}{n_\ell} \sum_{i=1}^{n_\ell} \sum_{k=1}^n
\EE\left[ \|\by_i\|_2 \cdot \mathbb{I}_{\{\|\bx_k\|_2 > M_{\bx}\}} \right]\text{, and }
\mathcal{E}_{\by} \coloneqq \frac{n}{n_\ell} \sum_{i=1}^{n_\ell} \sum_{j=1}^{n_\ell}
\EE\left[ \|\by_i\|_2 \cdot \mathbb{I}_{\{\|\by_j\|_2 > M_{\by}\}} \right],
\end{align*}
we obtain
\begin{align}\label{eq:term-i-ex}
\text{Term (i)} \le C_{\rm rep} (\mathcal{E}_{\bx} + \mathcal{E}_{\by}). 
\end{align}

Splitting $\mathcal{E}_{\bx}$ into diagonal ($k=i$) and off-diagonal ($k\neq i$) terms,
\begin{align*}
\mathcal{E}_{\bx} 
= 
\underbrace{
\frac{n}{n_\ell} \sum_{i=1}^{n_\ell}
\EE\left[ \|\by_i\|_2 \cdot \mathbb{I}_{\{\|\bx_i\|_2 > M_{\bx}\}} \right]
}_{\text{diagonal }(k=i)}
+
\underbrace{
\frac{n}{n_\ell} \sum_{i=1}^{n_\ell} \sum_{k\neq i}
\EE\left[ \|\by_i\|_2 \cdot \mathbb{I}_{\{\|\bx_k\|_2 > M_{\bx}\}} \right]
}_{\text{off-diagonal }(k\neq i)}.
\end{align*}

\medskip
By Lemma~\ref{lem:app-ge}, $\PP(\|\bx\|_2 > M_{\bx}) \le \rho\delta/(2n^4)$. By Cauchy-Schwarz inequality,
\begin{align*}
\EE\left[ \|\by_i\|_2 \cdot \mathbb{I}_{\{\|\bx_i\|_2 > M_{\bx}\}} \right]
\le
\sqrt{ \EE\|\by_i\|_2^2 } \cdot
\sqrt{ \PP(\|\bx_i\|_2 > M_{\bx}) }
\le
C \sqrt{q} \cdot \sqrt{\frac{\rho\delta}{2n^4}}
= C \sqrt{\frac{q\rho\delta}{2}} \cdot \frac{1}{n^2},
\end{align*}
where we used $\sqrt{\EE\|\by\|_2^2} \le C \sqrt{q}$ for $\by$ whose coordinates are all sub-Gaussian. Thus the diagonal term is bounded by
\begin{align*}
\frac{n}{n_\ell} \sum_{i=1}^{n_\ell} C \sqrt{\frac{q\rho\delta}{2}} \cdot \frac{1}{n^2}
\le
C \sqrt{\frac{q\rho\delta}{2}} \cdot \frac{1}{n}.
\end{align*}

\medskip
When $k\neq i$, independence gives
\begin{align*}
\EE\left[ \|\by_i\|_2 \cdot \mathbb{I}_{\{\|\bx_k\|_2 > M_{\bx}\}} \right]
= \EE\|\by_i\|_2 \cdot \PP(\|\bx_k\|_2 > M_{\bx})
\le
C \sqrt{q} \cdot \frac{\rho\delta}{2n^4}.
\end{align*}
The number of off-diagonal pairs is at most $n_\ell \cdot n$. Hence,
\begin{align*}
\frac{n}{n_\ell} \sum_{i=1}^{n_\ell} \sum_{k\neq i}
\EE\left[ \|\by_i\|_2 \cdot \mathbb{I}_{\{\|\bx_k\|_2 > M_{\bx}\}} \right]
\le
\frac{n}{n_\ell} \cdot n_\ell n \cdot C \sqrt{q} \frac{\rho\delta}{2n^4}
=
C \sqrt{q} \frac{\rho\delta}{2} \cdot \frac{1}{n^2}.
\end{align*}

Combining the diagonal and off-diagonal bounds,
\begin{align*}
\mathcal{E}_{\bx}
\le
C \sqrt{q} \left( \sqrt{\frac{\rho\delta}{2}} \, \frac{1}{n} + \frac{\rho\delta}{2} \, \frac{1}{n^2} \right).
\end{align*}
Thus,
\begin{align} \label{eq:ex-bound}
\mathcal{E}_{\bx} \le C \sqrt{q} \, \sqrt{\rho\delta} \, \frac{1}{n},    
\end{align}
where $C$ is an absolute constant.\\
Similarly,
\begin{align*}
\mathcal{E}_{\by}
=
\underbrace{
\frac{n}{n_\ell} \sum_{i=1}^{n_\ell}
\EE\left[ \|\by_i\|_2 \cdot \mathbb{I}_{\{\|\by_i\|_2 > M_{\by}\}} \right]
}_{\text{diagonal }(k=i)}
+
\underbrace{
\frac{n}{n_\ell} \sum_{i=1}^{n_\ell} \sum_{k\neq i}
\EE\left[ \|\by_i\|_2 \cdot \mathbb{I}_{\{\|\by_k\|_2 > M_{\by}\}} \right]
}_{\text{off-diagonal }(k\neq i)}.
\end{align*}

By Lemma~\ref{lem:app-ge}, $\PP(\|\by\|_2 > M_{\by}) \le \rho\delta/(2n_\ell^4)$. By Cauchy-Schwarz,
\begin{align*}
\EE\left[ \|\by_i\|_2 \cdot \mathbb{I}_{\{\|\by_i\|_2 > M_{\by}\}} \right]
\le
\sqrt{ \EE\|\by_i\|_2^2 } \cdot
\sqrt{ \PP(\|\by_i\|_2 > M_{\by}) }
\le
C \sqrt{q} \cdot \sqrt{\frac{\rho\delta}{2n_\ell^4}}
= C \sqrt{\frac{q\rho\delta}{2}} \cdot \frac{1}{n^2_\ell}.
\end{align*}
Thus,
\begin{align*}
\frac{n}{n_\ell} \sum_{i=1}^{n_\ell} C \sqrt{\frac{q\rho\delta}{2}} \cdot \frac{1}{n^2_\ell}
\le
C \sqrt{\frac{q\rho\delta}{2}} \cdot \frac{n}{n_\ell^2}.
\end{align*}

When $k\neq i$, independence gives
\begin{align*}
\EE\left[ \|\by_i\|_2 \cdot \mathbb{I}_{\{\|\by_k\|_2 > M_{\by}\}} \right]
= \EE\|\by_i\|_2 \cdot \PP(\|\by_k\|_2 > M_{\by})
\le
C \sqrt{q} \cdot \frac{\rho\delta}{2n_\ell^4}.
\end{align*}
The number of off-diagonal pairs is at most $n_\ell \cdot n$. Hence,
\begin{align*}
\frac{n}{n_\ell} \sum_{i=1}^{n_\ell} \sum_{k\neq i}
\EE\left[ \|\by_i\|_2 \cdot \mathbb{I}_{\{\|\by_k\|_2 > M_{\by}\}} \right]
\le
\frac{n}{n_\ell} \cdot n_\ell n \cdot C \sqrt{q} \frac{\rho\delta}{2n_\ell^4}
=
C \sqrt{q} \frac{\rho\delta}{2} \cdot \frac{n^2}{n_\ell^4}.
\end{align*}

Combining the diagonal and off-diagonal bounds,
\begin{align*}
\mathcal{E}_{\by}
\le
C \sqrt{q} \left( \sqrt{\frac{\rho\delta}{2}} \cdot \frac{n}{n_\ell^2}
+ \frac{\rho\delta}{2} \cdot \frac{n^2}{n_\ell^4} \right).
\end{align*}
Thus,
\begin{align}\label{eq:ey-bound}
\mathcal{E}_{\by} \le C \sqrt{q} \, \sqrt{\rho\delta} \left( \frac{n}{n_\ell^2} + \frac{n^2}{n_\ell^4} \right). 
\end{align}

Combining \eqref{eq:term-i-ex}, \eqref{eq:ex-bound}, and \eqref{eq:ey-bound},
\begin{align*}
\text{Term (i)} \le
C_{\rm rep} \cdot C \sqrt{q} \, \sqrt{\rho\delta}
\left( \frac{1}{n} + \frac{n}{n_\ell^2} + \frac{n^2}{n_\ell^4} \right).
\end{align*}

Recall that $C_{\rm rep} = \left\{ 2(1 + \kappa_{\mathrm{CI}}/\sqrt{\lambda})\kappa_{\mathrm{CI}}^2/\lambda + 2 \right\} M^{\mathrm{CI}}_{\bzeta}$. Moreover, by Lemma~\ref{lem:app-bd-general}, $M^{\mathrm{CI}}_{\bzeta} \asymp p$. Therefore,
\begin{align}\label{eq:term-i-final}
\text{Term (i)}
\le
C \left( \frac{\kappa_{\mathrm{CI}}^2}{\lambda} + \frac{\kappa_{\mathrm{CI}}^2}{\lambda^{3/2}} + 1 \right)
p
\sqrt{q} \, \sqrt{\rho\delta}
\left( \frac{1}{n} + \frac{n}{n_\ell^2} + \frac{n^2}{n_\ell^4} \right) 
\end{align}
where $C$ is an absolute constant. 

Define
\begin{align} \label{eq:term-ii-def}
\text{Term (ii)}
\coloneqq 
\frac{1}{n_\ell} \sum_{i=1}^{n_\ell}
\EE\left\{
\left\| 
\widehat{\bs}^{\cK_{\mathrm{CI}},\lambda}_{\mathcal{T}_{\ell,u}}(\bx_i)
(\by_i - \widetilde{\by}_i)^\top
\right\|_\mathrm{F}
\cdot \mathbb{I}_{\mathcal{G}^c}
\right\} 
\le 
\frac{1}{n_\ell} \sum_{i=1}^{n_\ell}
\EE\left\{
\|\widehat{\bs}^{\cK_{\mathrm{CI}},\lambda}_{\mathcal{T}_{\ell,u}}(\bx_i)\|_2
\cdot \|\by_i - \widetilde{\by}_i\|_2
\cdot \mathbb{I}_{\mathcal{G}^c}
\right\}.
\end{align}

Since $\by_i - \widetilde{\by}_i = 0$ whenever $\|\by_i\|_2 \le M_{\by}$, the global bad event indicator $\mathbb{I}_{\mathcal{G}^c}$ can be replaced by the local indicator $\mathbb{I}_{\{\|\by_i\|_2 > M_{\by}\}}$. On the projected data, the score estimator is uniformly bounded by Lemma~\ref{lem:kef-prop}:
\begin{align*}
\|\widehat{\bs}^{\cK_{\mathrm{CI}},\lambda}_{\mathcal{T}_{\ell,u}}(\bx_i)\|_2 \le C_s,
\end{align*}
where
$
C_s = (1 + \kappa_{\mathrm{CI}}) M^{\mathrm{CI}}_{\bzeta}/\lambda$.

Moreover,
\begin{align*}
\|\by_i - \widetilde{\by}_i\|_2 \le \|\by_i\|_2.
\end{align*}
Substituting the bounds above into \eqref{eq:term-ii-def},
\begin{align}
\text{Term (ii)}
&\le 
\frac{C_s}{n_\ell} \sum_{i=1}^{n_\ell}
\EE\left[
\|\by_i\|_2 \cdot \mathbb{I}_{\{\|\by_i\|_2 > M_{\by}\}}
\right]. \label{eq:term-ii-bound}
\end{align}

By Lemma~\ref{lem:app-ge}, $\PP(\|\by\|_2 > M_{\by}) \le \rho\delta / (2 n_\ell^4)$. Applying Cauchy-Schwarz inequality,
\begin{align*}
\EE\left[ \|\by_i\|_2 \cdot \mathbb{I}_{\{\|\by_i\|_2 > M_{\by}\}} \right]
\le
\sqrt{\EE(\|\by_i\|_2^2)} \cdot
\sqrt{\PP(\|\by_i\|_2 > M_{\by})}.
\end{align*}
Since $\by$ is elementwise sub-Gaussian, $\sqrt{\EE\|\by_i\|_2^2} \le C \sqrt{q}$ for some absolute constant $C$. Therefore,
\begin{align*}
\EE\left[ \|\by_i\|_2 \cdot \mathbb{I}_{\{\|\by_i\|_2 > M_{\by}\}} \right]
\le
C \sqrt{q} \cdot \sqrt{\frac{\rho\delta}{2 n_\ell^4}}
=
C \sqrt{\frac{q \rho\delta}{2}} \cdot \frac{1}{n^2_\ell}.
\end{align*}
Substituting this into \eqref{eq:term-ii-bound},
\begin{align*}
\text{Term (ii)}
\le
\frac{C_s}{n_\ell} \sum_{i=1}^{n_\ell}
C \sqrt{\frac{q \rho\delta}{2}} \cdot \frac{1}{n^2_\ell}
=
C_s \cdot C \sqrt{\frac{q \rho\delta}{2}} \cdot \frac{1}{n^2_\ell}.
\end{align*}
Recalling $C_s = (1 + \kappa_{\mathrm{CI}}) M^{\mathrm{CI}}_{\bzeta}/\lambda$, we obtain
\begin{align}\label{eq:term-ii-final}
\text{Term (ii)}
\le
C \frac{(1 + \kappa_{\mathrm{CI}}) M^{\mathrm{CI}}_{\bzeta}}{\lambda}
\sqrt{q \rho\delta} \cdot \frac{1}{n^2_\ell} \le
C \frac{(1 + \kappa_{\mathrm{CI}}) p}{\lambda}
\sqrt{q \rho\delta} \cdot \frac{1}{n^2_\ell}.   
\end{align}
where $C$ is an absolute constant. 
\medskip
Define
\begin{equation}\label{eq:term-iii-def}
\begin{aligned} 
\text{Term (iii)}
\coloneqq
\frac{1}{n_\ell} \sum_{i=1}^{n_\ell}
\EE\left[
\left\| 
\left\{
\widehat{\bs}^{\cK_{\mathrm{CI}},\lambda}_{\mathcal{T}_{\ell,u}}(\bx_i)
-
\widehat{\bs}^{\cK_{\mathrm{CI}},\lambda}_{\mathcal{T}_{\ell,u}}(\widetilde{\bx}_i)
\right\}
\widetilde{\by}_i^\top
\right\|_\mathrm{F}
\cdot \mathbb{I}_{\mathcal{G}^c}
\right]
\le 
\frac{1}{n_\ell} \sum_{i=1}^{n_\ell}
\EE\left\{
\|\widehat{\bs}^{\cK_{\mathrm{CI}},\lambda}_{\mathcal{T}_{\ell,u}}(\bx_i)
-
\widehat{\bs}^{\cK_{\mathrm{CI}},\lambda}_{\mathcal{T}_{\ell,u}}(\widetilde{\bx}_i)\|_2
\cdot \|\widetilde{\by}_i\|_2
\cdot \mathbb{I}_{\mathcal{G}^c}
\right\}.
\end{aligned}    
\end{equation}

By the Lipschitz property of the score estimator (Lemma~\ref{lem:kef-prop}),
\begin{align*}
\|\widehat{\bs}^{\cK_{\mathrm{CI}},\lambda}_{\mathcal{T}_{\ell,u}}(\bx_i)
-
\widehat{\bs}^{\cK_{\mathrm{CI}},\lambda}_{\mathcal{T}_{\ell,u}}(\widetilde{\bx}_i)\|_2
\le L_s \|\bx_i - \widetilde{\bx}_i\|_2,
\end{align*}
where
$L_s = \left(M_3 M^{\mathrm{CI}}_{\bzeta}/\lambda^2 + M_3/\lambda\right) p^{3/2},
$
and $\bx_i - \widetilde{\bx}_i = 0$ whenever $\|\bx_i\|_2 \le M_{\bx}$. Hence, the global bad event indicator can be replaced by the local indicator $\mathbb{I}_{\{\|\bx_i\|_2 > M_{\bx}\}}$. On the projected data,
$
\|\widetilde{\by}_i\|_2 \le M_{\by}, $ and $
\|\bx_i - \widetilde{\bx}_i\|_2 \le \|\bx_i\|_2$.

Substituting these bounds into \eqref{eq:term-iii-def}, we obtain
\begin{align}
\text{Term (iii)}
&\le 
\frac{L_s M_{\by}}{n_\ell} \sum_{i=1}^{n_\ell}
\EE\left[
\|\bx_i\|_2 \cdot \mathbb{I}_{\{\|\bx_i\|_2 > M_{\bx}\}}
\right]. \label{eq:term-iii-bound}
\end{align}

By Lemma~\ref{lem:app-ge}, $\PP(\|\bx\|_2 > M_{\bx}) \le \rho\delta / (2 n^4)$. Applying Cauchy-Schwarz inequality,
\begin{align*}
\EE\left[ \|\bx_i\|_2 \cdot \mathbb{I}_{\{\|\bx_i\|_2 > M_{\bx}\}} \right]
\le
\sqrt{\EE(\|\bx_i\|_2^2)} \cdot
\sqrt{\PP(\|\bx_i\|_2 > M_{\bx})}.
\end{align*}
Since $\bx$ is elementwise sub-Gaussian, $\sqrt{\EE\|\bx_i\|_2^2} \le C \sqrt{p}$ for some absolute constant $C$. Therefore,
\begin{align*}
\EE\left[ \|\bx_i\|_2 \cdot \mathbb{I}_{\{\|\bx_i\|_2 > M_{\bx}\}} \right]
\le
C \sqrt{p} \cdot \sqrt{\frac{\rho\delta}{2 n^4}}
=
C \sqrt{\frac{p \rho\delta}{2}} \cdot \frac{1}{n^2}.
\end{align*}
Substituting this and the definition of $L_s$ into \eqref{eq:term-iii-bound},
\begin{align*}
\text{Term (iii)}
\le
\frac{ \left(\frac{M_3 M^{\mathrm{CI}}_{\bzeta}}{\lambda^2} + \frac{M_3}{\lambda}\right) p^{\frac{3}{2}} M_{\by}}{n_\ell}
\sum_{i=1}^{n_\ell}
C \sqrt{\frac{p \rho\delta}{2}} \cdot \frac{1}{n^2}\le
C \left(\frac{M_3 M_{\bzeta}}{\lambda^2} + \frac{M_3}{\lambda}\right)
p^{\frac{3}{2}} M_{\by}
\sqrt{p \rho\delta} \cdot \frac{1}{n^2}.
\end{align*}
Thus,
\begin{align}\label{eq:term-iii-final}
\text{Term (iii)}
\le
C \left(\frac{M_3 M_{\bzeta}}{\lambda^2} + \frac{M_3}{\lambda}\right)
p^2 M_{\by}
\sqrt{\rho\delta} \cdot \frac{1}{n^2}\le
C \left(\frac{M_3 p}{\lambda^2} + \frac{M_3}{\lambda}\right)
p^2 M_{\by}
\sqrt{\rho\delta} \cdot \frac{1}{n^2}.    
\end{align}
where $C$ is an absolute constant.
Combining the explicit bounds for Term (i), Term (ii), and Term (iii) from \eqref{eq:term-i-final}, \eqref{eq:term-ii-final}, and \eqref{eq:term-iii-final}, define
\begin{align*}
\begin{aligned}
& D\\
\coloneqq &
\left\|\EE\widehat{\bM}^{\cK_{\mathrm{CI}},\lambda}_{\ell,u}\circ \Pi - \EE\widehat{\bM}^{\cK_{\mathrm{CI}},\lambda}_{\ell,u}\right\|_\mathrm{F} \\
\le &
C \left\{
\left( \frac{\kappa_{\mathrm{CI}}^2}{\lambda} + \frac{\kappa_{\mathrm{CI}}^2}{\lambda^{\frac{3}{2}}} + 1 \right)
M^{\mathrm{CI}}_{\bzeta}
\sqrt{q\rho\delta}
\left( \frac{1}{n} + \frac{n}{n_\ell^2} + \frac{n^2}{n_\ell^4} \right)  +
\frac{(1 + \kappa_{\mathrm{CI}}) M^{\mathrm{CI}}_{\bzeta}}{\lambda}
\sqrt{q\rho\delta} \cdot \frac{1}{n^2_\ell}  +
\left( \frac{M_3 M^{\mathrm{CI}}_{\bzeta}}{\lambda^2} + \frac{M_3}{\lambda} \right)
p^2 M_{\by}
\sqrt{\rho\delta} \cdot \frac{1}{n^2}
\right\},
\end{aligned}
\end{align*}
where $C$ is an absolute constant. Thus, up to absolute constants and assuming $n_\ell \le n$,
\begin{equation}\label{eq:app-db}
\begin{aligned}
 & D\\
\le& 
C \left( \frac{\kappa_{\mathrm{CI}}^2}{\lambda} + \frac{\kappa_{\mathrm{CI}}^2}{\lambda^{\frac{3}{2}}} + 1 \right)
M^{\mathrm{CI}}_{\bzeta}
\sqrt{q\rho\delta}
\left( \frac{1}{n} + \frac{n}{n_\ell^2} + \frac{n^2}{n_\ell^4} \right) + C
\left( \frac{M_3 M^{\mathrm{CI}}_{\bzeta}}{\lambda^2} + \frac{M_3}{\lambda} \right)
p^2 M_{\by}
\sqrt{\rho\delta} \cdot \frac{1}{n^2}\\
\le & C \sqrt{\rho\delta} \left\{ \left( \frac{1}{\lambda} + \frac{1}{\lambda^{\frac{3}{2}}} + 1 \right)
\left( \frac{1}{n} + \frac{n}{n_\ell^2} + \frac{n^2}{n_\ell^4} \right) + \left( \frac{1}{\lambda^2} + \frac{1}{\lambda} \right)
\sqrt{\ln(n_{\ell}})
 \cdot \frac{1}{n^2} \right\}.
\end{aligned}
\end{equation}
From the definition of $\epsilon$ in \eqref{eq:eps-def} and the expression for $\Delta$, the McDiarmid threshold satisfies
\begin{align}\label{eq:eps-order}
\epsilon \asymp 
& \left( \frac{1}{\lambda} + \frac{1}{\lambda^2} \right)
\frac{\sqrt{n}}{n_\ell}
\sqrt{\ln\left(\frac{n}{\delta}\right)\ln\left(\frac{n_\ell}{\delta}\right)\ln\left(\frac{1}{\delta}\right)} + \frac{1}{\sqrt{n}}
\sqrt{\ln\left(\frac{n_\ell}{\delta}\right)\ln\left(\frac{1}{\delta}\right)}.
\end{align}

We now combine the variance bound from McDiarmid's inequality and the bias bound from Theorem~B.1 to obtain the final convergence rate.

Let $n_\ell \asymp n^\beta$ with $0 < \beta \le 1$, and let $\lambda = n^{-\alpha}$ with $\alpha > 0$. The four error sources have the following polynomial orders, ignoring logarithmic factors:
\begin{align}
B_1 \;(\text{score bias}):&\quad n^{-\alpha \omega }, \label{eq:order-B1}\\
V_1 \;(\text{score variance}):&\quad n^{\alpha - \frac{1}{2}}, \label{eq:order-V1}\\
V_2 \;(\text{McDiarmid variance}):&\quad n^{\alpha + \frac{1}{2} - \beta} + n^{-\frac{1}{2}}, \label{eq:order-V2}\\
D \;(\text{truncation bias}):&\quad n^{\max\{\alpha + 1 - 2\beta,\; \alpha + 2 - 4\beta,\; 2\alpha - 2\}}. \label{eq:order-D}
\end{align}
Here $B_1$ and $V_1$ are inherited from Theorem~B.1 in \citet{zhou2020} via Cauchy-Schwarz; $V_2$ arises from the McDiarmid concentration of the projected estimator ($\epsilon$ in \eqref{eq:eps-order}); and $D$ is the truncation bias from the three-term decomposition (see \eqref{eq:app-db}).

We minimize the maximum of the four orders over $\alpha > 0$. We consider the following cases depending on $\beta$.

\medskip
\noindent Case 1: $\beta \le 1/2$.\\
When $\beta \le 1/2$, the $D$ term is dominated by $D_2 = \alpha + 2 - 4\beta$ (if $\alpha$ is not too large) or by the $V_2$ term. For any $\alpha > 0$, the error does not decay to zero; the best possible choice is $\alpha = 0$ (i.e., $\lambda = 1$), which yields:
\begin{itemize}
\item If $\beta < 1/2$: the error diverges as $n^{1/2 - \beta}$.
\item If $\beta = 1/2$: the error remains at a constant order $O(1)$.
\end{itemize}
Thus, the semi-supervised estimator is not consistent for $\beta \le 1/2$, as the truncation bias and McDiarmid variance prevent convergence to zero.

\medskip
\noindent Case 2: $1/2 < \beta < 1$. \\
In this regime, $V_2$ is dominated by $n^{\alpha + 1/2 - \beta}$, which is larger than $V_1$ (since $1-\beta > 0$). The truncation bias $D$ is dominated by $\alpha + 1 - 2\beta$, which is smaller than $V_2$ when $\beta > 1/2$. Thus $D$ is absorbed into $V_2$. The remaining dominant terms are $V_2$ and $B_1$. Balancing them gives:
\begin{align*}
\alpha + 1/2 - \beta = -\alpha \omega 
\;\Longrightarrow\;
\alpha = \frac{2\beta - 1}{2(\omega +1)}. 
\end{align*}
The resulting convergence rate is:
\begin{align*}
\|\widehat{\bM}^{\cK_{\mathrm{CI}},\lambda}_{\ell,u} - \bM\|_\mathrm{F} = O\left\{ n^{-\frac{\omega (2\beta - 1)}{2(\omega +1)}} \right\}. 
\end{align*}

\medskip
\noindent Case 3: $\beta = 1$.\\
When $\beta = 1$, $V_2$ becomes $n^{\alpha - 1/2} + n^{-1/2}$, which is of the same order as $V_1$ (the $n^{-1/2}$ term is smaller). The truncation bias $D$ is of order $n^{\alpha - 1}$, which is absorbed into $V_1$. The remaining dominant terms are $V_1$ and $B_1$. Balancing them gives:
\begin{align*}
\alpha - 1/2 = -\alpha r
\;\Longrightarrow\;
\alpha = \frac{1}{2\omega +2}. 
\end{align*}
The resulting convergence rate is:
\begin{align*}
\|\widehat{\bM}^{\cK_{\mathrm{CI}},\lambda}_{\ell,u} - \bM\|_\mathrm{F} = O\left( n^{-\frac{\omega }{2\omega +2}} \right). 
\end{align*}
This matches the classical nonparametric rate of the score estimator in Theorem 4.2 of \citet{zhou2020}.

\medskip
Combining the three cases, the optimal regularization parameter is

\begin{align*}
\lambda^* = n^{-\frac{2\beta - 1}{2(\omega +1)}}\text{, }1/2 < \beta \le 1, 
\end{align*}
and the corresponding convergence rate is
\begin{align*}
\|\widehat{\bM}^{\cK_{\mathrm{CI}},\lambda}_{\ell,u} - \bM\|_\mathrm{F} = O\left\{ n^{-\frac{\omega (2\beta - 1)}{2(\omega +1)}} \right\}\text{, } & 1/2 < \beta \le 1.  
\end{align*}

For $\beta \le 1/2$, the McDiarmid variance term does not vanish as $n \to \infty$, and hence the overall error does not converge to zero in probability.

To obtain a final probability of at least $1-\delta$, we let $\rho$ in the definition of $\cG$ equals $= 1/4$, and  
the bias bound from Theorem~B.1 in \citet{zhou2020} with failure probability $\delta/2$. This yields that for any $\delta \in (0,1)$, there exists a constant $C_\delta > 0$ such that for all sufficiently large $n$, with probability at least $1-\delta$,
\begin{align*}
\left\| \widehat{\bM}^{\cK_{\mathrm{CI}},\lambda^*}_{\ell,u} - \bM \right\|_\mathrm{F}
\le
C_\delta \cdot n^{-\frac{\omega(2\beta - 1)}{2(\omega + 1)}}.
\end{align*}
The constant $C_\delta$ absorbs all constants depending on $\delta$ arising from the thresholds of the concentration inequalities.

\end{proof}

\subsection{Proof of Corollary~\ref{cor:1}}

\begin{proof}
Define the symmetric matrices
\begin{align*}
\widetilde{\bM} \coloneqq \begin{pmatrix} 0 & \bM \\ \bM^\top & 0 \end{pmatrix}
\in \mathbb{R}^{(p+q)\times(p+q)}\text{, and }
\widetilde{\widehat{\bM}} \coloneqq \begin{pmatrix} 0 & \widehat{\bM} \\ \widehat{\bM}^\top & 0 \end{pmatrix}.
\end{align*}
The eigenvalues of $\widetilde{\bM}$ are $\sigma_1(\bM) \ge \cdots \ge \sigma_r(\bM) > 0 = \cdots = 0 > -\sigma_r(\bM) \ge \cdots \ge -\sigma_1(\bM)$,
where $\sigma_1(\bM)\ge\cdots\ge\sigma_r(\bM)$ are the positive singular values of $\bM$.

Let $\widetilde{\bU} \in \mathbb{R}^{(p+q)\times r}$ be the matrix whose columns are the orthonormal eigenvectors of $\widetilde{\bM}$ corresponding to its $r$ largest eigenvalues (i.e. $\sigma_1,\dots,\sigma_r$). 
Similarly, define $\widetilde{\widehat{\bU}}$ from $\widetilde{\widehat{\bM}}$. 
By the block structure, we have
\begin{align*}
\widetilde{\bU} = \frac{1}{\sqrt{2}} \begin{pmatrix} \bU \\ \bV \end{pmatrix},
\quad
\widetilde{\widehat{\bU}}= \frac{1}{\sqrt{2}} \begin{pmatrix} \widehat{\bU} \\ \widehat{\bV} \end{pmatrix},
\end{align*}
where $\bU,\widehat{\bU}\in\mathbb{R}^{p\times r}$ consist of the left leading singular vectors of $\bM$ and $\widehat{\bM}$, respectively, and $\bV,\widehat{\bV}$ consist of the right leading singular vectors. 

Apply Theorem 2 in \cite{yu2015useful} with $r_{\text{theo}}=1$ and $s_{\text{theo}}=r$. 
Since the eigenvalue gap above the first $r$ eigenvalues is infinite, the denominator of the upper bound in that theorem becomes
\begin{align*}
\min\{\lambda_{0}-\lambda_1,\; \lambda_r-\lambda_{r+1}\}
= \min\{\infty,\; \sigma_r(\bM)-0\} = \sigma_r(\bM).
\end{align*}
Also, the numerator is bounded by $2\|\widetilde{\widehat{\bM}}-\widetilde{\bM}\|_\mathrm{F}$. 
Therefore, there exists an orthogonal matrix $\boldsymbol{O}\in\mathbb{R}^{r\times r}$ such that
\begin{align*}
\|\widetilde{\widehat{\bU}}\boldsymbol{O} - \widetilde{\bU}\|_\mathrm{F}
\le \frac{2^{\frac{3}{2}}\,\|\widetilde{\widehat{\bM}}-\widetilde{\bM}\|_\mathrm{F}}{\sigma_r(\bM)}.
\end{align*}
Since $\|\widetilde{\widehat{\bM}}-\widetilde{\bM}\|_\mathrm{F} = \sqrt{2}\,\|\widehat{\bM}-\bM\|_\mathrm{F}$, we obtain
\begin{align*}
\|\widetilde{\widehat{\bU}}\boldsymbol{O} - \widetilde{\bU}\|_\mathrm{F}
\le \frac{2^{2}\,\|\widehat{\bM}-\bM\|_\mathrm{F}}{\sigma_r(\bM)}.
\end{align*}
Using the block structure, we obtain
\begin{align*}
\frac{1}{\sqrt{2}}\left\| \begin{pmatrix} \widehat{\bU} \\ \widehat{\bV} \end{pmatrix}\boldsymbol{O} - \begin{pmatrix} \bU \\ \bV \end{pmatrix} \right\|_\mathrm{F}
\le \frac{4\|\widehat{\bM}-\bM\|_\mathrm{F}}{\sigma_r(\bM)},
\end{align*}
which leads to
\begin{align*}
\|\widehat{\bU}\boldsymbol{O} - \bU\|_\mathrm{F}
\le \sqrt{2}\|\widetilde{\widehat{\bU}}\boldsymbol{O} - \widetilde{\bU}\|_\mathrm{F}
\le \frac{4\sqrt{2}\,\|\widehat{\bM}-\bM\|_\mathrm{F}}{\sigma_r(\bM)}.
\end{align*}
Taking the infimum over all orthogonal matrices $\bO\in\mathcal{V}_r$, we obtain 
\begin{align*}
\inf_{\bO \in \cV_r}\|\widehat{\bU}\boldsymbol{O} - \bU\|_\mathrm{F}
\le \frac{4\sqrt{2}\,\|\widehat{\bM}-\bM\|_\mathrm{F}}{\sigma_r(\bM)}.
\end{align*}

By definition, $\bU$ is an orthogonal basis of the column space of $\bM$, and $\bM_1$ is full-rank, therefore, $\bU$ is an orthogonal basis of $\cS(\bB)$. Since $\bB$ is also an orthogonal basis of $\cS(\bB)$, we obtain 
\begin{align*}
\inf_{\bO \in \cV_r}\|\widehat{\bU}\boldsymbol{O} - \bB\|_\mathrm{F}
\le \frac{4\sqrt{2}\,\|\widehat{\bM}-\bM\|_\mathrm{F}}{\sigma_r(\bM)}.
\end{align*}

Note that $\widehat{\bB} = \widehat{\bU}$, together with identifiability condition, we obtain
\begin{align}\label{eq:app-fb}
\inf_{\bO \in \cV_r}\|\widehat{\bB}\boldsymbol{O} - \bB\|_\mathrm{F}
\le \frac{4\sqrt{2}}{C} \cdot \sqrt{\frac{r}{q}} \|\widehat{\bM}-\bM\|_\mathrm{F}.    
\end{align}
Define $C_0 \coloneqq 4\sqrt{2}/C$, Inequality~\eqref{eq:bb} holds.
For the data splitting method, by Theorem~\ref{thm:subgcomp}, we obtain, with probability at least $1 - \delta$, 
\begin{align*}
&\inf_{\bO\in\mathcal{V}_r}\|\widehat{\bB}_{\ell,u}^{s}\bO - \bB\|_\mathrm{F}\\
\le & C_0 \sqrt{\frac{r}{q}} C K^2 \left\{ \sqrt{q} \, \epsilon(n_u)
+\sqrt{\frac{pq\ln\left(\frac{pq}{\delta}\right)}{n_{\ell}}} + \frac{\sqrt{pq}\ln\left(\frac{pq}{\delta}\right)}{n_{\ell}} \right\} \\ \le & C_{\mathrm{s}} \left\{ \sqrt{r} \, \epsilon(n_u)
+\sqrt{\frac{pr\ln\left(\frac{pq}{\delta}\right)}{n_{\ell}}} + \frac{\sqrt{pr}\ln\left(\frac{pq}{\delta}\right)}{n_{\ell}} \right\},
\end{align*}
where $C_{\mathrm{s}}$ is defined as $C_{\mathrm{s}} \coloneqq C_0 C K^2$.

For the coupled TRE method, by Theorem~\ref{thm:kernel}, we obtain, for the regime $1/2 < \beta \le 1$, for large enough $n$, with probability at least $1 - \delta$,

\begin{align*}
\inf_{\bO\in\mathcal{V}_r}\| \widehat{\bB}^{\mathcal{K}_{\mathrm{CI}},\lambda}_{\ell,u}\bO - \bB\|_\mathrm{F} \le C_0 \sqrt{\frac{r}{q}} C_{\delta} \cdot n^{-\frac{\omega (2\beta - 1)}{2(\omega +1)}} 
\le
C_{\mathrm{c},\delta} \, n^{-\frac{\omega (2\beta - 1)}{2(\omega +1)}},
\end{align*}
where $C_{\mathrm{c},\delta}$ is defined as $C_{\mathrm{c},\delta} \coloneqq C_0 C_{\delta} \sqrt{r/q} $. 
\end{proof}

\subsection{Proof of Theorem~\ref{the:rank-s}}
\begin{proof}
By Theorem 7 in~\citet{o2018random}, 
\begin{align}
  \max_{i \in \min{\{p,q\}}}  |\sigma_{i}(\widehat{\bM}) -  \sigma_{i}(\bM)| \leq \|\widehat{\bM} - \bM \|_{\text{op}} \leq  \|\widehat{\bM} - \bM \|_{F}.
\end{align}

By Inequality~\eqref{eq:dr}, for $\forall \delta \in (0,1)$, $\exists n_{\ell}, n_u$ large enough such that, with probability at least $1 - \delta$, 
\begin{align*}
\|\widehat{\bM} - \bM \|_{F} < \frac{1}{2}C,    
\end{align*}
which leads to
\begin{subequations}
\begin{align*}
\sigma_{r}(\widehat{\bM}) > \sigma_{r}(\bM) - \|\widehat{\bM} - \bM \|_{F} > \frac{1}{2}C,   
\end{align*}    
and 
\begin{align*}
\sigma_{r+1}(\widehat{\bM}) < \sigma_{r + 1}(\bM) + \|\widehat{\bM} - \bM \|_{F} < \frac{1}{2}C.   
\end{align*}    
\end{subequations}

Therefore, for $\tau = C/2$, we have $\widehat{r}_{\tau} = r$. 

\end{proof}

\subsection{Proofs of Lemmas in the Appendix}

\noindent\textbf{Proof of Lemma~\ref{lem:sub-exp}}
\begin{proof}
Recall the definition of sub-Gaussian norm, $\| \cdot \|_{\psi_2} = \inf\{t>0:\EE[\exp\{(\cdot)^2/t^2\}] \leq 2\}$ and sub-exponential norm, $\| \cdot \|_{\psi_1} = \inf\{t>0:\EE\{\exp(|\cdot|/t)\} \leq 2\}$.  

Since $x$ and $y$ are both sub-Gaussian, let 
$\| x \|_{\psi_2} = \inf\{t>0:\EE\exp(x^2/t^2) \leq 2\}< \infty$ and $\| y \|_{\psi_2} = \inf\{t>0:\EE\exp(y^2/t^2) \leq 2\}< \infty$. Then, let $\Tilde{z} = xy$, 
\begin{align*}
    \EE \left \{ \exp\left(\frac{|\Tilde{z}|}{\| x \|_{\psi_2}\| y \|_{\psi_2}}\right) \right \}\leq \EE \left \{ \exp \left ( \frac{x^2}{2\| x \|^2_{\psi_2}} + \frac{y^2}{2\| y \|^2_{\psi_2}} \right ) \right \} \leq \sqrt{\EE \left \{ \exp \left ( \frac{x^2}{2\| x \|^2_{\psi_2}}\right)\right\} \EE\left\{ \exp\left( \frac{y^2}{2\| y \|^2_{\psi_2}} \right ) \right \}} \leq 2.
\end{align*}
Therefore, $z$ is sub-exponential and $\| \Tilde{z} \|_{\psi_1} \leq \| x \|_{\psi_2} \| y \|_{\psi_2}$. 

For any constant $c \in \RR$, $\EE \{ \exp (|c|/t)\} = \exp (|c|/t) \leq 2$ if and only if $t \geq |c|/\ln{2}$. Therefore, $\| c \|_{\psi_1} = |c|/\ln{2}$.

For any sub-Gaussian random variable $x$, 
\begin{align*}
\EE\left(\frac{x^2}{\| x \|^2_{\psi_2}}\right) \leq \EE\left\{ \exp\left( \frac{x^2}{\| x \|^2_{\psi_2}} \right)\right\} \leq 2.    
\end{align*}
Therefore, $\sqrt{\EE(x^2)} \leq \sqrt{2}\| x \|_{\psi_2}$. Then, we obtain
\begin{align*}
    \| z \|_{\psi_1}  \leq  \| \Tilde{z}\|_{\psi_1} + \|\EE(xy) \|_{\psi_1} =  \| \Tilde{z}\|_{\psi_1} + \frac{|\EE(xy)|}{\ln2} \leq \|\Tilde{z}\|_{\psi_1} + \frac{\sqrt{\EE(x^2)\EE(y^2)}}{\ln2} \leq \left(1 + \frac{2}{\ln2}\right)\| x \|_{\psi_2} \| y \|_{\psi_2}.
\end{align*}
Let $C= 1 + 2/\ln2$, the statement holds. 
\end{proof}

\noindent\textbf{Proof of Lemma~\ref{lem:app-ge}}
\begin{proof}
For each $i\in[n]$ and $k\in[p]$,
\begin{align*}
\PP\left( |(\bx_i)_k| > t \right).
\end{align*}
Using $\|\bx_i\|_2 \le \sqrt{p}\|\bx_i\|_\infty$ and the union bound over coordinates,
\begin{align*}
\PP\left( \|\bx_i\|_2 > M_{\bx} \right)
\le \PP\left( \|\bx_i\|_\infty > \frac{M_{\bx}}{\sqrt{p}} \right) \notag 
\le \sum_{k=1}^p \PP\left( |(\bx_{i,k}| > \frac{M_{\bx}}{\sqrt{p}} \right) \notag \le 2p \exp\left( -\frac{c M_{\bx}^2}{K^2 p} \right).
\end{align*}
Substituting $M_{\bx} = K \sqrt{(p/c) \ln\left\{4 n^4 p/\rho\delta)\right\}}$ yields
\begin{align*}
\PP\left( \|\bx_i\|_2 > M_{\bx} \right)
\le 2p \cdot \frac{\rho\delta}{4 n^4 p}
= \frac{\rho\delta}{2 n^4}.
\end{align*}
Applying the union bound over $i=1,\dots,n$,
\begin{align*}
\PP\left( \max_{1\le i\le n} \|\bx_i\|_2 > M_{\bx} \right)
\le \sum_{i=1}^n \PP\left( \|\bx_i\|_2 > M_{\bx} \right)
\le \frac{\rho\delta}{2 n^3}.
\end{align*}
The same argument with $M_{\by}$ gives
\begin{align*}
\PP\left( \max_{1\le j\le n_\ell} \|\by_j\|_2 > M_{\by} \right)
\le \frac{\rho\delta}{2 n_\ell^3}.
\end{align*}
Therefore, by the union bound,
\begin{align*}
\PP(\mathcal{G}^c)
&\le
\frac{\rho\delta}{2 n^3}
+
\frac{\rho\delta}{2 n_\ell^3}
\le
\rho\delta.
\end{align*}
This completes the proof.
\end{proof}

\noindent\textbf{Proof of Lemma~\ref{lem:app-rp}}
\begin{proof}
By the reproducing property of the vector-valued RKHS~\citep{rkhs}, for any $\boldsymbol{c}\in\mathbb{R}^p$,
\begin{align}\label{eq:supp-rp}
\langle \boldsymbol{f}(\bx), \boldsymbol{c} \rangle_{\mathbb{R}^p}
= \langle \boldsymbol{f},\, \mathcal{K}(\cdot,\bx)\boldsymbol{c} \rangle_{\mathcal{H}_{\mathcal{K}}}. 
\end{align}

Taking $\boldsymbol{c} = \boldsymbol{f}(\bx)$ in \eqref{eq:supp-rp} gives
\begin{align*}
\|\boldsymbol{f}(\bx)\|_2^2
= \langle \boldsymbol{f},\, \mathcal{K}(\cdot,\bx)\boldsymbol{f}(\bx) \rangle_{\mathcal{H}_{\mathcal{K}}}. 
\end{align*}
Applying the Cauchy--Schwarz inequality to  the equation above yields
\begin{align}\label{eq:app-k0}
\|\boldsymbol{f}(\bx)\|_2^2
\le \|\boldsymbol{f}\|_{\mathcal{H}_{\mathcal{K}}}
\, \left\| \mathcal{K}(\cdot,\bx)\boldsymbol{f}(\bx) \right\|_{\mathcal{H}_{\mathcal{K}}}. 
\end{align}

By the reproducing property again,
\begin{align}\label{eq:app-k1}
\left\| \mathcal{K}(\cdot,\bx)\boldsymbol{f}(\bx) \right\|_{\mathcal{H}_{\mathcal{K}}}^2
= \left\langle \mathcal{K}(\cdot,\bx)\boldsymbol{f}(\bx),\,
        \mathcal{K}(\cdot,\bx)\boldsymbol{f}(\bx) \right\rangle_{\mathcal{H}_{\mathcal{K}}}
= \left\langle \boldsymbol{f}(\bx),\,
        \mathcal{K}(\bx,\bx)\boldsymbol{f}(\bx) \right\rangle_{\mathbb{R}^p}. 
\end{align}
Since $\|\mathcal{K}(\bx,\bx)\|_{\mathrm{op}}$ is the largest singular value,
\begin{align}\label{eq:app-k2}
\left\langle \boldsymbol{f}(\bx),\, \mathcal{K}(\bx,\bx)\boldsymbol{f}(\bx) \right\rangle_{\mathbb{R}^p}
\le \|\mathcal{K}(\bx,\bx)\|_{\mathrm{op}}
\, \|\boldsymbol{f}(\bx)\|_2^2. 
\end{align}
Combining \eqref{eq:app-k1} and \eqref{eq:app-k2} and taking the square root gives
\begin{align}\label{eq:app-k3}
\left\| \mathcal{K}(\cdot,\bx)\boldsymbol{f}(\bx) \right\|_{\mathcal{H}_{\mathcal{K}}}
\le \sqrt{\|\mathcal{K}(\bx,\bx)\|_{\mathrm{op}}}
\, \|\boldsymbol{f}(\bx)\|_2. 
\end{align}

Substituting \eqref{eq:app-k3} into \eqref{eq:app-k0} yields
\begin{align*}
\|\boldsymbol{f}(\bx)\|_2^2
\le \|\boldsymbol{f}\|_{\mathcal{H}_{\mathcal{K}}}
\, \sqrt{\|\mathcal{K}(\bx,\bx)\|_{\mathrm{op}}}
\, \|\boldsymbol{f}(\bx)\|_2.
\end{align*}
If $\|\boldsymbol{f}(\bx)\|_2 = 0$, the inequality is trivial. Otherwise, dividing by $\|\boldsymbol{f}(\bx)\|_2$ yields the desired result.
\end{proof}

\noindent\textbf{Proof of Lemma~\ref{lem:replacement-stability}}

\begin{proof}
For notational convenience, write $\boldsymbol{f}^* = \bg^{\cK,\lambda}_{\cT_{\ell}}$ and $\boldsymbol{f}^*_j = \bg^{\cK,\lambda}_{\cT^{j}_{\ell}}$. Then, for any $\boldsymbol{f}$ in the RKHS, define the empirical risk on the full dataset and on the replaced dataset as
\begin{subequations}
\begin{align*}
\widehat{\mathcal{R}}(\boldsymbol{f},\cT_{\ell}) \coloneqq \frac{1}{n_{\ell}}\sum_{i=1}^{n_{\ell}} \left \|\boldsymbol{f}(\bx_i)-\by_i \right \|_2^2,   
\end{align*}
and
\begin{align*}
\widehat{\mathcal{R}}(\boldsymbol{f},\cT^{j}_{\ell}) \coloneqq \frac{1}{n_{\ell}}\left\{\sum_{k\ne j} \|\boldsymbol{f}(\bx_k)-\by_k\|_2^2 + \|\boldsymbol{f}(\bx_j')-\by_j'\|_2^2\right\},
\end{align*}   
respectively.
\end{subequations}
The regularized risks are
\begin{align*}
\mathcal{R}_{\mathrm{reg}}(\boldsymbol{f},\cT_{\ell}) = \widehat{\mathcal{R}}(\boldsymbol{f},\cT_{\ell}) + \lambda \|\boldsymbol{f}\|_{\cH_{\cK}}^2,
\text{ and }
\mathcal{R}_{\mathrm{reg}}(\boldsymbol{f},\cT^{j}_{\ell}) = \widehat{\mathcal{R}}(\boldsymbol{f},\cT^{j}_{\ell}) + \lambda \|\boldsymbol{f}\|_{\cH_{\cK}}^2.
\end{align*}

By the definition of $\boldsymbol{f}^*$ ,
\begin{align*}
\lambda \|\boldsymbol{f}^*\|_{\cH_{\cK}}^2 \le \mathcal{R}_{\mathrm{reg}}(\boldsymbol{f}^*,\bX) \le \mathcal{R}_{\mathrm{reg}}(0,\bX) = \frac{1}{n_{\ell}}\sum_{i=1}^{n_{\ell}} \|\by_i\|_2^2 \le M_{\by}^2.
\end{align*}
Thus $\|\boldsymbol{f}^*\|_{\cH_{\cK}} \le M_{\by}/\sqrt{\lambda}$. The same argument applied to $\boldsymbol{f}^*_j$ gives $\|\boldsymbol{f}^*_j\|_{\cH_{\cK}} \le M_{\by}/\sqrt{\lambda}$.

For any $\boldsymbol{f}\in\mathcal{H}_{\cK}$ and any $\bx\in\mathcal{X}$, Lemma~\ref{lem:app-rp} gives
\begin{align*}
\|\boldsymbol{f}(\bx)\|_2 \le \sqrt{\|\cK(\bx,\bx)\|_{\mathrm{op}}} \, \|\boldsymbol{f}\|_{\cH_{\cK}} \le \kappa \|\boldsymbol{f}\|_{\cH_{\cK}},
\end{align*}
where we used the kernel bound $\|\cK(\bx,\bx)\|_{\mathrm{op}} \le \kappa^2$. Hence, for any test point $\bx \in \cX$,
\begin{align*}
\|\boldsymbol{f}^*(\bx)\|_2 \le \frac{\kappa M_{\by}}{\sqrt{\lambda}}\text{, and }
\|\boldsymbol{f}^*_j(\bx)\|_2 \le \frac{\kappa M_{\by}}{\sqrt{\lambda}}.
\end{align*}
Therefore, for any $\by \in \{\by_i\}_{i=1}^{n_{\ell}} \cup \{\by_j^{'} \}$
\begin{align*}
\|\by - \boldsymbol{f}^*(\bx)\|_2 \le M_{\by}\left(1 + \frac{\kappa}{\sqrt{\lambda}}\right)\text{, and }
\|\by - \boldsymbol{f}^{*}_j(\bx)\|_2 \le M_{\by}\left(1 + \frac{\kappa}{\sqrt{\lambda}}\right).
\end{align*}
Using the identity $|a^2-b^2| \le |a-b|(|a|+|b|)$ for $a,b\ge 0$, we obtain
\begin{align}\label{eq:app-lip}
\left| \|\by - \boldsymbol{f}^*(\bx)\|_2^2 - \|\by - \boldsymbol{f}^*_j(\bx)\|_2^2 \right|
\le 2M_{\by}\left(1 + \frac{\kappa}{\sqrt{\lambda}}\right) \|\boldsymbol{f}^*(\bx) - \boldsymbol{f}^*_j(\bx)\|_2. 
\end{align}

For an $\boldsymbol{f}$ in the RKHS, let $c_j(\boldsymbol{f}) = \|\boldsymbol{f}(\bx_j)-\by_j\|_2^2$ and $c_j'(\boldsymbol{f}) = \|\boldsymbol{f}(\bx_j')-\by_j'\|_2^2$. Since $\boldsymbol{f}^*$ and $\boldsymbol{f}^*_j$ are minimizers, for any $t\in(0,1]$,
\begin{align*}
\begin{aligned}
\mathcal{R}_{\mathrm{reg}}(\boldsymbol{f}^*,\cT_{\ell}) \le \mathcal{R}_{\mathrm{reg}}(\boldsymbol{f}^* + t(\boldsymbol{f}^*_j-\boldsymbol{f}^*), \cT_{\ell})\text{, and }
\mathcal{R}_{\mathrm{reg}}(\boldsymbol{f}^*_j,\cT^{j}_{\ell}) \le \mathcal{R}_{\mathrm{reg}}(\boldsymbol{f}^*_j + t(\boldsymbol{f}^*-\boldsymbol{f}^*_j), \cT^{j}_{\ell}).
\end{aligned}
\end{align*}
Adding the two optimality inequalities and expanding the regularized risks gives
\begin{align*}
\begin{aligned}
& \frac{1}{n_{\ell}}\sum_{k=1}^{n_{\ell}} c_k(\boldsymbol{f}^{*}) + \frac{1}{n_{\ell}}\left\{\sum_{k\ne j} c_k(\boldsymbol{f}^{*}_j) + c_j'(\boldsymbol{f}^{*}_j)\right\} + \lambda \|\boldsymbol{f}^{*}\|_{\cH_{\cK}}^2 + \lambda \|\boldsymbol{f}^{*}_j\|_{\cH_{\cK}}^2 \\
& - \frac{1}{n_{\ell}}\sum_{k=1}^{n_{\ell}} c_k(\boldsymbol{f}^{*} + t(\boldsymbol{f}^{*}_j-\boldsymbol{f}^{*})) \\
& - \frac{1}{n_{\ell}}\left\{\sum_{k\ne j} c_k(\boldsymbol{f}^{*}_j + t(\boldsymbol{f}^{*}-\boldsymbol{f}^{*}_j)) + c_j'(\boldsymbol{f}^{*}_j + t(\boldsymbol{f}^{*}-\boldsymbol{f}^{*}_j))\right\} \\
& - \lambda \|\boldsymbol{f}^{*}+t(\boldsymbol{f}^{*}_j-\boldsymbol{f}^{*})\|_{\cH_{\cK}}^2 - \lambda \|\boldsymbol{f}^{*}_j+t(\boldsymbol{f}^{*}-\boldsymbol{f}^{*}_j)\|_{\cH_{\cK}}^2\\
\le & 0.
\end{aligned}
\end{align*}
Separating the $j$-th term, the replacement term, and the remaining terms gives
\begin{equation}\label{eq:supp-ra}
\begin{aligned}
& \frac{1}{n_{\ell}}\left \{ c_j(\boldsymbol{f}^{*}) - c_j(\boldsymbol{f}^{*} + t(\boldsymbol{f}^{*}_j-\boldsymbol{f}^{*})) \right \}
+ \frac{1}{n_{\ell}}\left\{ c_j'(\boldsymbol{f}^{*}_j) - c_j'(\boldsymbol{f}^{*}_j + t(\boldsymbol{f}^{*}-\boldsymbol{f}^{*}_j)) \right \} \\
& + \sum_{k\ne j} \frac{1}{n_{\ell}}\left \{ c_k(\boldsymbol{f}^{*}) - c_k(\boldsymbol{f}^{*} + t(\boldsymbol{f}^{*}_j-\boldsymbol{f}^{*})) + c_k(\boldsymbol{f}^{*}_j) - c_k(\boldsymbol{f}^{*}_j + t(\boldsymbol{f}^{*}-\boldsymbol{f}^{*}_j)) \right \} \\
& + \lambda\left\{ \|\boldsymbol{f}^{*}\|_{\cH_{\cK}}^2 - \|\boldsymbol{f}^{*}+t(\boldsymbol{f}^{*}_j-\boldsymbol{f}^{*})\|_{\cH_{\cK}}^2 + \|\boldsymbol{f}^{*}_j\|_{\cH_{\cK}}^2 - \|\boldsymbol{f}^{*}_j+t(\boldsymbol{f}^{*}-\boldsymbol{f}^{*}_j)\|_{\cH_{\cK}}^2 \right\}\\
\le & 0. 
\end{aligned}
\end{equation}
For each $k\ne j$, since $c_k$ is convex in $\boldsymbol{f}$, the difference quotient is monotone:
\begin{align*}
c_k(\boldsymbol{f}^*) - c_k(\boldsymbol{f}^* + t(\boldsymbol{f}^*_j-\boldsymbol{f}^*)) \ge t\left\{ c_k(\boldsymbol{f}^*) - c_k(\boldsymbol{f}^*_j) \right\},
\end{align*}
and
\begin{align*}
c_k(\boldsymbol{f}^*_j) - c_k(\boldsymbol{f}^*_j + t(\boldsymbol{f}^*-\boldsymbol{f}^*_j)) \ge t\left\{ c_k(\boldsymbol{f}^*_j) - c_k(\boldsymbol{f}^*) \right\}.
\end{align*}
Adding these two inequalities yields
\begin{align*}
c_k(\boldsymbol{f}^*) - c_k(\boldsymbol{f}^* + t(\boldsymbol{f}^*_j-\boldsymbol{f}^*)) + c_k(\boldsymbol{f}^*_j) - c_k(\boldsymbol{f}^*_j + t(\boldsymbol{f}^*-\boldsymbol{f}^*_j)) \ge 0.
\end{align*}
Thus the sum over $k\ne j$ in Inequality~\eqref{eq:supp-ra} is nonnegative and can be omitted, leading to
\begin{align*}
\begin{aligned}
& \frac{1}{n_{\ell}}\left\{ c_j(\boldsymbol{f}^*) - c_j(\boldsymbol{f}^* + t(\boldsymbol{f}^*_j-\boldsymbol{f}^*)) \Big]
+ \frac{1}{n_{\ell}}\Big[ c_j'(\boldsymbol{f}^*_j) - c_j'(\boldsymbol{f}^*_j + t(\boldsymbol{f}^*-\boldsymbol{f}^*_j)) \right\} \\
& + \lambda\left\{ \|\boldsymbol{f}^*\|_{\cH_{\cK}}^2 - \|\boldsymbol{f}^*+t(\boldsymbol{f}^*_j-\boldsymbol{f}^*)\|_{\cH_{\cK}}^2 + \|\boldsymbol{f}^*_j\|_{\cH_{\cK}}^2 - \|\boldsymbol{f}^*_j+t(\boldsymbol{f}^*-\boldsymbol{f}^*_j)\|_{\cH_{\cK}}^2 \right\}\\
\le & 0.
\end{aligned}
\end{align*}

Multiplying by $n_{\ell}$ yields
\begin{align*}
\begin{aligned}
& \left\{ c_j(\boldsymbol{f}^*) - c_j(\boldsymbol{f}^* + t(\boldsymbol{f}^*_j-\boldsymbol{f}^*)) \right\}
+ \left\{[ c_j'(\boldsymbol{f}^*_j) - c_j'(\boldsymbol{f}^*_j + t(\boldsymbol{f}^*-\boldsymbol{f}^*_j)) \right\} \\
& + n_{\ell}\lambda\left\{ \|\boldsymbol{f}^*\|_{\cH_{\cK}}^2 - \|\boldsymbol{f}^*+t(\boldsymbol{f}^*_j-\boldsymbol{f}^*)\|_{\cH_{\cK}}^2 + \|\boldsymbol{f}^*_j\|_{\cH_{\cK}}^2 - \|\boldsymbol{f}^*_j+t(\boldsymbol{f}^*-\boldsymbol{f}^*_j)\|_{\cH_{\cK}}^2 \right\} \\
\le & 0.
\end{aligned}
\end{align*}

Dividing by $t$ and letting $t\to 0^+$, the RKHS norm terms tend to $2n_{\ell}\lambda \|\boldsymbol{f}^*-\boldsymbol{f}^*_j\|_{\cH_{\cK}}^2$. Using convexity for the directional derivatives,
\begin{align*}
\lim_{t\to 0^+} \frac{c_j(\boldsymbol{f}^*) - c_j(\boldsymbol{f}^* + t(\boldsymbol{f}^*_j-\boldsymbol{f}^*))}{t}
\ge c_j(\boldsymbol{f}^*) - c_j(\boldsymbol{f}^*_j),
\end{align*}
\begin{align*}
\lim_{t\to 0^+} \frac{c_j'(\boldsymbol{f}^*_j) - c_j'(\boldsymbol{f}^*_j + t(\boldsymbol{f}^*-\boldsymbol{f}^*_j))}{t}
\ge c_j'(\boldsymbol{f}^*_j) - c_j'(\boldsymbol{f}^*).
\end{align*}
Thus, we obtain
\begin{equation}\label{eq:app-tt}
2n_{\ell}\lambda \|\boldsymbol{f}^*-\boldsymbol{f}^*_j\|_{\cH_{\cK}}^2
\le c_j(\boldsymbol{f}^*_j) - c_j(\boldsymbol{f}^*) + c_j'(\boldsymbol{f}^*) - c_j'(\boldsymbol{f}^*_j).    
\end{equation}

Applying Inequality~\eqref{eq:app-lip} to the two terms on the right-hand side in Inequality~\eqref{eq:app-tt}, we obtain:
\begin{subequations}
\begin{align*}
c_j(\boldsymbol{f}^*_j) - c_j(\boldsymbol{f}^*) \le 2M_{\by}\left(1 + \frac{\kappa}{\sqrt{\lambda}}\right) \|\boldsymbol{f}^*_j(\bx_j)-\boldsymbol{f}^*(\bx_j)\|_2, 
\end{align*}
and 
\begin{align*}
c_j'(\boldsymbol{f}^*) - c_j'(\boldsymbol{f}^*_j) \le 2M_{\by}\left(1 + \frac{\kappa}{\sqrt{\lambda}}\right) \|\boldsymbol{f}^*(\bx_j')-\boldsymbol{f}^*_j(\bx_j')\|_2.
\end{align*}
\end{subequations}
Substituting inequalities above into~\eqref{eq:app-tt}, we obtain
\begin{align*}
2n_{\ell}\lambda \|\boldsymbol{f}^*-\boldsymbol{f}^*_j\|_{\cH_{\cK}}^2
\le M_{\by}\left(1 + \frac{\kappa}{\sqrt{\lambda}}\right) \left\{ \|\boldsymbol{f}^*_j(\bx_j)-\boldsymbol{f}^*(\bx_j)\|_2 + \|\boldsymbol{f}^*(\bx_j')-\boldsymbol{f}^*_j(\bx_j')\|_2 \right\}.
\end{align*}
By Lemma~\ref{lem:app-rp}, $\|\boldsymbol{f}^*(\bx_j)-\boldsymbol{f}^*_j(\bx_j)\|_2 \le \kappa \|\boldsymbol{f}^*-\boldsymbol{f}^*_j\|_{\cH_{\cK}}$. Therefore,
\begin{align*}
2n_{\ell}\lambda \|\boldsymbol{f}^*-\boldsymbol{f}^*_j\|_{\cH_{\cK}}^2 \le 4M_{\by}\left(1 + \frac{\kappa}{\sqrt{\lambda}}\right)\kappa \|\boldsymbol{f}^*-\boldsymbol{f}^*_j\|_{\cH_{\cK}}.
\end{align*}
Hence,
\begin{align*}
\|\boldsymbol{f}^*-\boldsymbol{f}^*_j\|_{\cH_{\cK}} \le 2M_{\by}\left(1 + \frac{\kappa}{\sqrt{\lambda}}\right)\frac{\kappa}{n_{\ell}\lambda}.
\end{align*}

For any test point $\bx \in \cX$,
\begin{align*}
\|\boldsymbol{f}^*(\bx)-\boldsymbol{f}^*_j(\bx)\|_2
\le \kappa \|\boldsymbol{f}^*-\boldsymbol{f}^*_j\|_{\cH_{\cK}}
\le 2M_{\by}\left(1 + \frac{\kappa}{\sqrt{\lambda}}\right) \frac{\kappa^2}{n_{\ell}\lambda} = \frac{CM_{\by}}{n_{\ell}},
\end{align*}
where $C = 2(1 + \kappa/\sqrt{\lambda})\kappa^2/\lambda$.
Replacing $\boldsymbol{f}^* = \bg^{\cK,\lambda}_{\cT_{\ell}}$ and $\boldsymbol{f}^*_j = \bg^{\cK,\lambda}_{\cT^{j}_{\ell}}$ completes the proof.
\end{proof}

\noindent\textbf{Proof of Lemma~\ref{lem:kef-replacement}}
\begin{proof}
By the decomposition of $\widehat{\bs}^{\cK,\lambda}_{\bX}(\bx)$ and $\widehat{\bs}^{\cK,\lambda}_{\bX^{j}}(\bx)$, we obtain 
\begin{align*}
\left \|\widehat{\bs}^{\cK,\lambda}_{\bX}(\bx) - \widehat{\bs}^{\cK,\lambda}_{\bX^{j}}(\bx) \right \|_2
\le \left \|\bg^{\cK,\lambda}_{\cT_{\cK,\bX}}(\bx) - \bg^{\cK,\lambda}_{\cT_{\cK,\bX^{j}}}(\bx) \right \|_2
+ \frac{1}{\lambda}\left\|\widehat{\bzeta}^{\cK}_{\bX}(\bx) - \widehat{\bzeta}^{\cK}_{\bX^{j}}(\bx) \right\|_2.
\end{align*}

Since $\sup_{\bx, \by \in \cX}\|\mathrm{div}_{\bx} \cK(\bx,\by)^\top\|_2 \le M_{\bzeta}$,
by Lemma~\ref{lem:replacement-stability}, the first term is bounded by $2(1 + \kappa/\sqrt{\lambda})\kappa^2 M_{\bzeta}/\lambda n$, and for the second term, 
\begin{align*}
\left \|\widehat{\bzeta}^{\cK}_{\bX}(\bx) - \widehat{\bzeta}^{\cK}_{\bX^{j}}(\bx) \right \|_2
\le \frac{2M_{\bzeta}}{n}.
\end{align*}
Combining the two terms gives the stated bound.
\end{proof}

\noindent\textbf{Proof of Lemma~\ref{lem:app-m3}}
\begin{proof}
Let $\br = \bx-\by$, $r = \|\br\|_2$, recall that 
\begin{align*}
\phi(r) \coloneqq \left(1 + \frac{r^2}{c^2}\right)^{-\frac{1}{2}}.
\end{align*}
Since $\partial_{y_j} = -\partial_{r_j}$, the absolute value of any third-order mixed partial derivative equals $
\left| \partial_{r_i}\partial_{r_k}\partial_{r_j} \phi(r) \right|$.

We show that this quantity is uniformly bounded for all $\br\in\mathbb{R}^p$.

For $r \rightarrow \infty$, $\phi(r) \sim c/r$. Consequently, $\phi'(r)=O(r^{-2})$, $\phi''(r)=O(r^{-3})$, and $\phi'''(r)=O(r^{-4})$. Any third-order partial derivative $\partial_{r_i}\partial_{r_k}\partial_{r_j}\phi(r)$ is a linear combination of terms of the form
\begin{align*}
\frac{\phi'(r)}{r^3},\quad \frac{\phi''(r)}{r^2},\quad \frac{\phi'''(r)}{r},\quad \frac{\phi''(r)}{r^3},
\end{align*}
each of which decays as $O(r^{-5})$ or faster. Hence all such partial derivatives tend to $0$ as $r\to\infty$.

For $r \rightarrow 0$, $\phi(r) = (1 + r^2/c^2)^{-1/2}$ is analytic in $r^2$ around $r=0$. Therefore, write
\begin{align*}
\phi(r)=a_0+a_2r^2+a_4r^4+O(r^6),
\end{align*}
with finite constants $a_0,a_2,a_4$. Then,
\begin{align*}
\phi'(r)=2a_2r+4a_4r^3+O(r^5),\quad
\phi''(r)=2a_2+12a_4r^2+O(r^4)\text{, and }
\phi'''(r)=24a_4r+O(r^3).
\end{align*}
Substituting these into the above formula, the only potentially singular terms are
\begin{align*}
\frac{2a_2}{r^2}\left(\delta_{ik}r_j+\delta_{ij}r_k+\delta_{jk}r_i\right)
\quad\text{and}\quad
-\frac{2a_2}{r^2}\left(\delta_{ik}r_j+\delta_{ij}r_k+\delta_{jk}r_i\right),
\end{align*}
which cancel exactly. The remaining terms are $O(r)$ or bounded by constants independent of $r$. Therefore the third-order partial derivative has a finite limit as $r\to0$.

Since the third-order partial derivative is continuous on $\mathbb{R}^p\setminus\{\boldsymbol{0}\}$, tends to $0$ as $r\to\infty$, and has a finite limit at $r=0$, it is uniformly bounded on $\mathbb{R}^p$. Taking the supremum over all $\bx,\by$ and indices yields $M_3<\infty$.
\end{proof}

\noindent\textbf{Proof of Lemma~\ref{lem:app-bd-general}}

\begin{proof}
Let $\br = \bx-\by$ and $r = \|\br\|_2$. By the Laplacian-Gradient Identity,
\begin{align*}
\mathrm{div}_{\bx} \mathcal{K}(\bx, \by)^\top = - \nabla_{\bx} \Delta \phi(r),
\end{align*}
we obtain
\begin{align*}
\left\| \mathrm{div}_{\bx} \mathcal{K}_{\mathrm{CI}}(\bx, \by)^\top \right\|_2
= \left\| \nabla_{\bx} \Delta \phi(r) \right\|_2
= \left| \frac{d}{dr}\Delta \phi(r) \right|.
\end{align*}
Hence,
\begin{align*}
M^{\mathrm{CI}}_{\boldsymbol{\zeta}} = \sup_{r\ge 0} \left| \frac{d}{dr}\Delta \phi(r) \right|.
\end{align*}

For the IMQ kernel, direct differentiation yields
\begin{align*}
\phi'(r) &= -\frac{r}{c^2}\left(1+\frac{r^2}{c^2}\right)^{-\frac{3}{2}}, \\
\phi''(r) &= -\frac{1}{c^2}\left(1+\frac{r^2}{c^2}\right)^{-\frac{3}{2}}
+ \frac{3r^2}{c^4}\left(1+\frac{r^2}{c^2}\right)^{-\frac{5}{2}}.
\end{align*}

Using the radial Laplacian formula $\Delta \phi(r) = \phi''(r) + \{(p-1)/r\}\phi'(r)$, and letting $u = r/c \ge 0$, we obtain
\begin{align}\label{eq:app-lap-fixed}
\Delta \phi(r) = \frac{1}{c^2}\frac{(3-p)u^2 - p}{(1+u^2)^{\frac{5}{2}}}.
\end{align}

Differentiating \eqref{eq:app-lap-fixed} with respect to $r$ gives
\begin{align*}
\frac{d}{dr}\Delta \phi(r)
= \frac{3}{c^3}\frac{u\left\{(p+2) + (p-3)u^2\right\}}{(1+u^2)^{\frac{7}{2}}}.
\end{align*}

For $p \ge 3$, the expression is non-negative for all $u\ge 0$, so
\begin{align*}
\left| \frac{d}{dr}\Delta \phi(r) \right|
= \frac{3}{c^3} \frac{u\left\{(p+2) + (p-3)u^2\right\}}{(1+u^2)^{\frac{7}{2}}}.
\end{align*}

Define the auxiliary function
\begin{align*}
F_p(u) \coloneqq \frac{u\left\{(p+2) + (p-3)u^2\right\}}{(1+u^2)^{\frac{7}{2}}}, \qquad u\ge 0.
\end{align*}
Then
\begin{align}\label{eq:app-mz}
M^{\mathrm{CI}}_{\boldsymbol{\zeta}} = \frac{3}{c^3} \sup_{u\ge 0} F_p(u).    
\end{align}

We now show that $\sup_{u\ge 0} F_p(u) \asymp p$.

Since $p+2 \le 3p$ and $p-3 \le p$ for $p\ge 3$, we have
\begin{align*}
F_p(u) \le p \cdot \frac{u(3+u^2)}{(1+u^2)^{\frac{7}{2}}}.
\end{align*}
The function $g(u) =u(3+u^2)/(1+u^2)^{7/2}$ is continuous on $[0,\infty)$ and decays as $O(u^{-4})$ as $u\to\infty$; hence it is bounded. Thus there exists an absolute constant $C_1>0$ such that
\begin{align*}
\sup_{u\ge 0} F_p(u) \le C_1 p.
\end{align*}

Letting $u=1$, we get
\begin{align*}
F_p(1) = \frac{1\cdot\{(p+2) + (p-3)\}}{(1+1)^{\frac{7}{2}}}
= \frac{2p-1}{2^{\frac{7}{2}}} \ge \frac{p}{2^{\frac{7}{2}}} \quad (p\ge 1).
\end{align*}
Hence there exists an absolute constant $C_2 = 2^{-7/2} > 0$ such that
\begin{align*}
\sup_{u\ge 0} F_p(u) \ge C_2 p.
\end{align*}

Combining the upper and lower bounds above yields $\sup_{u\ge 0} F_p(u) \asymp p$. 
By direct calculation, it can be shown that 
\begin{align*}
\left|\frac{d}{dr}\Delta \phi(r)\right|
\le \frac{3}{c^3}
\begin{cases}
\dfrac{\sqrt{x_1}\,(3-2x_1)}{(1+x_1)^{\frac{7}{2}}}, & p=1,\quad x_1=\dfrac{6-\sqrt{30}}{4},\\[1.8em]
\dfrac{\sqrt{x_2}\,(4-x_2)}{(1+x_2)^{\frac{7}{2}}}, & p=2,\quad x_2=\dfrac{27-\sqrt{665}}{8}.
\end{cases}  
\end{align*}
Combining these bounds and using the fact that $c$ is an absolute constant, we obtain
\begin{align*}
M^{\mathrm{CI}}_{\boldsymbol{\zeta}} \asymp \frac{3}{c^3} \cdot p \asymp p.
\end{align*}

\end{proof}

\noindent\textbf{Proof of Lemma~\ref{lem:app-opb}}
\begin{proof}
Let $\br = \bx - \by$ and $r = \|\br\|_2$. For any radial $\phi$, the curl-free kernel admits the explicit representation (see, Equation (19) in~\citet{zhou2020})
\begin{align}\label{eq:app-cf}
\mathcal{K}_{\mathrm{cf}}(\bx, \by)
= \left\{\frac{\phi'(r)}{r^3} - \frac{\phi''(r)}{r^2}\right\} \br\br^{\top}
- \frac{\phi'(r)}{r}\,\bI_p. 
\end{align}

For the IMQ kernel, direct differentiation gives
\begin{align*}
\phi'(r) = -\frac{r}{c^2}\left(1+\frac{r^2}{c^2}\right)^{-\frac{3}{2}}, \qquad
\phi''(r) = -\frac{1}{c^2}\left(1+\frac{r^2}{c^2}\right)^{-\frac{3}{2}}
+ \frac{3r^2}{c^4}\left(1+\frac{r^2}{c^2}\right)^{-\frac{5}{2}}.
\end{align*}
Substituting these into~\eqref{eq:app-cf} and simplifying yields the exact matrix form
\begin{align}\label{eq:app-kfv}
\mathcal{K}_{\mathrm{CI}}(\bx, \by)
= -\frac{3}{c^4} \left(1+\frac{r^2}{c^2}\right)^{-\frac{5}{2}} \br\br^{\top}
+ \frac{1}{c^2} \left(1+\frac{r^2}{c^2}\right)^{-\frac{3}{2}} \bI_P. 
\end{align}

Now take the limit $\by \to \bx$, i.e. $r \to 0$. In \eqref{eq:app-kfv}, the first term vanishes because $\|\br\br^{\top}\|_{\mathrm{op}} = r^2 \to 0$ while the scalar prefactor remains bounded. The second term converges to $(1/c^2)\bI_p$. Hence, by continuity of the kernel,
\begin{align*}
\mathcal{K}_{\mathrm{CI}}(\bx, \bx) = \lim_{\by\to\bx} \mathcal{K}_{\mathrm{CI}}(\bx, \by)
= \frac{1}{c^2}\,\bI_p, \qquad \forall \bx\in\mathbb{R}^p. 
\end{align*}

Therefore,
\begin{align*}
\|\mathcal{K}_{\mathrm{CI}}(\bx, \bx)\|_{\mathrm{op}}
= \left\|\frac{1}{c^2}\bI_p\right\|_{\mathrm{op}}
= \frac{1}{c^2},
\end{align*}
which is independent of $\bx$. Taking the supremum over $\bx\in\cX$ proves the claim.
\end{proof}

\noindent\textbf{Proof of Lemma~\ref{lem:kef-prop}}
\begin{proof}
By Theorem 3.1 in \citet{zhou2020}, the score estimator can be decomposed as 
\begin{align*}
\widehat{\bs}^{\cK_{\mathrm{CI}},\lambda}_{\bX}(\bx) = \bg^{\cK_{\mathrm{CI}},\lambda}_{\cT_{\cK_{\mathrm{CI}},\bX}}(\bx) - \frac{1}{\lambda}\widehat{\bzeta}^{\cK_{\mathrm{CI}}}_{\bX}(\bx),
\end{align*}
where $\bg^{\cK_{\mathrm{CI}},\lambda}_{\cT_{\cK_{\mathrm{CI}},\bX}}$ is the standard vector-valued KRR estimator trained on the training set $[\{\bx_i,(1/\lambda) \widehat{\bzeta}^{\cK_{\mathrm{CI}}}_{\bX}(\bx_i)\}]_{i=1}^n$:
\begin{align*}
\bg^{\cK_{\mathrm{CI}},\lambda}_{\cT_{\cK_{\mathrm{CI}},\bX}}
= \arg\min_{\boldsymbol{f}\in\cH_{\cK}} \frac{1}{n}\sum_{i=1}^n \left \|\boldsymbol{f}(\bx_i)-\frac{1}{\lambda} \widehat{\bzeta}^{\cK_{\mathrm{CI}}}_{\bX}(\bx_i) \right \|_2^2 + \lambda \|\boldsymbol{f}\|_{\cH_{\cK}}^2.
\end{align*}

Since the zero function $\boldsymbol{f}=\boldsymbol{0}$ is a feasible solution to the above minimization problem, the optimal solution $\bg^{\cK_{\mathrm{CI}},\lambda}_{\cT_{\cK_{\mathrm{CI}},\bX}}$ must have an objective value no larger than that of the zero function. Therefore,
\begin{align*}
\lambda \left \|\bg^{\cK_{\mathrm{CI}},\lambda}_{\cT_{\cK_{\mathrm{CI}},\bX}} \right \|_{\cH_{\cK}}^2
\le \frac{1}{n\lambda}\sum_{i=1}^n \left \|\widehat{\bzeta}^{\cK_{\mathrm{CI}}}_{\bX}(\bx_i) \right \|_2^2.
\end{align*}
By the definition of $M^{\mathrm{CI}}_{\bzeta}$ established in Lemma~\ref{lem:app-bd-general}, we obtain $\|\widehat{\bzeta}^{\cK_{\mathrm{CI}}}_{\bX}(\bx_i)\|_2 \le M^{\mathrm{CI}}_{\bzeta}$ for every $i \in [n]$. Therefore,
\begin{align*}
\left \|\bg^{\cK_{\mathrm{CI}},\lambda}_{\cT_{\cK_{\mathrm{CI}},\bX}} \right \|_{\cH_{\cK}} \le \frac{M^{\mathrm{CI}}_{\bzeta}}{\lambda}.
\end{align*}

Next, we convert the RKHS norm bound into a pointwise Euclidean norm bound. By the reproducing property of the vector-valued RKHS $\cH_{\cK}$~\citep{rkhs},
\begin{align*}
\bg^{\cK_{\mathrm{CI}},\lambda}_{\cT_{\cK_{\mathrm{CI}},\bX}}(\bx) = \left \langle \bg^{\cK_{\mathrm{CI}},\lambda}_{\cT_{\cK_{\mathrm{CI}},\bX}},\, \cK_{\mathrm{CI}}(\bx, \cdot) \right \rangle_{\cH_{\cK}}.
\end{align*}
Applying the Cauchy-Schwarz inequality in $\cH_{\cK}$,
\begin{align*}
\left \|\bg^{\cK_{\mathrm{CI}},\lambda}_{\cT_{\cK_{\mathrm{CI}},\bX}}(\bx) \right \|_2
\le \left\|\bg^{\cK_{\mathrm{CI}},\lambda}_{\cT_{\cK_{\mathrm{CI}},\bX}} \right \|_{\cH_{\cK}} \cdot \left \|\cK_{\mathrm{CI}}(\bx, \cdot)\right \|_{\cH_{\cK}}.
\end{align*}
By the definition of the vector-valued RKHS norm and Lemma~\ref{lem:app-opb},
\begin{align*}
\|\cK_{\mathrm{CI}}(\bx, \cdot)\|_{\cH_{\cK}}^2
= \|\cK_{\mathrm{CI}}(\bx, \bx)\|_{\mathrm{op}} \leq \kappa_{\mathrm{CI}}.
\end{align*}

Hence,
\begin{align*}
\left \|\bg^{\cK_{\mathrm{CI}},\lambda}_{\cT_{\cK_{\mathrm{CI}},\bX}}(\bx) \right \|_2 \le \kappa_{\mathrm{CI}} \left \|\bg^{\cK_{\mathrm{CI}},\lambda}_{\cT_{\cK_{\mathrm{CI}},\bX}} \right \|_{\cH_{\cK}}
\le \frac{\kappa_{\mathrm{CI}} M^{\mathrm{CI}}_{\bzeta}}{\lambda}.
\end{align*}

Now we bound the correction term $- \frac{1}{\lambda}\widehat{\bzeta}^{\cK_{\mathrm{CI}}}_{\bX}(\bx)$.
\begin{align*}
\left\| \frac{1}{\lambda}\widehat{\bzeta}^{\cK_{\mathrm{CI}}}_{\bX}(\bx) \right\|_2
= \frac{1}{\lambda} \left\| \frac{1}{n}\sum_{i=1}^n \mathrm{div}_{\bx}\cK_{\mathrm{CI}}(\bx,\bx_i)^\top \right\|_2
\le \frac{1}{\lambda} \cdot \frac{1}{n} \sum_{i=1}^n M^{\mathrm{CI}}_{\bzeta}
= \frac{M^{\mathrm{CI}}_{\bzeta}}{\lambda}.
\end{align*}

Finally, by the triangle inequality,
\begin{align*}
\left \|\widehat{\bs}^{\cK_{\mathrm{CI}},\lambda}_{\bX}(\bx) \right \|_2
\le \left \|\bg^{\cK_{\mathrm{CI}},\lambda}_{\cT_{\cK_{\mathrm{CI}},\bX}}(\bx) \right \|_2 + \frac{1}{\lambda}\left\| \widehat{\bzeta}^{\cK_{\mathrm{CI}}}_{\bX}(\bx) \right\|_2
\le \frac{\kappa_{\mathrm{CI}} M^{\mathrm{CI}}_{\bzeta}}{\lambda} + \frac{M^{\mathrm{CI}}_{\bzeta}}{\lambda} = \left( 1 + \kappa_{\mathrm{CI}}\right) \frac{M^{\mathrm{CI}}_{\bzeta}}{\lambda} .
\end{align*}

By the representation theorem, $\bg^{\cK_{\mathrm{CI}},\lambda}_{\cT_{\cK_{\mathrm{CI}},\bX}}$ has the form
\begin{align*}
\bg^{\cK_{\mathrm{CI}},\lambda}_{\cT_{\cK_{\mathrm{CI}},\bX}}(\bx) = \sum_{i=1}^n \cK_{\mathrm{CI}}(\bx,\bx_i)\boldsymbol{c}_i,
\end{align*}
where $\boldsymbol{c} = (\boldsymbol{c}_1,\dots,\boldsymbol{c}_n)^\top \in \mathbb{R}^{np}$ satisfies the normal equations
\begin{align*}
(\boldsymbol{K} + n\lambda \bI)\boldsymbol{c} = \frac{1}{\lambda}\boldsymbol{h},
\end{align*}
where $\bK \in \RR^{np \times np}$ is given by $\bK_{(m-1)p+i,(k-1)p+j} = [\cK_{\mathrm{CI}}(\bx_{m},\bx_{k})]_{i,j}$ and $\boldsymbol{h} = \{\widehat{\bzeta}^{\cK_{\mathrm{CI}}}_{\bX}(\bx_1)\T,\dots,\widehat{\bzeta}^{\cK_{\mathrm{CI}}}_{\bX}(\bx_n)\}^\top$.

By the definition of $M^{\mathrm{CI}}_{\bzeta}$, for each $i \in [n]$,
\begin{align*}
\|\widehat{\bzeta}^{\cK_{\mathrm{CI}}}_{\bX}(\bx_i)\|_2
\le \frac{1}{n}\sum_{m=1}^n \|\mathrm{div}_{\bx_i}\cK_{\mathrm{CI}}(\bx_i,\bx_m)^\top\|_2
\le M^{\mathrm{CI}}_{\bzeta}.
\end{align*}
Thus,
\begin{align*}
\|\boldsymbol{h}\|_2 = \sqrt{\sum_{i=1}^{n}\|\widehat{\bzeta}^{\cK_{\mathrm{CI}}}_{\bX}(\bx_i)\|^2_2}\le \sqrt{n} M^{\mathrm{CI}}_{\bzeta}.
\end{align*}
 
Therefore,
\begin{align}\label{eq:app-c}
\|\boldsymbol{c}\|_2 \le \frac{1}{\lambda}\|(\boldsymbol{K} + n\lambda \bI)^{-1}\|_{\mathrm{op}} \|\boldsymbol{h}\|_2
\le \frac{1}{\lambda} \cdot \frac{1}{n\lambda} \cdot \sqrt{n} M^{\mathrm{CI}}_{\bzeta}
= \frac{M^{\mathrm{CI}}_{\bzeta}}{\sqrt{n}\lambda^2}.
\end{align}

We now derive the Lipschitz constants for the kernel and its divergence.

For each $(i,j)$, the component $[\cK_{\mathrm{CI}}]_{ij}(\bx,\by) = -\partial_{x_i}\partial_{y_j} k_{\mathrm{CI}}(\bx,\by)$ has gradient with respect to $\bx$:
\begin{align*}
\nabla_{\bx} [\cK_{\mathrm{CI}}](\bx,\by)
= \left\{
-\frac{\partial^3}{\partial x_1 \partial x_i \partial y_j} k_{\mathrm{CI}}(\bx,\by),
\dots,
-\frac{\partial^3}{\partial x_p \partial x_i \partial y_j} k_{\mathrm{CI}}(\bx,\by)
\right\}^\top.
\end{align*}
By the definition of $M_3$, each component of this gradient is bounded by $M_3$, hence
\begin{align*}
\|\nabla_{\bx} [\cK_{\mathrm{CI}}]_{ij}(\bx,\by)\|_2 \le \sqrt{p}\, M_3.
\end{align*}
By the mean value theorem, for any $\bx,\bx'$,
\begin{align*}
|[\cK_{\mathrm{CI}}]_{ij}(\bx,\by) - [\cK_{\mathrm{CI}}]_{ij}(\bx',\by)|
\le \|\nabla_{\bx} [\cK_{\mathrm{CI}}]_{ij}(\xi,\by)\|_2 \|\bx-\bx'\|_2
\le \sqrt{p}\, M_3 \|\bx-\bx'\|_2,
\end{align*}
where $\xi$ lies on the segment between $\bx$ and $\bx'$. Consequently,
\begin{align}\label{eq:app-k}
\|\cK_{\mathrm{CI}}(\bx,\by) - \cK_{\mathrm{CI}}(\bx',\by)\|_{\mathrm{op}}
\le \|\cK_{\mathrm{CI}}(\bx,\by) - \cK_{\mathrm{CI}}(\bx',\by)\|_\mathrm{F} \le \sqrt{p^2 \cdot (\sqrt{p}\, M_3)^2} \, \|\bx-\bx'\|_2 = p^{\frac{3}{2}} M_3 \|\bx-\bx'\|_2.
\end{align}

Combining Inequalities~\eqref{eq:app-c} and~\eqref{eq:app-k}, we obtain
\begin{align}\label{eq:app-g}
&\left\| \bg^{\cK_{\mathrm{CI}},\lambda}_{\cT_{\cK_{\mathrm{CI}},\bX}}(\bx) - \bg^{\cK_{\mathrm{CI}},\lambda}_{\cT_{\cK_{\mathrm{CI}},\bX}}(\bx') \right\|_2 \nonumber\\
= & \left\| \sum_{i=1}^n \left\{ \cK_{\mathrm{CI}}(\bx,\bx_i) - \cK_{\mathrm{CI}}(\bx',\bx_i) \right\}\boldsymbol{c}_i \right\|_2\nonumber \\
\le & \sum_{i=1}^n \| \cK_{\mathrm{CI}}(\bx,\bx_i) - \cK_{\mathrm{CI}}(\bx',\bx_i) \|_{\mathrm{op}} \| \boldsymbol{c}_i \|_2\nonumber \\
\le & \sqrt{ \sum_{i=1}^n \| \cK_{\mathrm{CI}}(\bx,\bx_i) - \cK_{\mathrm{CI}}(\bx',\bx_i) \|_{\mathrm{op}}^2 } \cdot \sqrt{ \sum_{i=1}^n \| \boldsymbol{c}_i \|_2^2 } \\
\le & \sqrt{ \sum_{i=1}^n \left( p^{\frac{3}{2}} M_3 \| \bx - \bx' \|_2 \right)^2 } \cdot \| \boldsymbol{c} \|_2 \nonumber\\
= & \sqrt{n} \, p^{\frac{3}{2}} M_3 \| \bx - \bx' \|_2 \| \boldsymbol{c} \|_2 \nonumber\\
\le & \sqrt{n} \, p^{\frac{3}{2}} M_3 \| \bx - \bx' \|_2 \cdot \frac{M^{\mathrm{CI}}_{\bzeta}}{\sqrt{n}\lambda^2}\nonumber \\
= & \frac{p^{\frac{3}{2}} M_3 M^{\mathrm{CI}}_{\bzeta}}{\lambda^2} \| \bx - \bx' \|_2.\nonumber
\end{align}

The divergence $\mathrm{div}_{\bx}\cK_{\mathrm{CI}}(\bx,\by)^\top$ is a vector-valued function whose components are linear combinations of third-order derivatives of $k_{\mathrm{CI}}$. More precisely, its $i$-th component is
\begin{align*}
[\mathrm{div}_{\bx}\cK_{\mathrm{CI}}(\bx,\by)^\top]_i = \sum_{k=1}^p \frac{\partial}{\partial x_k} [\cK_{\mathrm{CI}}]_{k,i}(\bx,\by)
= -\sum_{k=1}^p \frac{\partial^3}{\partial x_k \partial x_k \partial y_i} k_{\mathrm{CI}}(\bx,\by).
\end{align*}
Each term is bounded by $M_3$, so
\begin{align*}
\| \mathrm{div}_{\bx}\cK_{\mathrm{CI}}(\bx,\by)^\top - \mathrm{div}_{\bx}\cK_{\mathrm{CI}}(\bx',\by)^\top \|_2
\le \sqrt{p} \cdot p M_3 \|\bx-\bx'\|_2 = p^{\frac{3}{2}} M_3 \|\bx-\bx'\|_2.
\end{align*}
Therefore,
\begin{align}\label{eq:app-cr}
\left\| \frac{1}{\lambda}\widehat{\bzeta}^{\cK_{\mathrm{CI}}}_{\bX}(\bx) - \frac{1}{\lambda} \widehat{\bzeta}^{\cK_{\mathrm{CI}}}_{\bX}(\bx')\right \|_2
\le \frac{1}{n\lambda}\sum_{i=1}^{n} \|\mathrm{div}_{\bx}\cK_{\mathrm{CI}}(\bx,\bx_i)^\top - \mathrm{div}_{\bx}\cK_{\mathrm{CI}}(\bx',\bx_i)^\top\|_2 \le \frac{ p^{\frac{3}{2}} M_3}{\lambda} \|\bx-\bx'\|_2.
\end{align}

Combining Inequalities~\eqref{eq:app-g} and~\eqref{eq:app-cr}, we obtain 
\begin{align*}
& \left \|\widehat{\bs}^{\cK_{\mathrm{CI}},\lambda}_{\bX}(\bx) - \widehat{\bs}^{\cK_{\mathrm{CI}},\lambda}_{\bX}(\bx') \right \|_2\\
\le & \left\| \bg^{\cK_{\mathrm{CI}},\lambda}_{\cT_{\cK_{\mathrm{CI}},\bX}}(\bx) - \bg^{\cK_{\mathrm{CI}},\lambda}_{\cT_{\cK_{\mathrm{CI}},\bX}}(\bx') \right\|_2 + \left\| \frac{1}{\lambda}\widehat{\bzeta}^{\cK_{\mathrm{CI}}}_{\bX}(\bx) - \frac{1}{\lambda} \widehat{\bzeta}^{\cK_{\mathrm{CI}}}_{\bX}(\bx')\right \|_2 \\ \le & \frac{p^{\frac{3}{2}} M_3 M^{\mathrm{CI}}_{\bzeta}}{\lambda^2} \| \bx - \bx' \|_2 +  \frac{ p^{\frac{3}{2}} M_3}{\lambda} \|\bx-\bx'\|_2 \\
\le & \left(\frac{M_3 M^{\mathrm{CI}}_{\bzeta}}{\lambda^2} + \frac{M_3}{\lambda}\right) p^{\frac{3}{2}} \|\bx-\bx'\|_2.
\end{align*}

This completes the proof.
\end{proof}

\section{ Implementation Details of Neural and Kernel-Based Estimators}

This section provides a comprehensive summary of the training hyperparameters for each of the three data-driven estimation methods (DNN score estimator, TRE, and MLP baseline). All random seeds are fixed (42 + repetition index) for reproducibility, and hyperparameters are kept constant across all experiments.

\subsection{DNN Score Estimator (for $\mathrm{S}_{\mathrm{NN}}^{\mathrm{sup}}$, $\mathrm{S}_{\mathrm{NN}}^{\mathrm{full}}$, and $\mathrm{S}_{\mathrm{NN}}^{\mathrm{split}}$)}

\subsubsection{Architecture}\label{sec:app-se-str}
The enoising score matching network consists of:
\begin{itemize}
\item Input projection: Linear layer $p \to 256$, followed by ReLU and LayerNorm.
\item Residual blocks: Two identical residual blocks, each containing:
  \begin{itemize}
  \item Linear layer $256 \to 256$, ReLU, Dropout (rate 0.3),
  \item Linear layer $256 \to 256$,
  \item LayerNorm and a skip connection (adding the input of the block), followed by ReLU.
  \end{itemize}
\item Output projection: Linear layer $256 \to p$ (no activation).
\end{itemize}

\begin{table}[htbp]
\centering
\caption{Training parameters for the denoising score matching network.}
\label{tab:dnn_score}
\begin{tabular}{l c l}
\toprule
Parameter & Value & Remarks \\
\midrule
Optimizer & AdamW & default betas (0.9,0.999) \\
Learning rate & $1\times10^{-3}$ & -- \\
Weight decay & $1\times10^{-5}$ & -- \\
Batch size & 256 & -- \\
Epochs & 150 & fixed, no early stopping \\
LR scheduler & Cosine annealing & period = 150 epochs \\
Loss & DSM & $\EE[\|s_\theta(x+\epsilon)+\epsilon/\sigma^2\|^2]$ \\
Noise std $\sigma$ & 0.1 & -- \\
Dropout & 0.3 & applied in residual blocks \\
Hidden dimension & 256 & -- \\
Input standardization & Yes & via \texttt{StandardScaler} \\
\bottomrule
\end{tabular}
\end{table}

\subsection{TRE (for $\mathrm{S}_{\mathrm{TRE}}^{\mathrm{sup}}$, $\mathrm{S}_{\mathrm{TRE}}^{\mathrm{full}}$ and $\mathrm{S}_{\mathrm{TRE}}^{\mathrm{split}}$)}

\begin{table}[htbp]
\centering
\caption{Parameters for the TRE.}
\label{tab:numethod}
\begin{tabular}{l c l}
\toprule
Parameter & Value & Remarks \\
\midrule
Kernel & curl-free IMQ & curl-free inverse multiquadric \\
Regularization parameter $\lambda$ & 1e-4/1e-5 & taken to be 1e-4 except $n_{\ell} = 500$ in the simulation study  \\
Data standardization & Yes & performed before fitting; gradients rescaled back \\
Random seed & Set & via \texttt{np.random.seed()} for reproducibility \\
\bottomrule
\end{tabular}
\end{table}

\subsection{MLP Baseline}

\subsubsection{Architecture}
A simple fully-connected network with two hidden layers:
\begin{itemize}
\item fc1: Linear $p \to 256$, followed by ReLU.
\item fc2: Linear $256 \to 128$, followed by ReLU.
\item fc3: Linear $128 \to q$ (no activation).
\end{itemize}
After training, the first-layer weights $\bW_1 \in \mathbb{R}^{256 \times p}$ are extracted, rescaled back to the original input space, and the right singular vectors of the rescaled matrix are taken as candidate directions.
\newpage
\subsubsection{Training Parameters}
\begin{table}[htbp]
\centering
\caption{Training hyperparameters for the MLP baseline.}
\label{tab:mlp_baseline}
\begin{tabular}{l c l}
\toprule
Parameter & Value & Remarks \\
\midrule
Optimizer & Adam & default betas \\
Learning rate & $1\times10^{-3}$ & -- \\
Weight decay & None & -- \\
Batch size & 64 & -- \\
Epochs & 100 & fixed \\
Loss function & MSE & -- \\
Hidden dimensions & 256, 128 & fc1, fc2 \\
Input standardization & Yes & via \texttt{StandardScaler} \\
\bottomrule
\end{tabular}
\end{table}

\subsection{Evaluation MLP (for assessing qualities of estimated CSs)}

This MLP is trained on the projected training set and early-stopped on the validation set. It is used solely to evaluate the quality of the reduced data, not for direction estimation.

\subsubsection{Architecture}
A single hidden layer network:
\begin{itemize}
\item fc1: Linear $r \to 32$, followed by ReLU.
\item fc2: Linear $32 \to q$ (no activation).
\end{itemize}

\subsubsection{Training Parameters}
\begin{table}[htbp]
\centering
\caption{Training hyperparameters for the evaluation MLP.}
\label{tab:eval_mlp}
\begin{tabular}{l c l}
\toprule
Parameter & Value & Remarks \\
\midrule
Optimizer & Adam & default betas \\
Learning rate & $1\times10^{-3}$ & -- \\
Batch size & 64 & -- \\
Epochs & 200 & early stopping based on validation MSE \\
Loss function & MSE & -- \\
Gradient clipping & Max norm = 1.0 & applied after each backward pass \\
Input standardization & Yes & fitted on projected training set \\
Early stopping & Monitored on validation MSE & best test MSE retained \\
\bottomrule
\end{tabular}
\end{table}

\section{Remaining details of the Simulation Study and Real Data Analysis}\label{sec:app-ss&rda}

\begin{algorithm}
\caption{Adaptive Rank Selection (ARS)}
\label{alg:ars}
\begin{algorithmic}[1]
\STATE \textbf{Input:} Data matrix $\bX \in \mathbb{R}^{n \times p}$, response matrix $\bY_{\ell} \in \mathbb{R}^{n_{\ell} \times q}$, sample cross-moment $\widehat{\bM}_{\ell} \in \mathbb{R}^{p \times q}$, estimated score matrix $\widehat{\bS}_{\ell} \in \mathbb{R}^{n_{\ell} \times p}$, skewness threshold $s_0$, kurtosis threshold $k_0$, maximum rank $r_{\max}$, energy threshold $\eta$, permutation percentile $\alpha$, number of permutations $m$.
\STATE \textbf{Output:} Estimated rank $\widehat r$.

\STATE For each feature $j=1,\dots,p$, compute sample skewness $s_j$ and excess kurtosis $k_j$ from $\bX$.
\STATE $s_{\max} \gets \max_j |s_j|$, $k_{\max} \gets \max_j k_j$.
\STATE Compute the singular values $\bm{\lambda} = (\lambda_1,\dots,\lambda_{r_{\max}})$ of $\widehat{\bM}_{\ell}$ (take the first $r_{\max}$ values, pad with zeros if fewer).

\IF{heavy tail is detected ($s_{\max} > s_0$ or $k_{\max} > k_0$)}
    \STATE Compute total variance $v \gets \sum_{i=1}^{r_{\max}} \lambda_i^2$.
    \STATE $\widehat r \gets r_{\max}$.
    \FOR{$k = 1$ to $r_{\max}$}
        \IF{$\sum_{i=1}^{k} \lambda_i^2/v \ge \eta$}
            \STATE $\widehat r \gets k$
            \STATE \textbf{break}
        \ENDIF
    \ENDFOR
\ELSE
    \STATE \textbf{Comment:} Light-tailed: permutation test
    \STATE Initialize matrix $\bP \in \mathbb{R}^{m \times r_{\max}}$.
    \FOR{$b = 1$ to $m$}
        \STATE Permute rows of $\bY_{\ell}$ to obtain $\bY_{\text{perm}}$.
        \STATE $\bM_{\text{perm}} \gets (1/n_{\ell}) \widehat{\bS}_{\ell}^\top \bY_{\text{perm}}$.
        \STATE Compute SVD of $\bM_{\text{perm}}$; store first $r_{\max}$ singular values in row $b$ of $\bP$.
    \ENDFOR
    \FOR{$k = 1$ to $r_{\max}$}
        \STATE $\theta_k \gets$ the $\alpha$-quantile of the $k$-th column of $\bP$.
    \ENDFOR
    \STATE $\widehat r \gets 0$.
    \FOR{$k = 1$ to $r_{\max}$}
        \IF{$\lambda_k > \theta_k$}
            \STATE $\widehat r \gets k$
        \ELSE
            \STATE \textbf{break}
        \ENDIF
    \ENDFOR
    \STATE $\widehat r \gets \max(1, \widehat r)$.
\ENDIF
\RETURN $\widehat r$.
\end{algorithmic}
\end{algorithm}

\begin{algorithm}
\caption{Rank Selection for SIR}
\label{alg:sir-rank}
\begin{algorithmic}[1]
\STATE \textbf{Input:} Covariates 
      $\bX \in \mathbb{R}^{n \times p}$, responses
      $\bY \in \mathbb{R}^{n \times q}$, whitened covariates
      $\bZ \in \mathbb{R}^{n \times p}$, maximum rank
      $r_{\max}$ , skewness threshold
      $s_0$, kurtosis threshold
      $k_0$, energy threshold for heavy tails
      $\tau$, percentile for permutation test
      $\alpha$, number of permutations
      $m$.
\STATE \textbf{Output:} Estimated rank $\widehat r$.

\STATE For each feature $j=1,\dots,p$, compute sample skewness $s_j$ and excess kurtosis $k_j$.
\STATE $s_{\max} \gets \max_j |s_j|$, $k_{\max} \gets \max_j k_j$.

\IF{heavy tail is detected ($s_{\max} > s_0$ or $k_{\max} > k_0$)}
    \STATE Compute the slice-averaged matrix $\bM_{\text{slice}}$ from $\bZ$ and $\bY$.
    \STATE Compute singular values $\bm{\lambda}$ of $\bM_{\text{slice}}$.
    \STATE Compute total variance $v \gets \sum_{i=1}^{r_{\max}} \lambda_i^2$.
    \STATE $\widehat r \gets r_{\max}$.
    \FOR{$k = 1$ to $r_{\max}$}
        \IF{$\frac{\sum_{i=1}^{k} \lambda_i^2}{v} \ge \tau$}
            \STATE $\widehat r \gets k$
            \STATE \textbf{break}
        \ENDIF
    \ENDFOR
    \STATE \textbf{return} $\widehat r$.
\ELSE
    \STATE $\rhd$ Light-tailed: use SIR-specific permutation test.
    \STATE Compute singular values $\bm{\lambda}$ of the original slice-averaged matrix $\bM_{\text{slice}}$.
    \STATE Initialize matrix $\bP \in \mathbb{R}^{m \times r_{\max}}$.
    \FOR{$b = 1$ to $m$}
        \STATE Permute rows of $\bY$ to obtain $\bY_{\text{perm}}$.
        \STATE Project $\bY_{\text{perm}}$ onto a random direction and slice $\bZ$ accordingly.
        \STATE Compute the permuted slice-averaged matrix $\bM_{\text{slice, perm}}$.
        \STATE Compute SVD of $\bM_{\text{slice, perm}}$; take first $r_{\max}$ singular values
              and store in row $b$ of $\bP$.
    \ENDFOR
    \FOR{$k = 1$ to $r_{\max}$}
        \STATE $\theta_k \gets$ the $\alpha$-quantile of the $k$-th column of $\bP$.
    \ENDFOR
    \STATE $\widehat r \gets 0$.
    \FOR{$k = 1$ to $r_{\max}$}
        \IF{$\lambda_k > \theta_k$}
            \STATE $\widehat r \gets k$
        \ELSE
            \STATE \textbf{break}
        \ENDIF
    \ENDFOR
    \STATE $\widehat r \gets \max(1, \widehat r)$.
    \STATE \textbf{return} $\widehat r$.
\ENDIF

\end{algorithmic}
\end{algorithm}

\begin{algorithm}
\caption{Rank Selection for MLP}
\label{alg:mlp-rank}
\begin{algorithmic}[1]
\STATE \textbf{Input:} Covariates
      $\bX \in \mathbb{R}^{n \times p}$, responses
      $\bY \in \mathbb{R}^{n \times q}$, first-layer weight matrix of a trained MLP
      $\bW_1 \in \mathbb{R}^{h \times p}$, maximum rank
      $r_{\max}$, skewness threshold
      $s_0$, kurtosis threshold
      $k_0$, energy threshold for heavy tails
      $\tau_h$, nergy threshold for light tails
      $\tau_l$, dominance threshold for light tails
      $\rho$.
\STATE \textbf{Output:} Estimated rank $\widehat r$.

\STATE For each feature $j=1,\dots,p$, compute sample skewness $s_j$ and excess kurtosis $k_j$.
\STATE $s_{\max} \gets \max_j |s_j|$, $k_{\max} \gets \max_j k_j$.

\STATE Compute singular values $\bm{\lambda} = (\lambda_1,\dots,\lambda_{r_{\max}})$ of $\bW_1$ (take the first $r_{\max}$ values).

\IF{heavy tail is detected ($s_{\max} > s_0$ or $k_{\max} > k_0$)}
    \STATE $\rhd$ Heavy-tailed:
    \STATE Compute total variance $v \gets \sum_{i=1}^{r_{\max}} \lambda_i^2$.
    \STATE $\widehat r \gets r_{\max}$.
    \FOR{$k = 1$ to $r_{\max}$}
        \IF{$\frac{\sum_{i=1}^{k} \lambda_i^2}{v} \ge \tau_h$}
            \STATE $\widehat r \gets k$
            \STATE \textbf{break}
        \ENDIF
    \ENDFOR
    \STATE \textbf{return} $\widehat r$.
\ELSE
    \STATE $\rhd$ Light-tailed: use conservative energy method.
    \STATE Compute total variance $v \gets \sum_{i=1}^{r_{\max}} \lambda_i^2$.
    \IF{$v = 0$}
        \STATE $\widehat r \gets 1$.
    \ELSIF{$\lambda_1^2 / v > \rho$}
        \STATE $\widehat r \gets 1$.  $\rhd$ First direction dominates.
    \ELSE
        \STATE Compute cumulative variance proportion.
        \STATE $\widehat r \gets \arg\min_{k} \left\{ \sum_{i=1}^{k} \lambda_i^2 / v \ge \tau_l \right\}$.
        \STATE $\widehat r \gets \max(1, \min(\widehat r, r_{\max}))$.
    \ENDIF
    \STATE \textbf{return} $\widehat r$.
\ENDIF
\end{algorithmic}
\end{algorithm}

\begin{table}[htbp]
\centering
\caption{Hyperparameters for rank selection.}
\label{tab:rank-params}
\begin{tabular}{l l}
\toprule
Method & Parameters \\
\midrule
$\mathrm{S}_{N,\text{sup}}$/$\mathrm{S}_{N,\text{semi}}$ & Heavy: $\tau=0.95$; Light: $\alpha=0.90, m=200$ \\
 $\mathrm{S}_{\nu,\text{sup}}$/
$\mathrm{S}_{\nu,\text{semi}}$ & Heavy: $\tau=0.95$; Light: $\alpha=0.90, m=200$ \\
SIR & Heavy: $\tau=0.95$; Light: $\alpha=0.90, m=200$ \\
MLP & Heavy: $\tau_h=0.95$; Light: $\tau_l=0.90,\ \rho=0.85$ \\
\bottomrule
\end{tabular}
\parbox{0.9\linewidth}{\small
$\tau$: energy threshold; $\alpha$: permutation quantile; $m$: number of permutations; $\tau_l$: light-tail energy threshold; $\rho$: dominance ratio.
}
\end{table}

\end{document}